\documentclass[11pt]{article}
\usepackage{curves}

\usepackage{amssymb,amsfonts,amsmath,amsthm,amscd,mathrsfs,dsfont,bbm}
\usepackage{graphicx,float,epsfig,psfrag}
\usepackage{wrapfig}
\usepackage{subfig}

\DeclareMathAlphabet{\mathpzc}{OT1}{pzc}{m}{it}

\newtheorem{propo}{Proposition}[section]
\newtheorem{lemma}[propo]{Lemma}
\newtheorem{definition}[propo]{Definition}

\newtheorem{thm}[propo]{Theorem}

\newtheorem{remark}[propo]{Remark}

\def\ttheta{\tilde{\theta}}
\def\ident{{\mathbbm{1}}}

\def\cE{{\cal E}}
\def\E{{\mathbb E}}
\def\uX{\underline{X}}
\def\eC{\widehat{C}}
\def\ux{\underline{x}}
\def\alg{{\sf Alg}}
\def\thres{{\sf Thr}}
\def\ind{{\sf Ind}}
\def\indD{{\sf IndD}}
\def\rlr{{\sf Rlr}}
\def\prob{{\mathbb P}}
\def\hprob{\widehat{\mathbb P}}
\def\atanh{{\rm atanh}}
\def\arctanh{{\rm arctanh}}
\def\dr{{\partial r}}
\def\di{{\partial i}}
\def\score{\mbox{\sc Score}}
\def\utheta{\underline{\theta}}
\def\hutheta{\hat{\underline{\theta}}}
\def\htheta{\hat{\theta}}
\def\hz{\hat{z}}
\def\ua{\underline{a}}
\def\eps{\epsilon}
\def\psucc{{\rm P}_{\rm succ}}
\def\<{\langle}
\def\>{\rangle}

\def\Tree{{\sf T}}
\def\dtree{\partial{\sf T}}
\def\Ball{{\sf B}}
\def\dist{{\rm dist}}

\newcommand{\reals}{{\mathds R }}

\newcommand{\naturals}{{ \mathds N}}

\begin{document}

\title{On the trade-off between complexity and correlation decay in
  structural learning algorithms}

\author{Jos\'e Bento\thanks{~Department of  
Electrical Engineering, Stanford University}\;\;\; and\;\; 
Andrea Montanari\thanks{~Departments of  
Electrical Engineering and Statistics, Stanford University}
}

\date{\today}

\maketitle

\begin{abstract}
We consider the problem of learning the structure of
Ising models (pairwise binary Markov random fields) from i.i.d. samples.
While several methods have been proposed to accomplish this task,
their relative merits and limitations remain somewhat obscure.
By analyzing a number of concrete examples, we show that low-complexity 
algorithms often fail when the Markov random field 
develops long-range correlations. More precisely, this phenomenon 
appears to be related to the Ising model phase transition 
(although it does not coincide with it).
\end{abstract}

\section{Introduction and main results}

Given a graph $G= (V=[p],E)$, and a positive parameter $\theta>0$
the \emph{ferromagnetic Ising model on $G$} is the pairwise
Markov random field 
\begin{eqnarray}
\mu_{G,\theta}(\ux) =\frac{1}{Z_{G,\theta}}\,
\prod_{(i,j)\in E} e^{\theta x_i x_j}\label{eq:IsingModel}
\end{eqnarray}
over binary variables $\ux = (x_1,x_2,\dots,x_p)$, $x_i\in
\{+1,-1\}$. Apart from being one of the best studied models in
statistical mechanics \cite{Huang,GrimmettRCM}, 
the Ising model is a prototypical undirected graphical model.
Since the seminal work of Hopfield \cite{Hopfield} and Hinton
and Sejnowski \cite{Hinton1}, it has found application in numerous
 areas of machine learning, computer vision, clustering and spatial statistics.  

The obvious generalization of the distribution  (\ref{eq:IsingModel})
to edge-dependent parameters $\theta_{ij}$, $(i,j)\in E$ is of central
interest in such applications, and will be introduced in Section \ref{sec:Pseudo}.
Let us stress that we follow the statistical mechanics convention
of calling (\ref{eq:IsingModel}) an Ising model even if the graph $G$
is not a grid.

In this paper we study the following structural learning problem:
\begin{quote}
\emph{Given $n$ i.i.d. samples $\ux^{(1)}$, $\ux^{(2)}$,\dots,
  $\ux^{(n)}\in\{+ 1, -1\}^p$
with distribution $\mu_{G,\theta}(\,\cdot\,)$, reconstruct the 
graph $G$.} 
\end{quote}
For the sake of simplicity, we assume in most of the paper that the parameter 
$\theta$ is known, and that $G$ has no double edges (it is a 
`simple' graph). We focus therefore on the key challenge of learning
the graph structure associated to the measure
$\mu_{G,\theta}(\,\cdot\,)$. This structure is particularly important
for extracting the qualitative features of the model, since it encodes
its conditional independence properties.

It follows from the general theory of exponential families that, for
any $\theta\in (0,\infty)$, 
the model (\ref{eq:IsingModel}) is identifiable \cite{LehmannCasella}. 
In particular, the structural learning problem is solvable with unbounded
sample complexity and computational resources. The question we
address is: for which classes of graphs and values of the parameter
$\theta$ is the problem solvable under realistic complexity
constraints? More precisely, given an algorithm $\alg$, a graph
$G$, a value $\theta$ of the model parameter, and a small 
$\delta>0$, the sample complexity is defined as
\begin{eqnarray}
n_{\alg}(G,\theta) \equiv \inf\left\{ n\in\naturals:\, 
\prob_{n,G,\theta}\{\alg(\ux^{(1)},\dots,\ux^{(n)})=G\}\ge 1-\delta\right\}\, ,
\end{eqnarray}
where $\prob_{n,G,\theta}$ denotes probability with respect to $n$ 
i.i.d. samples with distribution $\mu_{G,\theta}$. 
Further, we let $\chi_{\alg}(G,\theta)$ denote the number of operations
of the algorithm $\alg$, when run on $n_{\alg}(G,\theta)$ samples.
The general problem is therefore to characterize the functions
$n_{\alg}(G,\theta)$ and $\chi_{\alg}(G,\theta)$, and to design
algorithms that minimize the complexity.

Let us emphasize that these are not the only possible definitions
of sample and computational complexity. Alternative definitions are  
obtained by requiring that the reconstructed structure
$\alg(\ux^{(1)},\dots,\ux^{(n)})$ is only partially correct.  However, for the
algorithms considered in this paper, such definitions should not
result in qualitatively different behavior\footnote{Indeed
the algorithms considered in this paper reconstruct $G$ by separately 
estimating the neighborhood of each node $i$. 
This implies that any significant probability of error results
in a substantially different graph.}

General upper and lower
bounds  on the sample complexity $n_{\alg}(G,\theta)$ were proved
by Santhanam and Wainwright \cite{martin_info_limits,santhanam_info_limits}, without however
taking into account computational complexity. 
On the other end of the spectrum, several low complexity algorithms
have been developed in the last few years (see Section
\ref{sec:RelatedWork} for a brief overview). However
the resulting sample complexity bounds only hold under specific
assumptions on the underlying model (i.e. on the pair $(G,\theta)$).
A general understanding of the trade-offs between sample complexity
and computational  complexity  is largely lacking. 

This paper is devoted to the study of the tradeoff 
between sample complexity and computational complexity for some specific
structural learning algorithms, when applied to the Ising model. 
An important challenge consists in the fact that the model 
(\ref{eq:IsingModel}) induces subtle correlations between
the binary variables $(x_1,\dots,x_p)$.  The objective of a structural
learning algorithm  is to disentangle pairs $x_i,x_j$ that are conditionally 
independent given the other variables (and hence are not connected by 
an edge) from those that are instead conditionally dependent (and
hence connected by an edge in $G$). This becomes particularly
difficult when $\theta$ becomes large and hence 
pairs $x_i$, $x_j$ that are not connected by an edge in $G$ 
become strongly dependent.  The next section sets the stage for our
work by discussing a simple and concrete illustration
of this phenomenon.

\subsection{A toy example} \label{sec:toy_example}

\begin{figure}
\begin{center}
\includegraphics[width=0.7\linewidth]{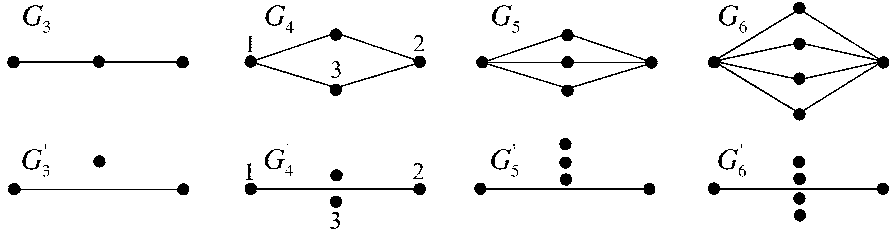}
\end{center}
\caption{Two families of graphs $G_p$ and $G'_p$ whose distributions 
$\mu_{G_p,\theta}$ and $\mu_{G'_p,\theta'}$ merge as $p$ gets large.
}\label{fig:SimpleGraph}
\end{figure}
As a toy illustration\footnote{A similar example was considered in
\cite{Netrapalli}.} of the challenges of structural learning, we will
study the two families of  graphs in Figure \ref{fig:SimpleGraph}. 
The two families will be denoted by $\{G_{p}\}_{p\ge 3}$ and
$\{G'_p\}_{p\ge 3}$ and are indexed by the number of vertices $p$. 

Graph $G_p$ has $p$ vertices
and $2(p-2)$ edges.
Two of the vertices (vertex $1$ and vertex $2$) 
have degree $(p-2)$, and $(p-2)$ have degree
$2$. Graph $G'_p$ has also $p$ vertices, but only one edge
between vertices $1$ and $2$.
In other words, graph $G_p'$ corresponds to variables $x_1$ 
and $x_2$ interacting `directly' (and hence not conditionally
independent), 
while graph $G_p$ describes a situation in which the two variables
interact `indirectly'  through
numerous weak intermediaries (but still are conditionally independent
since they are not connected).
Fix $p$, and assume that one of $G_p$ or $G'_p$ is chosen randomly and i.i.d.
samples $\ux^{(1)}$,\dots,$\ux^{(n)}$ from the corresponding Ising
distribution are given to us. 

Can we efficiently distinguish the two graphs, i.e. 
infer whther the samples were generated using $G_p$ or $G'_p$?
 As mentioned above, 
since the model is identifiable, this task can be achieved with
unbounded sample and computational complexity.
Further, since model (\ref{eq:IsingModel}) is an exponential family,
the $p\times p$ matrix of empirical covariances
$(1/n)\sum_{\ell=1}^n\ux^{(\ell)}(\ux^{(\ell)})^T$
provides a sufficient statistic for inferring
the graph structure. 

In this specific example, we assume that different edge strengths
are used in the two graphs: $\theta$ for graph $G_p$
and $\theta'$ for graph $G_p'$ (i.e. we have to distinguish between
$\mu_{G_p,\theta}$ and $\mu_{G'_p,\theta'}$).
We claim that, by properly choosing the parameters
$\theta$ and $\theta'$, we  can ensure that the covariances
approximately match $|\E_{G_p,\theta}\{x_ix_j\}-\E_{G'_p,\theta'}\{x_ix_j\}| = O(1/\sqrt{p})$.
Indeed the same remains true for all marginals
involving a bounded number of variables. Namely, for all subsets of vertices $U
\subseteq [p]$ of bounded size
$|\mu_{G_p,\theta}(\ux_U)-\mu_{G'_p,\theta'}(\ux_U)| = O(1/\sqrt{p})$.
Low-complexity algorithms typically estimate each edge using only a small subset
low--dimensional marginal. Hence, they are bound to fail unless
the number of samples $n$ diverges with the graph size $p$.
On the other hand, a naive information-theoretic lower bound 
(in the spirit of \cite{martin_info_limits,santhanam_info_limits})
only yields $n_{\alg}(G,\theta) = \Omega(1)$.
This sample complexity is achievable by using global statistics to distinguish the
two graphs.

In other words, even for this simple example, a dichotomy emerges:
either a number of samples has to grow with the number of parameters,
or algorithms have to exploit a large number of marginals of $\mu_{G,\theta}$.

To confirm our claim, we need to 
compute the covariance of the Ising measures distributions
$\mu_{G_p,\theta}$, $\mu_{G'_p,\theta'}$. We easily obtain, for the latter
graph 
\begin{eqnarray}
\E_{G_p',\theta'}\{x_1x_2\} &= & \tanh \theta'\, ,\label{eq:Simple1}\\
\E_{G_p',\theta'}\{x_ix_j\} &= & 0\, .\;\;\;\;\;\;\;\;\;\;\label{eq:Simple2}
(i,j)\neq (1,2)\, .
\end{eqnarray}
The calculation is somewhat more intricate for graph $G_p$,
so we defer complete formulae to Appendix \ref{sec:toy_examp_comp}
and report here only the result for $p\gg 1$, $\theta\ll 1$:
\begin{eqnarray}
\E_{G_p,\theta}\{x_1x_2\} &= & \tanh\big\{p\theta^2+O(p\theta^4,\theta)\big\}\,\, ,\\
\E_{G_p,\theta}\{x_ix_j\} &= & O(\theta,p\theta^3)\, ,     \;\;\;\;\;\;\;\;\;\;\;\;
i\in\{1,2\},j\in\{3,\dots,p\}\, ,\\
\E_{G_p,\theta}\{x_ix_j\} &= & O(\theta^2,p\theta^4)\, , \;\;\;\;\;\;\;\;\;\;\;
i,j\in\{3,\dots,p\}\, .
\end{eqnarray}
In other words, variables $x_1$ and $x_2$ are strongly correlated 
(although not connected), while all the other variables are weakly
correlated. By letting $\theta=\sqrt{\theta'/p}$ this covariance structure
matches Eqs.~(\ref{eq:Simple1}), (\ref{eq:Simple2}) up to corrections
of order $1/\sqrt{p}$. 

Notice that the ambiguity between the two models $G_p$ and $G'_p$
arises because several weak, indirect paths between
$x_1$ and $x_2$ in graph $G_p$, add up to the same effect as a 
strong direct connection. 
This toy example is hence suggestive of the general phenomenon that 
strong long-range correlations can `fake' a direct connection.
However, the example is not completely convincing for
several reasons:
\begin{enumerate}
\item Most algorithms of interest estimate each edge on the basis of 
a large number of low-dimensional marginals (for instance \emph{all}
pairwise correlations).
\item Reconstruction guarantees have been proved for graphs
  with bounded degree
  \cite{abbeel,mossel,martin_info_limits,santhanam_info_limits,martin},
while here we are letting the maximum degree be as large as the system
size. Notice however that a the graph considered here are
only sparse `on average'.
\item It may appear that the difficulty in distinguishing graph $G_p$
  from $G'_p$ is related to the fact that in the former we take
  $\theta=O(1/\sqrt{p})$. This is however the natural scaling when the 
degree of a vertex is large, in order to obtain a non-trivial
distribution. If the graph $G_p$ had $\theta$ bounded away from
$0$, this would result in a distribution $\mu_{G_p,\theta}(\ux)$ 
concentrated on the two antipodal configurations:
all-$(+1)$ and  all-$(-1)$. Structural learning would be equally
difficult in this case.
\end{enumerate}
Despite these points, this model provides already 
a useful counter-example. 
In Appendix \ref{sec:rlf_for_gp} we will show why, even for bounded 
$p$ (and hence $\theta$ bounded away from $0$) the model $G_p$ in
Figure \ref{fig:SimpleGraph} `fools' regularized
logistic regression algorithm of Ravikumar, Wainwright and Lafferty \cite{martin}.
Regularized logistic regression reconstructs $G'_p$ instead of $G_p$.

\subsection{Outline of the paper}

The rest of this paper is devoted to bounding
the sample complexity $n_{\alg}$ and computational complexity 
$\chi_{\alg}$ for a number of graph models, as a function of $\theta$.
Results of this analysis are presented in Section \ref{sec:Results}
for three algorithms: a simple thresholding algorithm, the
conditional independence test method of \cite{mossel} and the penalized
pseudo-likelihood method of \cite{martin}. In Section
\ref{sec:num_exper}, we validate our analysis through numerical
simulations. Finally, Section \ref{sec:proofs_main} contains the proofs with
some technical details deferred to the appendices.

This analysis unveils a general pattern: 
\emph {when the model (\ref{eq:IsingModel}) develops strong
correlations,  several low-complexity algorithms fail, or require a large
number of samples.}
What does `strong correlations' mean? As the toy example in the
previous section demonstrates, correlations arise from a trade-off
between the degree (which we will characterize here via the maximum
degree $\Delta$), and the interaction strength $\theta$.
It can be ascribed to a few strong connections (large $\theta$) or
to a large number of weak connections (large $\Delta$).
Is there any meaningful way to compare and combine these quantities 
($\theta$ and $\Delta$)? 
An answer is suggested by the theory of Gibbs measures
which predicts a dramatic change of behavior when $\theta$
crosses the so-called 
`uniqueness threshold' $\theta_{\rm uniq}(\Delta) =
\atanh(1/(\Delta-1))$ \cite{Georgii}.
For $\theta<\theta_{\rm uniq}(\Delta)$ Gibbs sampling mixes rapidly
and far apart variables in $G$ are roughly independent
\cite{MosselSly}. Vice versa,
for any $\theta>\theta_{\rm uniq}(\Delta)$ there exist graph families
on which Gibbs sampling is slow, and far apart variables are strongly
dependent \cite{GerschenfeldMontanari}.
While polynomial sampling algorithms exists for 
all $\theta>0$ \cite{JerrumIsing}, for $\theta<0$, in the regime $|\theta|>\theta_{\rm
  uniq}(\Delta)$ sampling is arguably $\#$-P hard \cite{SlyIndependentSet}.
Related to the uniqueness threshold is also the phase transition
threshold, which is graph dependent, with typically 
$\theta_{\rm crit}\le {\rm const.}/\Delta$.

We will see that this is indeed a relevant way of comparing
interaction strength and degree, even for structural learning.
Al the algorithms we analyzed (mentioned above) provably fail for 
$\theta\gg {\rm const.}/\Delta$, for a number of `natural' graph families.
Our work raises several fascinating questions, the most important being
the construction of structural learning algorithm with provable
performance guarantees in the strongly dependent regime
$\theta_{\rm crit}\gg {\rm const.}/\Delta$.
The question as to whether such an algorithm exists  is left 
open by the present paper (but see next section for an overview of
earlier work).

Let us finally emphasize that
we do not think that any of the specific families of graphs studied in the
present paper is intrinsically `hard' to learn. For instance, we show
below that the regularized logistic regression method of \cite{martin}
fails on random regular graphs, while it is easy to learn such graphs
using the simple thresholding algorithm of Section
\ref{sec:SimpleThr}. 
 The specific families where indeed chosen 
mostly because they are analytically tractable.
 
\subsection{Further related work}
\label{sec:RelatedWork}

Traditional algorithms for learning Ising models were developed in the
context of Boltzmann machines \cite{Hinton1,Hinton2,Hinton3}. These algorithms
try to solve the maximum likelihood problem by gradient ascent.
Estimating the  gradient of the log-likelihood function
requires to compute expectations with respect to the Ising
distribution. In these works, this was done using 
the Markov Chain Monte Carlo (MCMC) method, and more specifically
Gibbs sampling. 

We shall not consider this approach
in our study for two type of reasons. First of all, it does not output
a `structure' (i.e. a sparse subset of the $\binom{p}{2}$ potential
edges): because of approximation errors, it yields non-zero values for all the edges.
This problem can in principle be overcome by using suitably
regularized objective functions, but such a modified algorithm was never studied.

Second, the need to compute expectation values with respect to the
Ising distribution, and the use of MCMC to achieve this goal, 
poses some  fundamental limitations. 
As mentioned above, the Markov chain commonly used by these methods is simple Gibbs
sampling. This is known to have mixing time that grows exponentially
in the number of variables for  $\theta>\theta_{\rm uniq}(\Delta)$,
and hence does not yield good estimates of the expectation values in practice.
While polynomial sampling schemes exist for  models with $\theta>0$
\cite{JerrumIsing}, they do not apply to $\theta<0$ or to general
models with edge-dependent parameters $\theta_{ij}$. Already in
the case  $\theta<0$, estimating expectation values of the Ising 
distribution is likely to be $\#$-P hard \cite{SlyIndependentSet}.

Abbeel, Koller and Ng \cite{abbeel} first developed a method 
with computational complexity provably polynomial in the number of
variables, for bounded maximum degree,  and logarithmic sample
complexity. Their approach is based on ingenious use of the
Hammersley-Clifford representation of 
Markov Random Fields.  Unfortunately, the computational complexity of
this approach
is of order $p^{\Delta+2}$ which becomes unpractical for reasonable
values of the degree and network size (and superpolynomial for
$\Delta$ diverging with $p$).   
The algorithm by Bresler, Mossel and Sly \cite{mossel} studied in
Section \ref{sec:LocalInd} presents similar limitations, that the
authors overcome (in the small $\theta$ regime) by exploiting the
correlation decay phenomenon.

An alternative point of view consists in using standard regression
methods. In the context of Ising models, Ravikumar, Wainwright and
Lafferty \cite{martin} showed that the neighborhood of a vertex $i$
can be  efficiently reconstructed by solving an appropriate
regularized regression problem. More precisely, the values of variable $x_i$
are regressed against the value of all the other variables. 
The logistic regression log-likelihood is regularized by
adding an $\ell_1$-penalty that promotes the selection of sparse graph structures.
We will analyze this method in Section  \ref{sec:Pseudo}.
The  approach of \cite{martin} extends to non-Gaussian models earlier work by
Meinshausen and B\"{u}hlmann \cite{Meinshausen}.
Let us notice in passing that the case of Gaussian graphical models is
substantially easier since the log-likelihood of a given model can be
evaluated easily in this case \cite{Glasso}.

A short version of this paper was presented at the 2009 Neural Information
Processing Systems symposium.
Since then, at least two groups explored the challenges put forward
in our work.
Anandkumar, Tan and Willsky \cite{Anandkumar} prove that,
for sequences of random graphs which are sparse on average (i.e.
with bounded average degree), structural learning is possible
throughout the correlation decay regime $\theta<\theta_{\rm crit}$.
This result generalizes  our analysis of random regular graphs
(see next section), to the more challenging case of graphs with 
random degrees.
Cocco and Monasson \cite{Cocco} proposed and `adaptive cluster'
heuristics and demonstrated empirically good performances for
specific graph families, also for $\theta>\theta_{\rm crit}$. A
mathematical
analysis of their approach is lacking.


\section{Results} \label{sec:Results}

\subsection{The simple thresholding algorithm}
\label{sec:SimpleThr}

In order to illustrate the interplay between graph structure,
sample complexity and interaction strength $\theta$, it is
instructive to consider a simple example. The thresholding 
algorithm reconstructs $G$ by thresholding the empirical correlations
\begin{eqnarray}
\eC_{ij} \equiv \frac{1}{n}\sum_{\ell =1}^{n}x^{(\ell)}_ix^{(\ell)}_j\, ,
\end{eqnarray}
for $i,j\in V$. 

\phantom{a}

\begin{tabular}{ll}
\hline
\multicolumn{2}{l}{ {\sc Thresholding}( samples $\{x^{(\ell)}\}$, threshold 
$\tau$ )}\\
\hline
1: & Compute the empirical correlations $\{\eC_{ij}\}_{(i,j)\in V\times V}$;\\
2: & For each $(i,j)\in V\times V$\\
3: & \phantom{aaa}If $\eC_{ij}\ge \tau$, set $(i,j)\in E$;\\
\hline
\end{tabular}

\phantom{a}

We will denote this algorithm by $\thres(\tau)$. Notice that its complexity
is dominated by the computation of the empirical correlations,
i.e. $\chi_{\thres(\tau)} = O(p^2n)$. 
The sample complexity $n_{\thres (\tau)}$ can be bounded for specific
classes of graphs as follows (for proofs see Section
\ref{sec:simp_thres_proofs}). 
\begin{thm}
If $G$ is a tree, and $\tau(\theta) = (\tanh \theta + \tanh^2 \theta)/2$,
then
\begin{eqnarray}
n_{\thres (\tau)}(G,\theta) \le \frac{32}{(\tanh \theta - \tanh^2 \theta )^{2}} 
\;\log \frac{2p}{\delta}\,.
\end{eqnarray}
\label{th:tresh1}
\end{thm}

\begin{thm}
If $G$ has maximum degree $\Delta > 1$ and if $\theta < 
\atanh (1/(2\Delta))$ then there exists $\tau = \tau(\theta)$ such that
\begin{eqnarray}
n_{\thres (\tau)}(G,\theta) \le \frac{32}{(\tanh \theta - \frac{1}{2\Delta}  )^{2}} 
\;\log \frac{2p}{\delta}\,.
\end{eqnarray}
Further, the choice $\tau(\theta) = (\tanh \theta + (1/2\Delta))/2$ achieves
this bound.\label{th:tresh2}
\end{thm}

\begin{thm}
There exists a numerical constant $K$ such that the following is true.
If $\Delta > 3$ and $\theta> K/\Delta$, there are graphs of bounded degree $\Delta$ such that for any $\tau$, $n_{\thres(\tau)} = \infty$, i.e. the thresholding algorithm always fails with high probability. 
\label{th:tresh3}
\end{thm}

These results confirm the idea that the failure of low-complexity algorithms 
is related to long-range correlations in the underlying graphical model.
If the graph $G$ is a tree, then correlations between far apart variables
$x_i$, $x_j$ decay exponentially with the distance between vertices
$i$, $j$. Hence trees can be learnt from $O(\log p)$ samples irrespectively
of their topology and maximum degree (assuming $\theta \neq \infty$).
The same happens on bounded-degree graphs if 
$\theta\le {\rm const.}/\Delta$. However, for $\theta > {\rm const.}/\Delta$,
there exists families of bounded degree graphs with long-range correlations.
%
%
\subsection{More sophisticated algorithms}

In this section we characterize $\chi_{\alg}(G,\theta)$ and 
$n_{\alg}(G,\theta)$ for more advanced algorithms. We again obtain
very distinct behaviors of these algorithms depending on 
the strength of correlations.
We focus on two type of algorithms and only include the proof
of our most challenging result, 
Theorem \ref{th:mart2} (for the proof see Section \ref{sec:ProofMainTheorem}). 

In the following we denote by $\di$ the neighborhood of a node 
$i\in G$ ($i \notin \partial i$), and assume the degree to be bounded: $|\di|\le \Delta$. 

\subsubsection{Local Independence Test}
\label{sec:LocalInd}

A recurring approach to structural learning consists in exploiting
the conditional independence structure encoded by the graph
\cite{abbeel,mossel,Csiszar,Friedman}. 

Let us consider, to be definite,
the approach of \cite{mossel}, specializing it to the model 
(\ref{eq:IsingModel}). Fix a vertex $r$, whose neighborhood 
we want to reconstruct,  and consider the conditional
distribution of $x_r$ given its 
neighbors\footnote{If $\ua$ is a vector and $R$ is a set of indices then  
we denote by $\ua_R$ the vector formed by the components of $\ua$ with index 
in $R$.}:
$\mu_{G,\theta}(x_r|\ux_{\dr})$. Any change of $x_i$, $i\in\dr$,
produces a change in this distribution which is bounded away from $0$.
Let $U$ be a candidate neighborhood, and assume
 $U\subseteq\dr$. Then changing the value of $x_j$, $j \in U$ 
 will produce a noticeable change in the marginal of $X_r$, even
if we condition   on the remaining values in $U$ and in any $W$,
$|W|\le \Delta$. 
On the other hand, if $U\nsubseteq \dr$, then it is  possible to find $W$ 
(with $|W|\le \Delta$) and a node $i \in U$ such that, changing its value 
after fixing all other values in $U \cup W$ will produce no noticeable 
change in the conditional marginal. (Just choose $i \in U \backslash \dr$ 
and $W = \dr\backslash U$). This procedure allows us to distinguish subsets 
of $\dr$ from other sets of vertices, thus motivating the following algorithm.

\phantom{a}

\begin{tabular}{ll}
\hline
\multicolumn{2}{l}{ {\sc Local Independence Test}( samples $\{x^{(\ell)}\}$, thresholds $(\epsilon,\gamma)$ )}\\
\hline
1: & Select a node $r \in V$;\\
2: & Set as its neighborhood the largest candidate neighbor $U$ of\\
   & size at most $\Delta$ for which the score function $\score(U) > \epsilon/2$;\\
3: & Repeat for all nodes $r \in V$;\\
\hline
\end{tabular}

\phantom{a}

The score function $\score(\,\cdot\,)$ depends on $(\{x^{(\ell)}\},\Delta,\gamma)$ and is defined as follows,
\begin{align}
\min_{W,j} \max_{x_i,\underline{x}_W,\underline{x}_U,x_j} &|\hprob_{n,G,\theta}\{X_i = x_i | \underline{X}_W = \underline{x}_W,\underline{X}_{U} = \underline{x}_{U}\} - \nonumber \\
&\hprob_{n,G,\theta}\{X_i = x_i | \underline{X}_W = \underline{x}_W,\underline{X}_{U \backslash j} = \underline{x}_{U \backslash j}, X_j = x_j\}|\, .
\label{eq:ScoreDef}
\end{align}
In the minimum, $|W| \leq \Delta$ and $j \in U$.
In the maximum, the values must be such that
\begin{align}
&\hprob_{n,G,\theta}\{\underline{X}_W = \underline{x}_W,\underline{X}_{U} = \underline{x}_{U}\} > \gamma / 2 \nonumber \\
&\hprob_{n,G,\theta}\{\underline{X}_W = \underline{x}_W,\underline{X}_{U \backslash j} = \underline{x}_{U \backslash j}, X_j = x_j\} > \gamma / 2
\end{align}
$\hprob_{n,G,\theta}$ is the empirical distribution calculated from
the samples $\{\ux^{(\ell)}\}^n_{\ell = 1}$. We denote this algorithm by
$\ind(\epsilon,\gamma)$. The search over candidate neighbors $U$,
the search for minima and maxima in the  computation of the
$\score(U)$ and the computation of $\hprob_{n,G,\theta}$ all
contribute for $\chi_{\ind}(G,\theta)$.
 
Both theorems that follow are consequences
of the analysis of \cite{mossel}, hence omitted.
\begin{thm}
Let $G$ be a graph of bounded degree $\Delta\ge 1$.
For every $\theta$ there exists $(\epsilon, \gamma)$, and a numerical constant $K$, such that 
\begin{align}
&n_{\ind (\epsilon,\gamma)}(G,\theta) \le  \frac{100 \Delta }{\epsilon^2 \gamma^4} \log \frac{2p}{\delta}\, ,\\
&\chi_{\ind(\epsilon,\gamma)}(G,\theta) \le K\, (2p)^{2\Delta +1} \log p\, .
\end{align}
More specifically, one can take 
$\epsilon = \frac{1}{4} \sinh (2\theta)$, $\gamma = e^{-4 \Delta \theta} \; 2^{-2\Delta}$.\label{th:mossel1}
\end{thm}
This first result  implies in particular that $G$ can be reconstructed with 
polynomial complexity for any bounded $\Delta$. However, the degree of
such polynomial is pretty high and non-uniform in $\Delta$. This makes the
above approach impractical.

A way out was proposed in \cite{mossel}. The idea is to identify 
a set of `potential neighbors' of vertex $r$ via thresholding:
\begin{equation}
B(r) = \{i \in V: \eC_{r i} > \kappa / 2\}\, .
\end{equation}
For each node $r \in V$, we evaluate $\score(U)$ by restricting the minimum in
Eq.~(\ref{eq:ScoreDef}) over $W\subseteq B(r)$, and search only over
$U\subseteq B(r)$. We call this algorithm $\indD(\epsilon,\gamma,\kappa)$.
The basic intuition here is that $C_{ri}$ decreases rapidly 
with the graph distance between vertices $r$ and $i$. As mentioned
above, this is true at low temperature.

\begin{thm}
Let $G$ be a graph of bounded degree $\Delta\ge 1$. Assume that 
$\theta < K / \Delta$ for some small enough constant $K$. 
Then there exists $\epsilon, \gamma,\kappa$ such that 
\begin{align}
&n_{\indD (\epsilon,\gamma,\kappa)}(G,\theta) \le  8(\kappa^2+8^{\Delta}) 
\log \frac{4p}{\delta} \, ,\\
&\chi_{\indD_{(\epsilon,\gamma,\kappa)}}(G,\theta) \le K' p \Delta^{\Delta \frac{\log(4/ \kappa)}{\alpha}} + K' \Delta p^2 \log p\, .
\end{align}
More specifically, we can take 
$\kappa = \tanh \theta$, $\epsilon = \frac{1}{4} \sinh (2\theta)$
and $\gamma = e^{-4 \Delta \theta} \; 2^{-2\Delta}$. \label{th:mossel2}
\end{thm}

%
%
\subsubsection{Regularized Pseudo-Likelihoods}
\label{sec:Pseudo}

A different approach to the learning problem 
consists in maximizing an appropriate empirical likelihood
function \cite{martin,Tibshirani,Ghaoui,Yuan,Meinshausen,tibshirani_lasso}. 
In order to control statistical fluctuations,
and select sparse graphs, a regularization term is often added to the
cost function.

As a specific low complexity implementation of this idea,
we consider the $\ell_1$-regularized pseudo-likelihood method
of \cite{martin}. For each node $r$, the following likelihood 
function is considered 
\begin{equation}
L (\utheta;\{x^{(\ell)}\}) = - \displaystyle{ \frac{1}{n} \sum_{\ell = 1}^n {\log  \prob_{n,G,\utheta} (x_r^{(\ell)}| x_{\backslash r}^{(\ell)}) }}
\end{equation}
where $\ux_{\backslash r}=\ux_{V\setminus r} 
= \{x_i:\, i\in V\setminus r\}$ 
is the vector of all variables except  $x_r$ and 
$\prob_{G,\utheta}$ is defined from the following extension of (\ref{eq:IsingModel}),
\begin{eqnarray}
\mu_{G,\utheta}(\ux) =\frac{1}{Z_{G,\utheta}}\,
\prod_{i,j \in V} e^{\theta_{ij} x_i x_j}\label{eq:IsingModel2}
\end{eqnarray}
where $\utheta = \{\theta_{ij}\}_{i,j\in V}$ is a vector of real 
parameters. Model (\ref{eq:IsingModel}) corresponds to 
$\theta_{ij} = 0, \; \forall (i,j) \notin E$ and $\theta_{ij} = \theta, \; \forall (i,j) \in E$.

The function $L (\utheta;\{x^{(\ell)}\})$ depends only on 
$\utheta_{r,\cdot} = \{\theta_{rj}, \, j \in \dr \}$ and is used to estimate 
the neighborhood of each node by the following algorithm, $\rlr(\lambda)$,

\phantom{a}

\begin{tabular}{ll}
\hline
\multicolumn{2}{l}{ {\sc Regularized Logistic Regression}( samples $\{x^{(\ell)}\}$, regularization $(\lambda)$)} \\
\hline
1: & Select a node $r \in V$;\\
2: & Calculate $\displaystyle \hat{\utheta}_{r, \cdot} = \arg \min_{\utheta_{r,\cdot} \in \reals^{p-1} } \{ L(\utheta_{r,\cdot};\{x^{(\ell)}\}) + \lambda \| \utheta_{r,\cdot} \|_1 \}$; \label{eq:rlr_conv_prob} \\
3: & \phantom{aaa} If $\hat{\theta}_{rj} >0$, set $(r,j) \in E$;\\

\hline
\end{tabular}

\phantom{a}

Our first result shows that $\rlr(\lambda)$ indeed reconstructs 
$G$ if $\theta$ is sufficiently small.

\begin{thm}
There exists numerical constants $K_1$, $K_2$, $K_3$,
such that the following is true.
Let $G$ be a graph with degree bounded by $\Delta\ge 3$. 
If $\theta \le K_1/\Delta$, then there exist $\lambda$ such that
\begin{equation}
n_{\rlr (\lambda)}(G,\theta) \leq K_2 \, \theta^{-2} \, \Delta\,  
\log  \frac{8p^2}{\delta} \, .
\end{equation}
Further, the above holds with $\lambda = K_3 \, \theta \, \Delta^{-1/2}$.\label{th:mart1}
\end{thm}
This theorem is proved by noting that for $\theta \le K_1/\Delta$ correlations decay exponentially,
which makes all conditions  in Theorem 1 of \cite{martin} (denoted there by A1 and A2) hold,
and then computing the probability of success as a function of $n$ with slightly more care. The details
of the proof are written in Appendix \ref{sec:succ_reg_log_reg}.

In order to prove a converse to the above result, we need
to make some assumptions on  $\lambda$.

\begin{definition}
Given $\theta>0$, we say that $\lambda$ is \emph{reasonable}
for that value of $\theta$ if the following conditions hold:
$(i)$ $\rlr(\lambda)$ is successful with probability larger than $1/2$
on any star graphs  (a graph composed by a vertex $r$ connected to 
$\Delta$ neighbors, plus isolated vertices) if $n$ is chosen sufficiently high;
$(ii)$ $\lambda\le \delta(n)$ for some sequence $\delta(n)\downarrow 0$.
\end{definition}

In other words, assumption $(i)$ requires the algorithm to be
successful on a particularly simple class of graphs, and hence does
not entail any loss of generality. Assumption $(ii)$ 
encodes instead the standard way of scaling regularization terms,
by letting them vanish as the number of samples increases. This is
necessary in order to get asymptotic consistency of the parameter
values $\theta_{ij}$.
With these assumptions we can state the following converse theorem, whose proof is
deferred to Section \ref{sec:ProofMainTheorem}.
\begin{thm}
There exists a numerical constant $K$ such that the following happens.
If $\theta> K/\Delta$,$\Delta>3$, then there exists graphs $G$ 
of degree bounded by $\Delta$ such that for all reasonable 
$\lambda$, $n_{\rlr(\lambda)}(G) = \infty$, i.e.
regularized logistic regression fails with high probability.
\label{th:mart2}
\end{thm}
The graphs for which regularized logistic regression fails
are not contrived examples. Indeed, as part of
the proof of Theorem \ref{th:mart2}, and as proved in Appendix
\ref{sec:inco_other_graphs}, we have the following facts
about $\rlr(\lambda)$:
\begin{itemize}
\item If $G$ is a tree, then $\rlr(\lambda)$ recover $G$ with high probability for any $\theta$ (for a suitable $\lambda$);
\item For every graph $G_p$ in the family described in Section \ref{sec:toy_example},
$\rlr(\lambda)$ fails with high probability for $\theta$ large enough
and for all $\lambda$;
\item  If $G$ is sampled uniformly from the ensemble of  regular graphs
$\rlr(\lambda)$ fails with high probability for $\theta$ large enough and $\lambda$ `reasonable';
\item  if $G$ is a large two dimensional grid 
It fails with high probability for $\theta$ large enough
and $\lambda$ `reasonable'.
\end{itemize}

We note here that Theorem \ref{th:mart2} relies on proving that a
so-called `Incoherence condition' is necessary for $\rlr$ to
successfully reconstruct $G$. Although a similar result was proven
in \cite{zhao} for model selection using the Lasso, this paper is the
first to prove that a similar Incoherence condition is also necessary
when the underlying model is the Ising model.

 The intuition behind this is quite simple.
Begin by noticing that when $n \rightarrow \infty$, and under the restriction
that $\lambda \rightarrow 0$, solutions given
by $\rlr$ converge to $\utheta^*$ as $n \rightarrow \infty$ \cite{martin}. Hence,
for large $n$, we can expand $L$ in a quadratic function centered
around $\utheta^*$ plus a small stochastic error term. Consequently,
when adding the regularization term
to $L$, we obtain cost function analogous  to the Lasso plus
an error term that needs to be controlled. 
The study of the dominating contribution leads to the incoherence
condition.

In general there are no practical ways to evaluate 
the incoherence condition for a given graphical model. 
This requires in fact to compute expectations with respect to 
the Ising distribution. As discussed above, this is hard for
$|\theta |>\theta_{\rm uniq}(\Delta)$.
Hence this condition was not checked for families of graphs.
A large part of our technical contribution consists indeed in filling this gap.
To this end, we use tools from mathematical statistical mechanics,
namely low temperature series for Ising models on grids
\cite{Domb,Lebowitz}, and local weak convergence results
for Ising models on random graphs \cite{andrea,montmosselsly}.

%
%
\section{Numerical experiments} \label{sec:num_exper}

In order to explore the practical relevance of the above results,
we carried out extensive numerical simulations using the 
regularized logistic regression algorithm $\rlr(\lambda)$.
Among other learning algorithms, $\rlr(\lambda)$
strikes a good balance of complexity and performance.
Samples from the Ising model (\ref{eq:IsingModel}) where generated 
using Gibbs sampling (a.k.a. Glauber dynamics). Mixing time can be 
very large for $\theta\ge \theta_{\rm uniq}$, and was estimated using
the time required for the overall bias to change sign (this is a quite 
conservative estimate at low temperature). Generating the samples
$\{\ux^{(\ell)}\}$ was indeed the bulk of our computational effort and
took about $50$ days CPU time on Pentium Dual Core processors. 
Notice that $\rlr(\lambda)$ had been tested in \cite{martin}
only on tree graphs $G$, or in the weakly coupled
regime $\theta<\theta_{\rm uniq}$. 
In these cases sampling from the Ising model is easy,
but structural learning is also intrinsically easier.

\begin{figure}
\phantom{a}\hspace{-1.2cm}\includegraphics[width=0.61\linewidth]{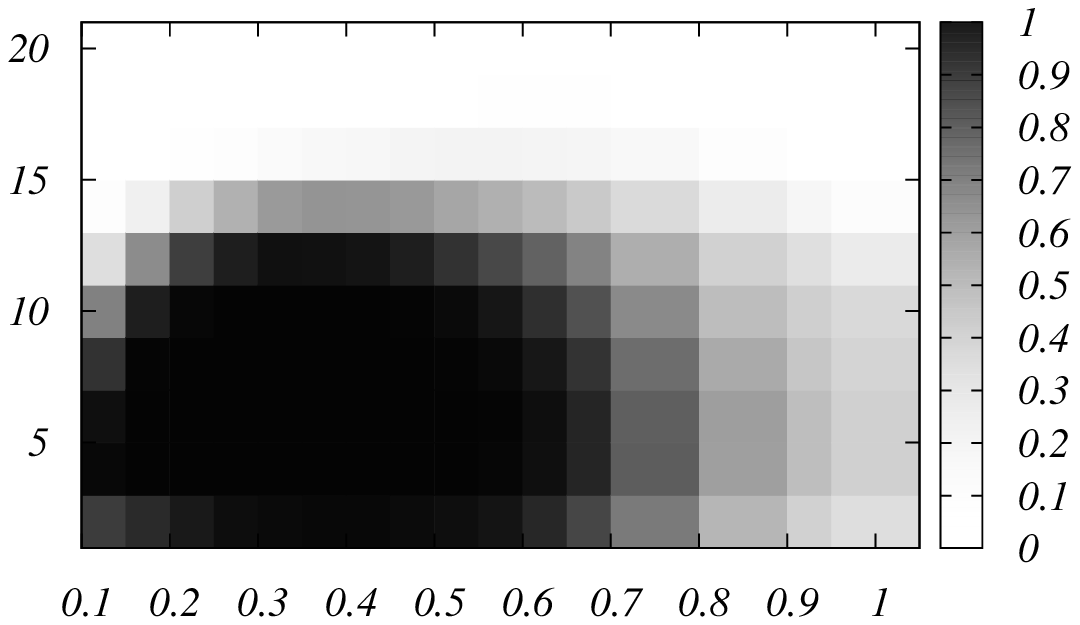}
\hspace{-0.1cm}\includegraphics[width=0.52\linewidth]{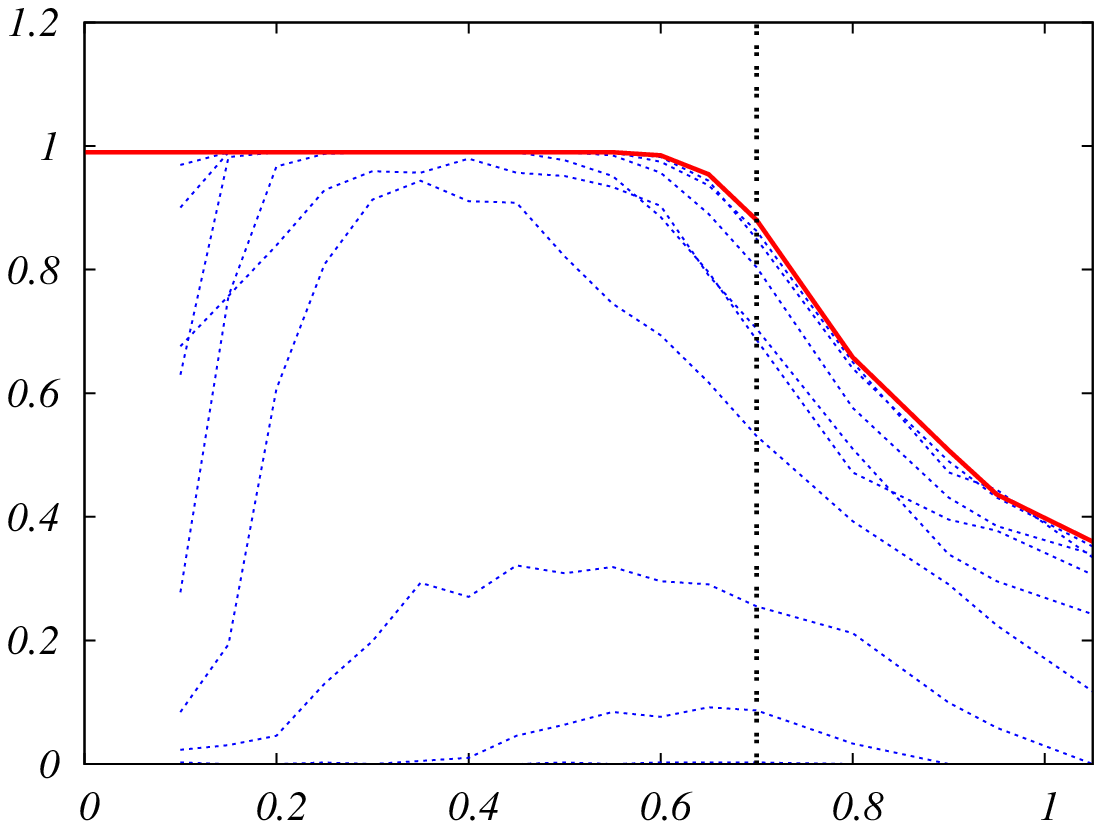}
\put(-525,115){$\lambda_0$}
\put(-390,10){$\theta$}
\put(-120,-5){$\theta$}
\put(-250,100){P${}_{\rm succ}$}
\caption{Learning random subgraphs of a $7\times 7$ ($p=49$) 
two-dimensional grid from $n=4500$ Ising models samples, using regularized 
logistic regression. Left: success probability as a function of
the model parameter $\theta$ and of the regularization parameter $\lambda_0$
(darker corresponds to highest probability). Right: the same data 
plotted for several choices of $\lambda$ versus $\theta$. The vertical line
corresponds to the model critical temperature. The thick line is
an envelope of the curves obtained for different $\lambda$, and 
should correspond to optimal regularization.}\label{fig:Grid}
\end{figure}

Figure reports the success probability of $\rlr(\lambda)$
when applied to random subgraphs of a $7\times 7$
two-dimensional grid. Each such graphs was obtained by removing each edge
independently with probability $\rho = 0.3$. Success probability was
estimated by applying $\rlr(\lambda)$ to each vertex of $8$ graphs 
(thus averaging over $392$ runs of $\rlr(\lambda)$), using $n=4500$ samples.
We scaled the regularization  parameter as 
$\lambda = 2\lambda_0 \theta (\log p/n)^{1/2}$ (this choice is motivated by 
the algorithm analysis \cite{martin} and is empirically the most satisfactory),
and searched over $\lambda_0$.

The data clearly illustrate the phenomenon discussed in the previous pages.
Despite the large number of samples $n\gg \log p$,
when $\theta$ crosses a threshold, the algorithm
starts performing poorly irrespective of $\lambda$.
Intriguingly, this threshold is not far from the critical point of
the Ising model on a randomly diluted grid $\theta_{\rm crit}(\rho=0.3)
\approx 0.7$ \cite{Zobin,fisher}. 

\begin{figure}
\phantom{a}\hspace{-0.5cm}\includegraphics[width=0.52\linewidth]{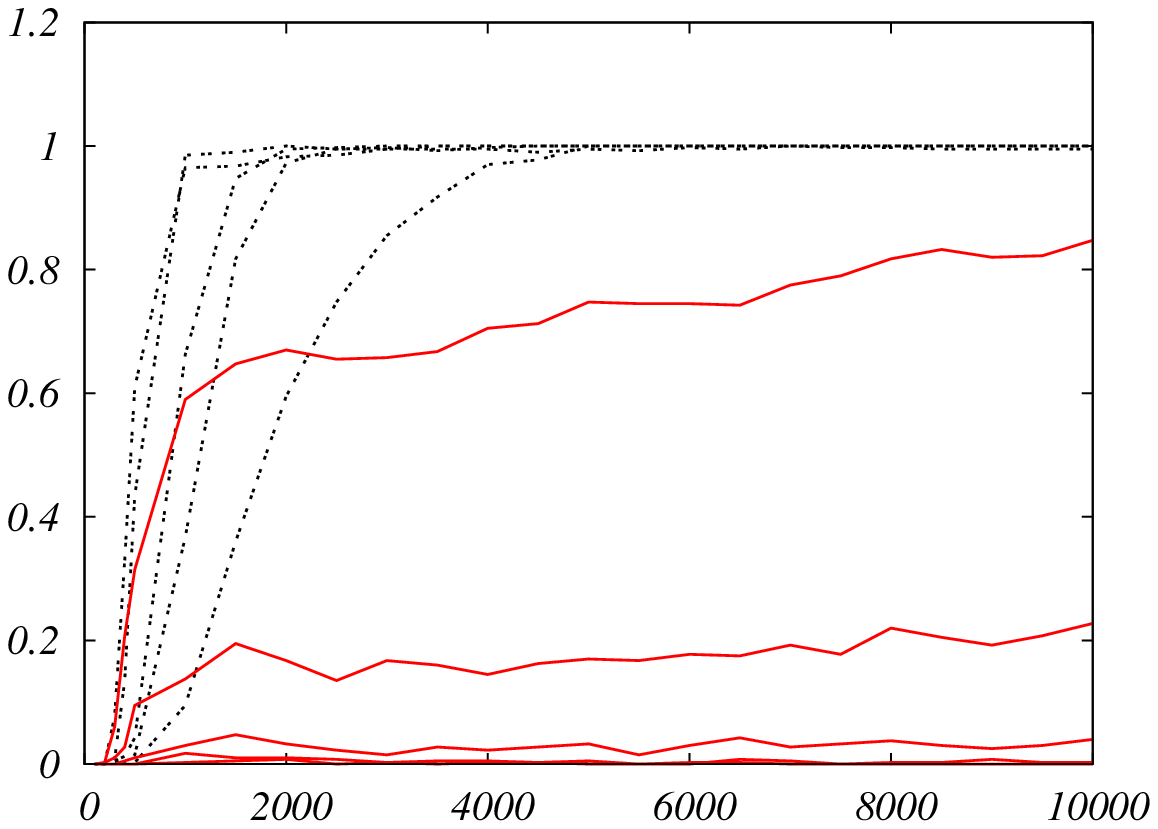}
\includegraphics[width=0.52\linewidth]{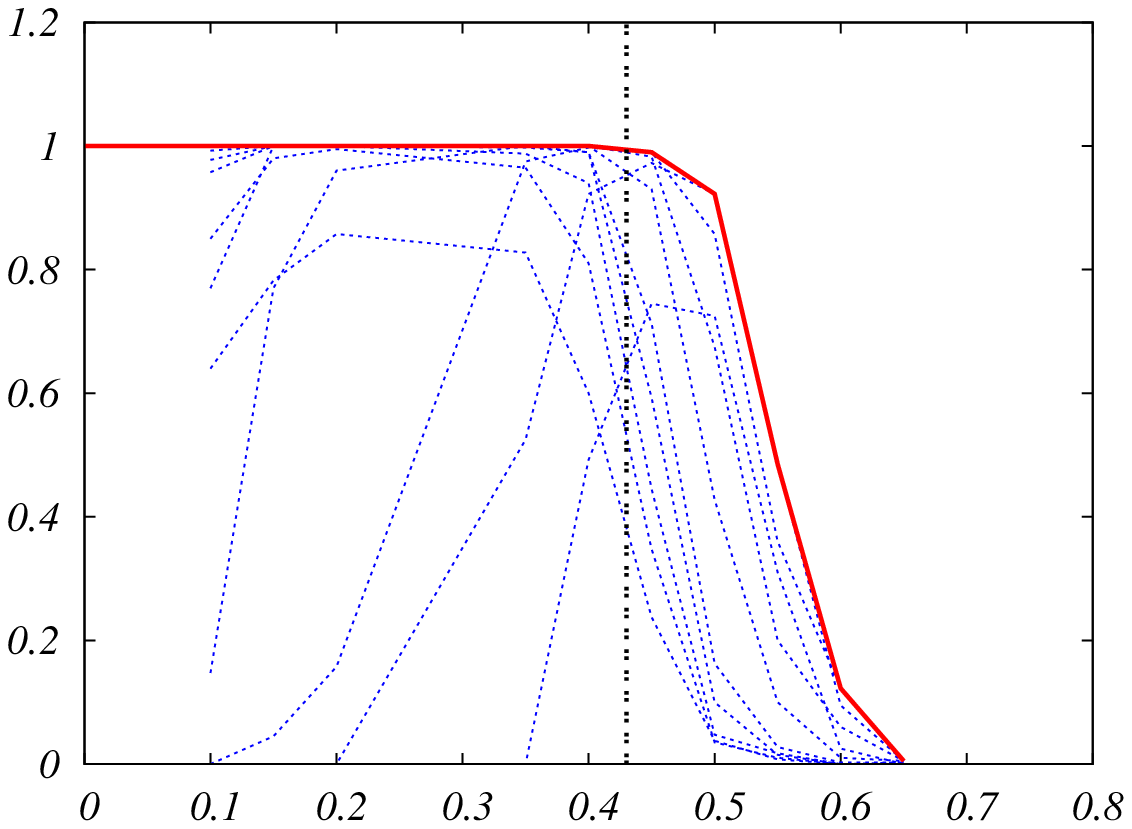}
\put(-500,100){P${}_{\rm succ}$}
\put(-390,125){$\theta = 0.10$}
\put(-435,110){$\theta = 0.15$}
\put(-310,105){$\theta = 0.45$}
\put(-310,32){$\theta = 0.45$}
\put(-360,-5){$n$}
\put(-100,-5){$\theta$}
\put(-250,100){P${}_{\rm succ}$}

\caption{Learning uniformly random graphs of degree $\Delta=4$
 from Ising models samples, using regularized 
logistic regression. Left: success probability as a
function of the number of samples $n$ for several values of $\theta$.
Dotted: $\theta = 0.10$, $0.15$, $0.20$,
$0.35$, $0.40$ (in all these cases 
$\theta<\theta_{\rm thr}(\Delta=4)$). Dashed:  
$\theta = 0.45$, $0.50$, $0.55$, $0.60$, $0.65$
($\theta>\theta_{\rm thr}(4)$,
some of these are indistinguishable from the axis). Right: the same data 
plotted for several choices of $\lambda$ versus $\theta$
as in Fig.~\ref{fig:Grid}, right panel.}\label{fig:RandomG}
\end{figure}

Figure \ref{fig:RandomG} presents similar data when $G$ 
is a uniformly random graph of degree  $\Delta=4$, over
$p=50$ vertices. The evolution of the success probability 
with $n$ clearly shows a dichotomy. When $\theta$ is below a 
threshold, a small number of samples is sufficient to reconstruct $G$ 
with high probability. Above the threshold even $n=10^4$ samples
are to few. In this case we can predict the threshold analytically,
cf. Lemma \ref{th:mart23} below, and get 
$\theta_{\rm thr}(\Delta=4)\approx 0.4203$, which compares favorably 
with the data.
%
%

\section{Proofs} \label{sec:proofs_main}

\subsection{Notation and important remarks} \label{sec:notation_remarks}

Before proceeding it is convenient to introduce some notation and make
some important remarks. If $V$ is a matrix and $R$ is an index set then 
$V_R$ denotes the vector formed by all entries whose index lies in $R$
 and similarly, if $M$ is a matrix and 
$R,P$ are index sets then $M_{R\;P}$ denotes the submatrix 
with  row indices in $R$ and column indices in $P$.
As before, we let $r$ be the vertex whose neighborhood we are trying to 
reconstruct and define $S=\dr$ and $S^c = V \setminus{\dr \cup r}$.  
Since the cost function $L (\utheta; \{\ux^{(\ell)}\})+\lambda\|\utheta\|_1$
only depends on $\utheta$ through its components $\utheta_{r,\cdot}=
\{\theta_{rj}\}$,
we will hereafter neglect all the other parameters and write
$\utheta$ as a shorthand of $\utheta_{r,\cdot}$.

Let $\hz^*$  be a subgradient of $\|\utheta\|_1$  evaluated at the 
true parameters values, $\utheta^* = \{\theta_{rj}: \theta_{ij} = 0, \; 
\forall j \notin \dr, \theta_{rj} = \theta, \; \forall j \in \dr\}$. 
Let $\hutheta^n$ be the parameter estimate returned
by $\rlr(\lambda)$ when the number of samples is $n$. 
Note that, since we assumed
$\utheta^*\ge 0$, we have $\utheta_S^* > 0$ and hence $\hz^*_S = \mathbbm{1}$. Define 
$Q^n(\utheta;\{\ux^{(\ell)} \})$ to be the Hessian 
of $L (\utheta; \{x^{(\ell)}\})$ and 
$Q(\utheta) = \lim_{n \rightarrow \infty} Q^n(\utheta;\{\ux^{(\ell)}\})$. 
By the law of large numbers $Q(\utheta)$ exists a.s. and is the Hessian
of  $\E_{G,\utheta}\log  \prob_{G,\utheta} (X_r| X_{\backslash r})$ where 
$\E_{G,\utheta}$ is the expectation with respect to (\ref{eq:IsingModel2}) and 
$\uX$ is a random variable distributed according to (\ref{eq:IsingModel2}). 
It is convenient to recall here the expressions for the Hessian and
gradient of $L$ for finite $n$ and in the limit
when $n \rightarrow \infty$. For all $i,j \in V \backslash \{r\}$ we have,
\begin{align}
Q^n_{ij}(\utheta) &=  \frac{1}{n} \sum^n_{\ell = 1} \frac{x^{(\ell)}_i x^{(\ell)}_j}{\cosh^2 \left( \sum_{t \in V \backslash \{r\}} \theta_{rt} x^{(\ell)}_t \right) }, \label{eq:Hessian_expression}\\
Q_{ij}(\utheta) &= \E_{G,\utheta^*} \left( \frac{X_i X_j}{\cosh^2(\sum_{t \in V \backslash \{r\}} \theta_{rt} X_t) } \right),\label{eq:Hessian_expression_2}\\
[\nabla L^n(\utheta)]_i &= \frac{1}{n} \sum^n_{\ell = 1} x^{(\ell)}_i \Big(\tanh \Big(\sum_{t \in V \backslash \{r\}} \theta_{rt} x^{(\ell)}_t \Big) -  x^{(\ell)}_r\Big),\\
[\nabla L(\utheta)]_i &= \E_{G,\utheta^*} \Big\{ X_i \tanh \Big(\sum_{t \in V \backslash \{r\}} \theta_{rt} X_t  \Big) \Big\} - \E_{G,\utheta^*} \{ X_i X_r \} \label{eq:grad_expression}.
\end{align}
Note that from the last expression it follows that $\nabla L(\utheta^*) = 0$. 

We will denote the maximum and minimum eigenvalue of a symmetric matrix 
$M$ by $\sigma_{\rm max}(M)$ and $\sigma_{\rm min}(M)$ respectively.
Recall that $\|M\|_{\infty} = \max_i \sum_j |M_{ij}|$.

We will omit arguments whenever clear from the context. 
Any quantity evaluated at the true parameter values will be represented 
with a $^*$, e.g. $Q^* = Q(\theta^*)$. Quantities under a 
$\wedge$ depend on $n$. When clear from the context and since all the examples that we work on have $\theta_{ij} \in \{ 0, \theta$ \}, we will write $\E_{G,\theta}$ as $\E_{G,\utheta}$ or even simply $\E$. Similarly, $\prob_{G,\utheta}$ will be sometimes written as
simply $\prob_{G,\theta}$ or just $\prob$. A subscript $n$ under $\prob_{G,\utheta}$, i.e. $\prob_{n,G,\utheta}$, will be introduced to denote the product measure formed by $n$ copies of model \eqref{eq:IsingModel2}.
Through out this section $\psucc$ will denote the probability of success of a given algorithm, that is,
the probability that the algorithm is able to recover the underlying $G$ exactly.

Throughout this section $G$ is a graph of maximum degree $\Delta$.

%
%

\subsection{Simple Thresholding} \label{sec:simp_thres_proofs}

In the following we let $C_{ij}\equiv\E_{G,\theta}\{X_iX_j\}$ 
where expectation is taken with 
respect to the Ising model \eqref{eq:IsingModel}.

\begin{proof}(Theorem \ref{th:tresh1} )
If $G$ is a tree then $C_{ij}= \tanh \theta$ 
for all $(ij) \in E$ and
$C_{ij} \leq \tanh ^2 \theta$ for all $ (ij) \notin E$. To see this 
notice that only paths that connect $i$ to $j$ contribute to $C_{ij}$ and
given that $G$ is a tree there is only one such path and its length is exactly 1 
if $(i,j) \in E$ and at least 2 when $(i,j) \notin E$.
The probability that $\thres(\tau)$ fails is
\begin{eqnarray}
1-\psucc= 
\prob_{n,G,\theta}\{ \eC_{ij} < \tau \mbox{ for some } (i,j) \in E \; {\rm or} \; 
\eC_{ij} \geq \tau \;\mbox{ for some }\; (i,j) \notin E \}\, .
\end{eqnarray}
Let $\tau = (\tanh \theta + \tanh ^2 \theta ) /2$. Applying 
Azuma-Hoeffding inequality to $\eC_{ij} = \frac{1}{n} \sum^n_{\ell = 1} x^{(\ell)}_i x^{(\ell)}_j$ 
we have that if $(i,j) \in E$ then,
\begin{equation}
\prob_{n,G,\theta}(\eC_{ij} < \tau) =  \prob_{n,G,\theta}\left( \sum^n_{\ell = 1} 
(x^{(\ell)}_i x^{(\ell)}_j - C_{ij})< n(\tau - \tanh \theta) \right) \leq e^{-\frac{1}{32} 
n (\tanh \theta - \tanh ^2 \theta )^2}
\end{equation}
and if $(i,j) \notin E$ then similarly,
\begin{equation}
\prob_{n,G,\theta}(\eC_{ij} \geq \tau) =  \prob_{n,G,\theta}\left( \sum^n_{\ell = 1} 
(x^{(\ell)}_i x^{(\ell)}_j - C_{ij})\geq n(\tau - \tanh^2 \theta) \right) \leq e^{-\frac{1}{32} 
n (\tanh \theta - \tanh ^2 \theta )^2}.
\end{equation}
Applying union bound over the two possibilities, $(i,j) \in E$ or $(i,j) \notin E$, and over the edges ($|E| < p^2/2$), 
we can bound $\psucc$ by
\begin{eqnarray}
\psucc\ge 1-p^2\, e^{-\frac{1}{32} n (\tanh \theta - \tanh ^2 \theta )^2 }\, .
\end{eqnarray}
Imposing the right hand side to be larger than $\delta$ proves our result.
\end{proof}

\begin{proof}(Theorem \ref{th:tresh2})
We will prove that, for $\theta < \arctanh (1/(2\Delta))$, 
$C_{ij} \geq \tanh \theta$ for all $(i,j) \in E$ and 
$C_{ij} \leq 1/(2\Delta)$ for all $(ij) \notin E$.
In particular  $C_{ij} < C_{kl}$ for all $(i,j) \notin E$ and all 
$(k,l) \in E$ .
The theorem follows from this fact via union bound and 
Azuma-Hoeffding inequality as in the proof of Theorem \ref{th:tresh1}.

The bound $C_{ij} \geq \tanh \theta$ for $(ij) \in E$ is a direct 
consequence of Griffiths inequality \cite{griffiths1967} : compare the expectation of
$x_ix_j$ in $G$ with the same expectation in the graph 
that only includes edge $(i,j)$. 

The second bound is derived using the technique of \cite{fisher}, i.e., 
bound $C_{ij}$ by the generating function for self-avoiding walks on the
 graphs from $i$ to $j$. More precisely, assume $l=\dist(i,j)$
and denote by  $N_{ij}(k)$  the number of self avoiding walks 
of length $k$ between 
 $i$ and $j$ on $G$. Then \cite{fisher} proves that
\begin{eqnarray}
C_{ij}  \leq \sum^\infty_{k = l} (\tanh \theta)^k N_{ij}(k) \leq \sum^\infty_{n = l} \Delta^{k-1} (\tanh \theta)^k \leq \frac{\Delta^{l-1} (\tanh \theta)^l }{1 - \Delta \tanh \theta} \leq \frac{\Delta (\tanh \theta)^2 }{1 - \Delta \tanh \theta}. \label{bound_corr_high_temp}
\end{eqnarray}
If $\theta < \arctanh (1/(2 \Delta))$ the above implies
$C_{ij}  \leq 1/(2 \Delta)$ which is our claim.
\end{proof}

\begin{proof}(Theorem \ref{th:tresh3})
The theorem is proved by constructing $G$ as follows:
sample a uniformly 
random regular graph of degree $\Delta$ over the $p-2$ vertices
$\{1,2,\dots,p-2\}\equiv[p-2]$.
Add an extra edge between nodes $p-1$ and $p$.
The resulting graph is not connected. We claim that for $\theta > 
K/\Delta$ and with probability converging to $1$ as 
$p\to\infty$,  there exist $i,j \in [p-2]$ such that 
$(i,j)\notin E$ and $C_{ij} > C_{p-1,p}$. As a consequence, thresholding fails. 

Obviously $C_{p-1,p}=\tanh\theta$. Choose $i\in[p-2]$ uniformly 
at random, and $j$ a node at a fixed distance $t$ from $i$.
We can compute $C_{ij}$ as $p\to\infty$ using the same
local weak convergence result as in the proof of Lemma \ref{th:mart23}.
Namely, $C_{ij}$ converges to the correlation between the root
and a leaf node in the tree Ising model \eqref{IsingTree}. 
In particular one can show, \cite{montmosselsly}, that
\begin{eqnarray}
\lim_{p\to\infty}C_{ij}\ge m(\theta)^2\, ,
\end{eqnarray}
where $m(\theta)=\tanh(\Delta h^*/(\Delta-1))$ and $h^*$
is the unique positive solution of Eq. \eqref{eq:ising_field_fix_point}.

The proof is completed by showing that $\tanh\theta < m(\theta)^2$
for all $\theta>K/\Delta$.
\end{proof}

%
%

\subsection{Proof of Theorem \ref{th:mart2}: failure of regularized logistic regression}\label{sec:ProofMainTheorem}

In order to prove Theorem \ref{th:mart2}, we need a few auxiliary results.
Our first auxiliary results establishes that, 
if $\lambda$ is small, then 
$\|Q^*_{S^c S} {Q^*_{SS}}^{-1} \hz^*_S\|_{\infty}>1$ is a sufficient condition
for the failure of $\rlr(\lambda)$. We recall here that
the subgradient of $\|\utheta \|_1$ evaluated at $\utheta^*$,that is $\hz^*$, satisfies
$\hz^*_S = \ident$.
\begin{lemma}
Assume $[ Q^*_{S^c S} {Q^*_{SS}}^{-1} \hz^*_S]_i \geq 1 + \epsilon$ for some 
 $\epsilon>0$ and some row $i\in V$, 
 $\sigma_{\rm min}(Q^*_{SS}) \ge  C_{\rm min} > 0$, and 
$\lambda < C^3_{\min} \epsilon / (2^7 (1+\epsilon^2) \Delta^3)$.
Then the success probability of $\rlr(\lambda)$ is upper bounded 
as
\begin{equation}
\psucc\le 4 \Delta^2  e^{-n \delta_A^2} +
 4\Delta\, e^{-n \lambda^2 \delta_B^2}
\label{th:mart21prob}
\end{equation}
where $\delta_A = (C_{\rm min}^2/32 \Delta) \epsilon$ and 
$\delta_B = (C_{\rm min} /64 \sqrt{\Delta}) \epsilon$.
\label{th:mart21}
\end{lemma}

The next Lemma implies that, for $\lambda$ to be `reasonable'
(in the sense introduced in Section \ref{sec:Pseudo}),
$n\lambda^2$ must be unbounded with respect to $p$. In fact, by this lemma,
if we choose $n$ to be very large and choose a sequence of star graphs
of increasing number of nodes but with only one edge between the central node and
the remaining nodes,
then, unless $K$ is increasing with $p$,
$\rlr(\lambda)$ will fail to reconstruct the graph with a probability 
greater than $1/2$, which is a contradiction if $\lambda$ is `reasonable'.
\begin{lemma}
There exist $M = M(K,\theta)>0$ decreasing with $K$ for $\theta>0$
such that the following is true: 
If $G$ is the (star) graph with vertex set $V = [p]$ and edge set $E = \{(r,i)\}$ (e.g. $r = 1$, $i = 2$)
and  $n \lambda^2 \leq K$, then 
\begin{equation}
\psucc\le  e^{-M(K,\theta)p}+ e^{-n(1-\tanh\theta)^2/32}\, .
\end{equation}
\label{th:mart22}
\end{lemma}

Finally, our key result shows that the condition 
$\|Q^*_{S^c S} {Q^*_{SS}}^{-1} \hz^*_S\|_{\infty}\le 1$ is violated
with high probability for large random graphs.
The proof of this result relies on a local weak convergence result
for ferromagnetic Ising models on random graphs proved in \cite{andrea}.
\begin{lemma} \label{th:key_result_reg_graph}
Let $G$ be a uniformly random regular graph of degree $\Delta>3$.
Then, there exists $\theta_{\rm thr}(\Delta)$
such that, for $\theta>\theta_{\rm thr}(\Delta)$,
 $\| Q^*_{S^c S} {Q^*_{SS}}^{-1}\hz^*_S\|_{\infty} \ge 1 + \epsilon(\theta,\Delta)$ with 
probability converging to $1$ as  $p\to\infty$ ($\epsilon(\theta,\Delta) > 0$ and $\epsilon(\theta,\Delta) \rightarrow 0$
as $\theta \rightarrow \infty$).

Furthermore, for large $\Delta$, 
$\theta_{\rm thr}(\Delta) = \tilde{\theta}\, \Delta^{-1}(1+o(1))$.
The constant $\tilde{\theta}$ is given by $\tilde{\theta} = h^2_{\infty}$ and $h_{\infty}$ is the unique positive solution of
\begin{equation}
h_{\infty} \tanh h_{\infty} = 1.
\end{equation} \label{th:mart23}
Finally, there exist $C_{\rm min}>0$ dependent only on $\Delta$
and $\theta$ such that $\sigma_{\rm min}(Q^*_{SS})\ge C_{\rm min}$
 with 
probability converging to $1$ as  $p\to\infty$. 
\end{lemma}
The proofs of Lemmas \ref{th:mart21}, \ref{th:mart22} and \ref{th:mart23} are 
sketched in the next subsection.

\begin{proof}(Theorem \ref{th:mart2})
Fix $\Delta>3$, $\theta > K / \Delta$ (where $K$ is a large enough 
constant independent of $\Delta$), and $\eps,C_{\rm min}>0$ and both small enough. 
By Lemma \ref{th:mart23},
for any $p$ large enough we can choose a $\Delta$-regular
graph $G_p=(V=[p],E_p)$ and vertex $r\in V$ such that 
$|Q^*_{S^c S} {Q^*_{SS}}^{-1} \hz^*_S|_i > 1 + \epsilon$ for 
some $i\in V\setminus r$ (Indeed most vertices $r$ and graphs $G_p$ will work).

By Theorem 1 in \cite{mossel} we can assume without loss of generality
$n>K'\Delta\log p$ for some small constant $K'$. Further by
Lemma \ref{th:mart22}, $n\lambda^2\ge F(p)$ for some $F(p)\uparrow\infty$
as $p\to\infty$ and the condition of Lemma \ref{th:mart21} on $\lambda$
is satisfied since by the assumption that $\lambda$ is `reasonable' we have
$\lambda \to 0$ as $n \rightarrow \infty$.
Using these results in Eq.~(\ref{th:mart21prob}) 
of Lemma \ref{th:mart21} we get the following upper bound on the success probability
\begin{equation}
\psucc(G_p)\le 4 \Delta^2  p^{-\delta_A^2K'\Delta} +
 2\Delta\, e^{- F(p) \delta_B^2}\, .
\end{equation}
In particular $\psucc(G_p)\to 0$ as $p\to\infty$.
\end{proof}

%
%
\subsubsection{Proofs of auxiliary lemmas}

\begin{proof}(Lemma \ref{th:mart21})
This proof follows closely the proof of Proposition 1 in \cite{martin}.
For a matter of clarify of exposition we will include all the steps,
even if these do not differ from the exposition done in \cite{martin}.

We will show that (under the assumptions of the Lemma on the Incoherence Condition,
$\sigma_{\min}(Q^*_{SS})$ and $\lambda$) if
$\hat{\utheta} = (\hat{\utheta}_S, \hat{\utheta}_{S^C}) = (\hutheta_S,0)$ 
with $\hutheta_S>0$ then the probability that $\rlr(\lambda)$ returns $\hutheta$
is upper bounded as in Eq.~\eqref{th:mart21prob}. More specifically, we will
show that this $\hutheta$ will not satisfy the stationarity condition
$\nabla L(\hutheta) + \lambda \hat{z}=0$ with high probability for any
subgradient $\hz$ of the function $\|\utheta\|_1$ at $\hutheta$.

To simplify notation we will omit $\{\ux^{(\ell)}\}$ in all the
expressions involving and derived from $L$.

Assume the event $\nabla L(\hutheta) + \lambda \hat{z}=0$ holds for some $\hutheta$
as specified above.
An application of the mean value theorem yields
\begin{equation}
\nabla^2 L (\utheta^*) [\hat{\utheta} - \utheta^*] = W^n  - \lambda \hat{z} - R^n\, ,
\end{equation}
where $W^n = -\nabla L (\utheta^*)$ and $[R^n]_j = [\nabla^2 L (\bar{\utheta}^{(j)}) - \nabla^2 L (\utheta^*)]_j^T (\hat{\utheta} - \utheta^*)$  with $\bar{\utheta}^{(j)}$ a point in the line from $\hat{\utheta}$ to $\utheta^*$. 
Notice that by definition $\nabla^2 L (\utheta^*) = {Q^n}^* = 
Q^n (\utheta^*) $. To simplify notation we will omit the $*$ in all ${Q^n}^*$.
All $Q^n$ in this proof are thus evaluated at $\utheta^*$.

Breaking this expression into its $S$ and $S^C$ components and since 
$\hutheta_{S^C} =\utheta^*_{S^C} = 0$ we can write
\begin{align}
Q^n_{S^C S} (\hutheta_S - \utheta_S^*) &= W^n_{S^C} - \lambda \hz_{S^C} + R_{S^C}^n,\\
Q^n_{S S} (\hutheta_S - \utheta_S^*) &= W^n_{S} - \lambda \hz_{S} + R_{S}^n.
\end{align}
Eliminating $\hutheta_S - \utheta_S^*$ from the two expressions we obtain
\begin{equation}
[W^n_{S^C} - R^n_{S^C}] - Q^n_{S^C S}(Q^n_{SS})^{-1} [W^n_S - R^n_S] + \lambda  Q^n_{S^C S}(Q^n_{SS})^{-1} \hat{z}_S = \lambda \hat{z}_{S^C}\, .
\end{equation}
Now notice that $Q^n_{S^C S}(Q^n_{SS})^{-1} = T_1 + T_2 + T_3 + T_4$ where
\begin{align*}
T_1 & =  Q^*_{S^CS}[(Q^n_{SS})^{-1} - (Q^*_{SS})^{-1}]\, ,
\;\;\;\;\;\;\;\;\;\;\;\;\;\;\;\;\;\;\;\;\;\;\;\;\;
T_2  =  [Q^n_{S^C S} - Q^*_{S^C S}] {Q^*_{SS}}^{-1}\, , \\
T_3 & =  [Q^n_{S^C S} - Q^*_{S^C S}] [(Q^n_{SS})^{-1} - (Q^*_{SS})^{-1}]\, ,
\;\;\;\;\;\;\;\;\;\;
T_4  =  Q^*_{S^CS} {Q^*_{SS}}^{-1}  \, .
\end{align*}
Recalling that $\hz_S = \ident$ and using the above decomposition we can lower bound the absolute 
value of the indexed-$i$ component of $\hat{z}_{S^C}$ by
\begin{align}
|\hz_i|  \geq \|[Q^*_{S^C S} {Q^*_{SS}}^{-1}\hz_S]_i\|_{\infty} - \|T_{1,i} \|_1 - 
\|T_{2,i} \|_1 - \|T_{3,i}\|_1\\
 - \Big| \frac{W^n_i}{\lambda} \Big| - \Big| \frac{R^n_i}{\lambda} 
\Big| - \| [Q^n_{S^C S}(Q^n_{SS})^{-1} ]_i \|  \left( \Big\| 
\frac{W^n_S}{\lambda} \Big\|_\infty + \Big\| \frac{R^n_S}{\lambda}\Big\|_\infty  \right). \nonumber
\end{align}
We will now assume that the samples $\{x^{(\ell)}\}$ are such that 
the following event holds (notice that $i \in S^C$),
\begin{equation}
\cE_i \equiv \Big\{ \|Q^n_{S \cup \{i \} \; S} - Q^*_{S \cup \{i \}\;S}\|_\infty < \xi_A , \Big\|\frac{W^n_{S \cup \{i\} }}{\lambda}\Big\|_\infty < \xi_B \Big\}\, ,
\end{equation}
where $\xi_A \equiv C_{\rm min}^2\eps/(8 \Delta)$ and $\xi_B 
\equiv C_{\rm min}\eps/(16 \sqrt{\Delta})$.

From relations \eqref{eq:Hessian_expression} to \eqref{eq:grad_expression} in Section \ref{sec:notation_remarks} we know that $\mathbb{E}_{G, \theta} (Q^n) = Q^*$, $\E_{G, \theta} (W^n) = 0$ and
that both $Q^n - Q^*$ and $W^n$ are sums i.i.d.
random variables bounded by 2. From this, a simple
application of Azuma-Hoeffding inequality yields
\footnote{For full details see the proof of Lemma 2 and the discussion following Lemma 6 in \cite{martin}}.
\begin{align}
\prob_{n,G,\theta}(|Q_{ij}^n - Q_{ij}^*| > \delta) \leq 2 e^{-\frac{\delta^2 n}{8}},\\
\prob_{n,G,\theta}(|W_{ij}^n | > \delta) \leq 2 e^{-\frac{\delta^2 n}{8}},
\end{align}
for all $i$ and $j$.
Applying union bound we conclude that the event $\cE_i$ holds with probability
at least
\begin{align}
1 - 2 \Delta (\Delta + 1) e^{- \frac{n \xi^2_A}{8}} - 2 (\Delta + 1) e^{- \frac{n \lambda^2 \xi^2_B}{8}}
\geq 1 - 4 \Delta^2 e^{-n \delta^2_A} - 4 \Delta e^{- n \lambda^2 \delta_B },
\end{align} 
where $\delta_A = C^2_{\min} \epsilon / (32 \Delta)$ and $\delta_B = C_{\min} \epsilon / (64 \sqrt{\Delta})$.

If the event $\cE_i$ holds then
$\sigma_{\rm min}(Q^n_{SS}) > \sigma_{\rm min}(Q^*_{SS}) - C_{\rm min}/2 > C_{\rm min}/2$.
Since $\|[Q^n_{S^C S}(Q^n_{SS})^{-1}]_i\|_{\infty} \leq \|{Q^n_{SS}}^{-1}\|_2 \|Q^n_{S i}\|_2$
and $|Q^n_{j i}| \leq 1 \forall i,j$
we can write $\|[Q^n_{S^C S}(Q^n_{SS})^{-1}]_i\|_{\infty} \leq 2 \sqrt{\Delta} / C_{\min}$
and simplify our lower bound to
\begin{align}
|\hz_i|  \geq \|[Q^*_{S^C S} {Q^*_{SS}}^{-1}\hz_S]_i\|_{\infty} - \|T_{1,i} \|_1 - 
\|T_{2,i} \|_1 - \|T_{3,i}\|_1\\
 - \Big| \frac{W^n_i}{\lambda} \Big| - \Big| \frac{R^n_i}{\lambda} 
\Big| - \frac{2 \sqrt{\Delta} }{C_{\min}}  \left( \Big\| 
\frac{W^n_S}{\lambda} \Big\|_\infty + \Big\| \frac{R^n_S}{\lambda}\Big\|_\infty  \right). \nonumber
\end{align}

The proof is completed by showing that the event 
$\cE_i$ and the assumptions of the theorem 
imply that each of last $7$ terms in this expression is smaller than 
$\epsilon/8$. Since $|[Q^*_{S^C S} {Q^*_{SS}}^{-1}]_i^T \hat{z}^n_S|\ge
1+\eps$ by assumption, this implies  $|\hat{z}_i|\ge 1 + \epsilon/8 > 1$ 
which cannot be true since any subgradient of the $1$-norm has components of 
magnitude at most $1$. 

Taking into account that $\sigma_{\min}(Q^*_{SS}) \leq \max_{ij} Q^*_{ij} \leq 1$ and that $\Delta > 1$,
the last condition on $\cE_i$ immediately bounds all terms involving $W^n$ by 
$\epsilon/8$.
Some straightforward manipulations imply (see Lemma 7 from \cite{martin} for a similar computation)
\begin{align*}
\|T_{1,i} \|_1 & \leq \frac{\Delta}{C_{\rm min}^2} \|Q^n_{SS} - Q^*_{SS}\|_\infty \, ,\;\;\;\;\;\;\;\;\;
\|T_{2,i} \|_1  \leq \frac{\sqrt{\Delta}}{C_{\rm min}} \|[Q^n_{S^C S} - Q^*_{S^C S}]_i\|_\infty\, , \\
\|T_{3,i} \|_1 & \leq \frac{2 \Delta}{C_{\rm min}^2} \|Q^n_{SS} - Q^*_{SS}\|_\infty   \|[Q^n_{S^C S} - Q^*_{S^C S}]_i\|_\infty \, ,
\end{align*}
and thus, again making use of the fact that  $\sigma_{\min}(Q^*_{SS}) \leq 1$,  all will be bounded by $\epsilon/8$ when $\cE_i$ holds.
The final step of the proof consists in showing that if $\cE_i$ holds and $\lambda$
satisfies the condition given in the Lemma enunciation then the terms involving
$R^n$ will also be bounded above by $\epsilon/8$. The details of this calculation
are included in Appendix \ref{sec:bound_on_R_n}.
\end{proof}

\begin{proof}(Lemma \ref{th:mart23}.)
Let us state explicitly the local weak convergence result mentioned in
Sec.~\ref{sec:ProofMainTheorem} right before our statement of Lemma \ref{th:mart23}.
For $t\in\naturals$, let $\Tree(t)= (V_\Tree,E_{\Tree})$ be the regular rooted tree of degree $\Delta$ of $t$ generations and 
define the associated Ising measure as
\begin{eqnarray}
\mu_{\Tree,\theta}^+(\ux) =\frac{1}{Z_{\Tree,\theta}}\,
\prod_{(i,j)\in E_{\Tree}} e^{\theta x_i x_j}\prod_{i\in \dtree(t)}e^{h^*x_i}\, .
\label{IsingTree}
\end{eqnarray}
Here $\dtree(t)$ is the set of leaves of $\Tree(t)$ and
 $h^*$ is the unique positive solution of
\begin{equation}
h = (\Delta -1)\, \atanh\, \{\tanh \theta \, \tanh  h \}\, . \label{eq:ising_field_fix_point}
\end{equation}
It was proved in \cite{montmosselsly} that non-trivial local expectations
with respect to $\mu_{G,\theta}(\ux)$ converge to local expectations
with respect to $\mu_{\Tree,\theta}^+(\ux)$, as $p\to\infty$.

More precisely, let 
$\Ball_r(t)$ denote a ball of radius $t$ around node $r\in G$
(the node whose neighborhood we are trying to reconstruct).
For any fixed $t$, the probability that $\Ball_r(t)$ is not
isomorphic to $\Tree(t)$ goes to $0$ as $p\to\infty$. 

Let $g(\ux_{\Ball_r(t)})$ be any function of the variables in
$\Ball_r(t)$ such that $g(\ux_{\Ball_r(t)}) = g(-\ux_{\Ball_r(t)})$.
Then  almost surely over graph sequences $G_p$ of uniformly
random regular graphs with $p$ nodes (expectations here are 
taken with respect to the measures (\ref{eq:IsingModel}) and 
(\ref{IsingTree}))
\begin{eqnarray}
\lim_{p\to\infty}\E_{G,\theta} \{g(\uX_{\Ball_r(t)})\}
=\E_{\Tree(t),\theta,+} \{g(\uX_{\Tree(t)})\}\, .\label{eq:Weak}
\end{eqnarray}
Notice that this characterizes expectations completely since if
$g(\ux_{\Ball_r(t)}) = -g(-\ux_{\Ball_r(t)})$
then,
\begin{equation}
\E_{G,\theta} \{g(\uX_{\Ball_r(t)})\} = 0.
\end{equation}
The proof consists in considering 
$[ Q^*_{S^c S} {Q^*_{SS}}^{-1} \hat{z}^*_S]_i$ for $t=\dist(r,i)$ bounded.
We then write  $(Q^*_{SS})_{lk} = \E_{G,\theta} \{ g_{l,k}(\uX_{_{\Ball_r(t)}})\}$
and $(Q^*_{S^cS})_{il} = \E_{G,\theta}\{ g_{i,l}(\uX_{_{\Ball_r(t)}})\}$ for some functions
$g_{\cdot,\cdot}(\uX_{_{\Ball_r(t)}})$ and apply the weak convergence result
(\ref{eq:Weak}) to these expectations. We thus reduced the calculation 
of  $[ Q^*_{S^c S} {Q^*_{SS}}^{-1} \hat{z}^*_S]_i$ to the calculation 
of expectations with respect to the tree measure (\ref{IsingTree}).
The latter can be implemented explicitly through a recursive procedure,
with simplifications arising thanks to the tree symmetry and by
taking $t\gg 1$. The actual calculations consist in a (very) long
exercise in calculus and is deferred to Appendix \ref{sec:rand_reg_incoh_cond}.

The lower bound on $\sigma_{\rm min}(Q^*_{SS})$ is proved by a similar 
calculation.
\end{proof}

\subsection*{Acknowledgments}
This work was partially supported by a Terman fellowship, the NSF
CAREER award CCF-0743978 and the NSF grant DMS-0806211 and by a
Portuguese Doctoral FCT fellowship.

\newpage

\appendix

%
%

\section{Covariance calculation for the toy example} \label{sec:toy_examp_comp}

In this section we compute the covariance matrix for the Ising model
on the graph  $G_p$ introduced within the toy example of Section \ref{sec:toy_example}, see Fig.~\ref{fig:SimpleGraph}.
In fact we only need to compute $\E_{G_p,\theta} \{x_1 x_2\}$, $\E_{G_p,\theta} \{x_1 x_3\}$
and $\E_{G_p,\theta} \{x_3x_4\}$, since all other covariances reduce
to one of these tree by symmetry. First recall that by \cite{fisher} we can write the correlation between
$x_i$ and $x_j$ as follows
\begin{equation}
\E_{G,\theta}\{x_ix_j\}= 
\frac{\sum_{F \in \mathcal{I}(G)} (\tanh \theta)^{|F|} }{\sum_{G \in \mathcal{P}(G)} (\tanh \theta)^{|F|}}, \label{eq:high_temp_exp}
\end{equation}
where:$(i)$ $\mathcal{I}(G)$ is the set of all subsets of edges of graphs of $G$ with
odd number of edges adjacent to node $i$ and $j$ and even number of
edges adjacent to every other node; $(ii)$ $\mathcal{P}(G)$ is the set
of all subsets of edges of $G$ with
even number of edges in all nodes; $(iii)$ 
$|F|$ is the number of edges in $F$. 

Expression \eqref{eq:high_temp_exp} implies three basic facts that we will use to compute the correlations of $G_p$. Some of these observations can be proved in different and maybe simpler ways but for a matter of unity, we will explain them from the point of view of
\eqref{eq:high_temp_exp}.

First, if $i,j$ are two nodes in a graph $G$ and $k,l$ two nodes in a
graph $G'$ and we `glue' $j$ and $k$ together 
(i.e. we fix $x_j = x_k$) to form a new graph $G''$ (see Figure \ref{fig:comp_corr} (a)) then
\begin{equation}
\E_{G'',\theta} \{x_i x_l\} = \E_{G,\theta} \{x_i x_j\} \; \E_{G',\theta} \{x_k x_l\}.
\label{eq:serial_corr}
\end{equation}

Second, if instead we `glue' $i$ with $k$ and $j$ with $l$ (i.e. we fix $x_i = x_k$ and $x_j = x_l$) (see Figure \ref{fig:comp_corr} (b)) then
\begin{align}
\E_{G'',\theta} \{x_i x_j\} &= \frac{\E_{G,\theta} \{x_i x_j\} + \E_{G',\theta} \{x_k x_l\} }{1 + \E_{G,\theta} \{x_i x_j\} \; \E_{G',\theta} \{x_k x_l\}}\\
&= \tanh (\arctanh(\E_{G,\theta} \{x_i x_j\}) + \arctanh(\E_{G',\theta} \{x_k x_l\}) ).
\label{eq:para_coor}
\end{align}
Note that in this second case we are computing $\E_{G'',\theta} \{x_i x_j\}$ and not $\E_{G'',\theta} \{x_i x_l\}$.

Finally, if $G$ is the square graph formed by nodes $\{1, 2, 3, 4\}$ and edge set $\{(1,3),(1,4),(2,3),(2,4)\}$ and $G'$ is some other graph to which nodes $i$ and $j$ belong and we `glue' node $1$ with $i$ and node $2$ with $j$ (i.e. $x_1 = x_i$ and $x_2 = x_j$) to form $G''$ (see Figure \ref{fig:comp_corr} (c)) then
\begin{equation}
\E_{G_{12},\theta} \{x_3 x_4\} = \frac{2 \tanh^2 \theta + 2 \E_{G_{2},\theta} \{x_i x_j\}  \tanh^2 \theta }{1 + \tanh^4 \theta + 2 \E_{G_{2},\theta} \{x_i x_j\} \tanh^2 \theta}.
\label{eq:square_coor}
\end{equation}
\begin{figure}
\begin{center}
\subfloat[`Series` composition]{\includegraphics[width = 0.38\linewidth]{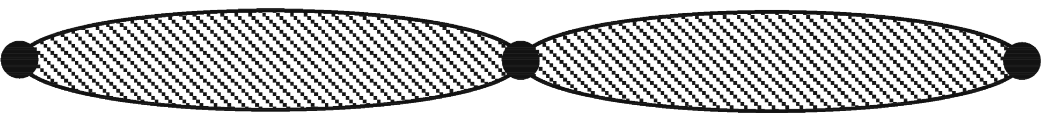}} \;\;\;\;\;\;\;\;
\subfloat[`Parallel' composition]{\includegraphics[width = 0.2\linewidth]{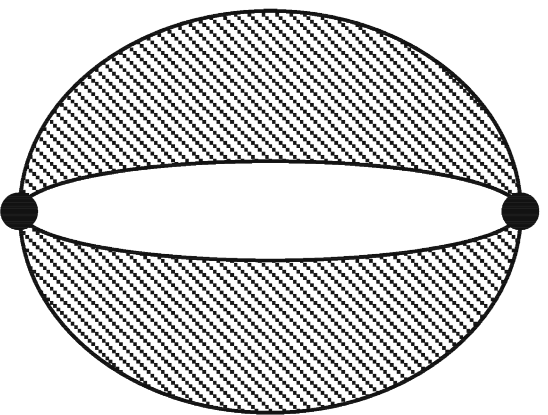}} \;\;\;\;\;\;\;\;
\subfloat['Bridge' composition]{\includegraphics[width = 0.2\linewidth]{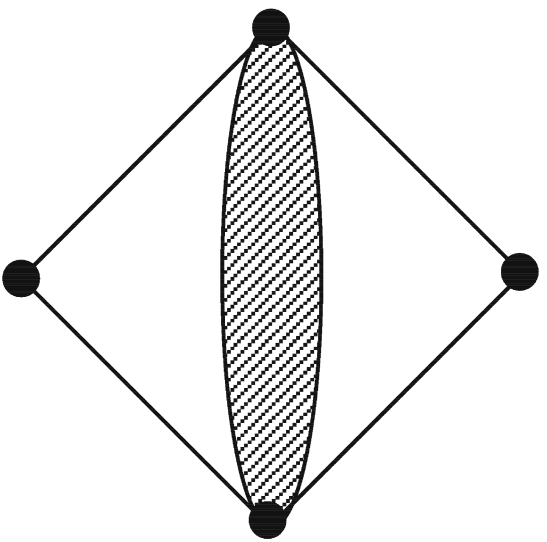}}
\put(-385,7){$G_1$}
\put(-300,7){$G_2$}
\put(-430,15){$i$}
\put(-343,15){$j,k$}
\put(-245,15){$l$}
\put(-230,42){$i,k$}
\put(-127,42){$j,l$}
\put(-175,55){$G_1$}
\put(-175,10){$G_2$}
\put(-102,45){$3$}
\put(-54,97){$1,i$}
\put(2,45){$4$}
\put(-42,-2){$2,j$}
\put(-83,70){$G_1$}
\put(-54,45){$G_2$}
\end{center}
\caption{Correlation for different composite graphs}\label{fig:comp_corr}
\end{figure}

With these three relationships we can quickly compute $\E_{G_p,\theta}
\{x_1 x_2\}$, $\E_{G_p,\theta} \{x_1 x_3\}$ and $\E_{G_p,\theta} \{x_3
x_4\}$. Let $p=3$ and note that from \eqref{eq:serial_corr} we have
that $\E_{G_{3},\theta} \{x_1 x_2\} = \tanh^2 \theta$. Since $G_p$ is
formed by $p-2$ copies of $G_3$ glued 
in `parallel' in between nodes 1 and 2, by \eqref{eq:para_coor} we have that
$\E_{G_{p},\theta} \{x_1 x_2\} = \tanh((p-2) \; \arctanh (\tanh^2 \theta))$. Now notice that $G_p$ can also be seen as a single edge connecting 1 and 3 in 'parallel' with the graph formed by connecting in 'series' the edge $(2,3)$ to a copy of $G_{p-1}$. This tells us that $\E_{G_{p},\theta} \{x_1 x_3\} = \tanh(\theta + \arctanh(\E_{G_{p-1},\theta} \{x_1 x_2\} \tanh \theta))$. Finally, we can also see $G_p$ as a square graph formed by nodes $\{1,2,3,4\}$ and edges $\{(1,3),(1,4),(2,3),(2,4)\}$ to which we add $G_{p-2}$ as a `bridge' in between nodes 1 and 2. Making use of \eqref{eq:square_coor} we get that
\begin{equation}
\E_{G_p,\theta} \{x_3 x_4\} = \frac{2 \tanh^2 \theta + 2 \E_{G_{p-2},\theta} \{x_1 x_2\} \tanh^2 \theta }{1 + \tanh^4 \theta + 2 \E_{G_{p-2},\theta} \{x_1 x_2\} \tanh^2 \theta}.
\end{equation}
From these closed form expressions it is now easy to obtain the behavior of the correlations for the regime $\theta \gg 1$ and $p \ll 1$.

\section{Success of regularized logistic regression for small $\theta$} \label{sec:succ_reg_log_reg}

\begin{proof}(Theorem \ref{th:mart1} )
The proof of this theorem consists in verifying the
conditions of Theorem 1 in \cite{martin}
(denoted there by A1 and A2)
and computing the probability of success as a function
of $n$ with slightly more care.

In what follows, $C_{\rm min}$ is a lower bound for
$\sigma_{\rm min}(Q^*_{SS})$ and $D_{\max}$ \footnote{It is easy to prove that $C_{\min} \leq D_{\max}$} is an upper bound for
\begin{equation}
\sigma_{\max}(\E_{G,\utheta^*}(X_S X^T _S)).
\end{equation}

We define
$1-\alpha \equiv \|Q^*_{S^C S} {Q^*_{SS}}^{-1}\|_{\infty}$ and
$\theta_{\min}$ is the minimum absolute value of the
components of $\utheta^*$. Throughout this proof we will
have $\hat{C}_{\min}$ and $\hat{D}_{\max}$ denote
$\sigma_{\min}({Q^n}^*_{SS})$ and 
$\sigma_{\max} \Big(\frac{1}{n}\sum^n_{l=1}x^{(l)}_S {x^{(l)}}^T_S\Big)$
respectively and $1 - \hat{\alpha} = \|Q^n_{S^CS} {Q^n_{SS}}^{-1} \|_{\infty}$.

Consider the event, $\mathcal{E}$, that the following
conditions hold (these conditions are part of the conditions
required for Theorem 1 in \cite{martin} to be applicable and
are labeled by the names of the theorems that use them that
help proving Theorem 1),
\begin{eqnarray}
\text{In Lemma 5:}&&\|Q^n_{SS} - Q^*_{SS}\|_2 < C_{\min} / 2\, ,\\
\text{In Lemma 6:}&\text{for T1}&\sigma_{\min}({Q^n}^*_{SS}) \leq
C_{\min}/2\, ,\\
&\text{for T1}&\|Q^n_{SS} - Q^*_{SS}\|_{\infty} \leq \frac{1}{12}
\frac{\alpha}{1-\alpha} C_{\min} / \sqrt{\Delta}\, ,\\
&\text{for T2}&\|Q^n_{S^CS} - Q^*_{S^C S}\|_{\infty} \leq
\frac{\alpha}{6} C_{\min} / \sqrt{\Delta}\, ,\\
&\text{for T3}&\|Q^n_{S^CS} - Q^*_{S^C S}\|_{\infty} \leq
\sqrt{\alpha/6}\, ,\\
\text{In Lemma
  7:}&&\sigma_{\min}({Q^n}^*_{SS}) \leq C_{\min}/2 \;
\text{and} \;\|Q^n_{SS} - Q^*_{SS}\|_2 \leq \sqrt{\frac{\alpha}{6}}
\frac{C^2_{\min}}{2 \sqrt{\Delta}}\, ,
\end{eqnarray}
\begin{eqnarray}
\text{In Proposition 1:}&&\frac{\|W^n\|_{\infty}}{\lambda} <
\frac{\hat{\alpha}}{4(2 - \hat{\alpha})}\, ,\\
&&\lambda \Delta \leq \frac{\hat{C}_{\min}^2}{10 \hat{D}_{\max}}\, ,\\
&&\frac{5}{\hat{C}_{\min}} \lambda \sqrt{\Delta} \leq
\frac{\theta_{\min}}{2}\, .
\end{eqnarray}

Note that these conditions imply that
$\hat{C}_{\min} \geq C_{\min}/2$, $\hat{D}_{\max} \leq 2 D_{\max}$
and also, from the proof of Proposition 1 in \cite{martin}, they imply that without
loss of generality we can assume $\hat{\alpha} = \alpha / 2$.
Since the assumption $\sigma_{\min}({Q^n}^*_{SS}) \leq C_{\min}/2$ follows from the
assumption of Lemma 5 in \cite{martin}, all assumptions are in fact assumptions on the
proximity, under different norms, of empirical vectors and matrices to their correspondent mean values.

Having the definition of $\mathcal{E}$ in mind we beging by noting that Theorem 1 can be rewritten in the following form.
\begin{thm}
If $\lambda \leq (C_{\min}/2)^2 / (20 \Delta D_{\max})$, $\lambda \leq (C_{\min}/2) \theta_{\min} / (20 \sqrt{\Delta})$, $1- \alpha < 1$ and $\mathcal{E}$ holds then \rlr \;
will not fail.\label{th:refor_theom_1_martin}
\end{thm}
A straightforward application of Azuma's inequality yields the following upper bound on the probability of
these assumptions not occurring together,
(the first three terms are for the conditions involving matrix $Q$
and the fourth with the event dealing with matrix $W$),
\begin{align} \label{proof:theorem_1_matrin_prob}
\prob_{n,G,\theta}(\mathcal{E}^c) & \leq 2e^{-n\frac{1}{32 \Delta^2}
  (d^{(2)}_{SS})^2 + 2 \log \Delta}+ 2e^{-n\frac{1}{32 \Delta^2}
  (d^{(\infty)}_{SS})^2+ 2 \log \Delta} \\ 
&+ 2e^{-n\frac{1}{32 \Delta^2}
  (d^{(\infty)}_{S^C S})^2 + \log \Delta + \log p - \Delta} 
+ 2e^{-n \frac{\lambda^2}{2^7} (\frac{\hat{\alpha}}{2-\hat{\alpha}})^2
  + \log p} \, ,\nonumber
\end{align}
where
\begin{eqnarray}
d^{(2)}_{SS} &=& \min \Big\{\frac{C_{\min}}{2}, \sqrt{\frac{\alpha}{6}} \frac{C^2_{\min}}{2 \sqrt{\Delta}} \Big\},\\
d^{(\infty)}_{SS} &=& \frac{1}{12} \frac{\alpha}{1-\alpha} \frac{C_{\min}}{\sqrt{\Delta}},\\
d^{(\infty)}_{S^C S} &=& \min \Big\{ \frac{\alpha}{6} \frac{C_{\min}}{\sqrt{\Delta}},\sqrt{\frac{\alpha}{6}}\Big\}.
\end{eqnarray}
Under the assumption that $\theta \leq K_1 / \Delta$ for $K_1$ small enough we now calculate lower bounds
for $C_{\min}$ and $\alpha$ and upper bound for $D_{\max}$ which will
allow us to verify the condition of Theorem \ref{th:refor_theom_1_martin} and
simplify expression for the upper bound on $ \prob_{n,G,\theta}(\mathcal{E}^c)$.

First notice that by \eqref{eq:Hessian_expression_2} we have $C_{\min} = \sigma_{\min} \{ \E_{G,\utheta^*} ((1 - \tanh^2 \theta M)  X_S X^T_S) \}$ where $M = \sum_{t \in \partial r} X_t$.
Since $\theta M \leq \theta \Delta \leq K_1$ and because
$\sigma_{\min}(AB) \geq \sigma_{\min}(A) \sigma_{\min}(B)$ we have,
$C_{\min} \geq (1-K^2_1) \sigma_{\min} (\E_{G,\utheta^*} \{ X_S X^T_S\})$. Now write
$\E_{G,\utheta^*} \{ X_S X^T_S\} \equiv \ident + Q$ and notice that by \eqref{bound_corr_high_temp} 
$Q$ is a symmetric matrix whose values are non-negative and smaller than $\tanh \theta/(1 - \Delta \tanh \theta)$.
Since $\sigma_{\min}(\E_{G,\utheta^*} \{ X_S X^T_S\}) = 1 - v^T (-Q) v$ for some unit norm vector $v$ and since, by Cauchy--Schwarz inequality, we have $v^T (-Q) v \leq \|v\|^2_1 \max_{ij} |Q_{ij}| \leq \Delta \tanh \theta / (1 - \Delta \tanh \theta) \leq K_1 / (1 - K_1)$, it follows that
$\sigma_{\min}(\E_{G,\utheta^*} \{ X_S X^T_S\}) \geq (1 - 2 K_1)/(1 - K_1)$.
Consequently, $C_{\min} \geq (1 + K_1)(1-2K_1)$.
With the bound \eqref{bound_corr_high_temp}, and again for $K_1$ small, we can write $D_{\max} \leq 1 +  \Delta \tanh \theta / (1 - \Delta \tanh \theta) \leq (1 - K_1)^{-1}.$ A similar calculation yields $1 - \alpha \leq K_1/((1-K^2_1)(1-2 K_1))$.

For $K_1$ small enough, and looking at the bounds just obtained for $C_{\min}$ and $D_{\min}$, the restriction on $\lambda$ in Theorem \ref{th:refor_theom_1_martin}, namely
\begin{equation}
\lambda \leq C_{\min}/40 \sqrt{\Delta} \min \{ \theta, C_{\min}/40 D_{\max} \sqrt{\Delta}\},
\end{equation}
can be simplified to $C_{\min} \theta / 40 \sqrt{\Delta}$. Choosing $\lambda = K_3 \theta \Delta^{-1/2}$ it
is easy to see we can simplify the expression for the probability upper bound and write 
\begin{eqnarray}
1 - \prob_{n,G,\theta}(\mathcal{E}) \leq 8e^{-n K_2^{-1} \theta^2 \Delta^{-1}  + 2 \log p}
\end{eqnarray}
for some constant $K_2$ which in turn implies the bound on $n_\rlr(\lambda)$.
\end{proof}

\section{Failure of regularized logistic regression at large $\theta$}

\subsection{Bound on terms involving $R^n$} \label{sec:bound_on_R_n}

\begin{proof}(Lemma \ref{th:mart21})
We outline here the upper bound on the term $R^n$.
Note that we are omitting the samples $\{x^{(i)}\}$ in the argument of
function $L$ and we are representing $\utheta_{r,.}$ by $\utheta$.
This proof is just a replica and fusion of Lemmas 3 and 4 in \cite{martin}.
Through out this proof we have $\hutheta_{S^C} = \utheta^*_{S^C} = 0$.

First we write,
\begin{align}
R^n_{j} & = [\nabla^2 L (\bar{\utheta}^{(j)} ) - \nabla^2 L (\utheta^* ) ]^T_j [\hutheta - \utheta^*] \\
&= \frac{1}{n} \sum^n_{i = 1} [\eta(\bar{\utheta}^{(j)}) - \eta(\utheta^*) ][x^{(i)} \, {x^{(i)}}^T]^T_j [\hutheta - \utheta^*]
\end{align}
for some point $\bar{\utheta}^{(j)}$ lying in the line
between $\hutheta$ and $\utheta^*$,i.e.
$\bar{\utheta}^{(j)} = t_j \hutheta + (1 - t_j) \utheta^*$.
Since $\eta (\utheta) = g (x_r^{(i)} \sum_{t \in V \backslash r}{\theta_{rt} x_t^{(i)}}) = g(x_r^{(i)} \utheta^T \underline{x}^{(i)} ) = g(\utheta^T \underline{x}^{(i)})$
where $g(s) = 4 e^{2s}/(1 + e^{2s})^2$ another application of the chain rule yields,
\begin{align}
R^n_{j} & = \frac{1}{n} \sum^n_{i = 1} g'({\bar{\bar{\utheta}}^{(j)}}^T x^{(i)}) {x^{(i)}}^T [ \bar{\utheta}^{(j)} - \utheta^*] \{x^{(i)}_j \, {x^{(i)}}^T [\hutheta - \utheta^*]\}\\
& = \frac{1}{n} \sum^n_{i = 1} \{ g'({\bar{\bar{\utheta}}^{(j)}}^T x^{(i)}) {x^{(i)}_j \} \{ [\bar{\utheta}^{(j)} - \utheta^*]^T x^{(i)}} \, {x^{(i)}}^T [\hutheta - \utheta^*] \}
\end{align}
where $\bar{\bar{\utheta}}^{(j)}$ is a point in the line between $\bar{\utheta}^{(j)}$ and $\utheta^*$. Let 
\begin{equation}
b_i :=  [\bar{\utheta}^{(j)} - \utheta^*]^T x^{(i)} \, {x^{(i)}}^T [\hutheta - \utheta^*]  = t_j [\hutheta - \utheta^*]^T x^{(i)} \, {x^{(i)}}^T [\hutheta - \utheta^*] \geq 0
\end{equation}
then, noticing that $\hutheta_{S^C} = \utheta^*_{S^C} = 0$ and $|g'| \leq 1$ we can apply Holder's inequality to obtain,
\begin{equation}
|R^n_{j}|  \leq \frac{1}{n} \|b\|_1 \leq t_j [\hutheta_S - \utheta_S^*]^T \left\{ \frac{1}{n}  \sum^n_{i = 1} x_S^{(i)} \, {x_S^{(i)}}^T \right\} [\hutheta_S - \utheta_S^*] \leq \Delta \|\hutheta_S - \utheta_S^*\|^2_2.
\end{equation}


Slightly readapting the proof of Lemma 3 from \cite{martin} we now show that
\begin{equation}
\|\hutheta_S - \utheta^*_S\|_2 \leq \frac{C_{\rm min}}{4 \Delta^{3/2}} \left(  1 - \sqrt{1 - \lambda \frac{16 \Delta^2}{C^2_{\rm min}} \left(1  + \Big\| \frac{W^n_S}{\lambda} \Big\|_\infty \right) } \right)\, . \label{opt_dist_bound}
\end{equation}

Define $G(u) = L(\utheta^* + u) - L(\utheta^*) + \lambda (\|\utheta^* + u\|_1 - \|\utheta^*\|_1)$. Since $G(0) = 0$ and $G$ is strictly convex we have that if $G(u) > 0$ for $\|u\|_2 = B$ then $\|\hat{u}\|_2 < B$, where $\hat{u} = \hutheta - \utheta^*$ is the unique minimum point of $G(u)$. To prove \eqref{opt_dist_bound} we will compute a lower bound on the set of points for which $G(u) > 0$.

By the mean value theorem we can write,

\begin{equation}
G(u) = - {W^n}^T u + u^T \nabla^2 L (\utheta^* + \alpha u) u + \lambda (\|\utheta^* + u\|_1 - \|\utheta^*\|_1).
\end{equation}

Note that $W^n = -\nabla L (\utheta^*)$.

We now get bounds on each of the terms of the previous expression,
\begin{eqnarray}
|{W^n}^T u | &\leq& \|W^n_S\|_\infty \sqrt{\Delta} \|u\|_2,\\
\lambda (\|\utheta^* + u\|_1 - \|\utheta^*\|_1) &\geq& -\lambda \sqrt{\Delta} \|u \|_2\\
u^T \nabla^2 L (\utheta^* + \alpha u) u &\geq& \|u\|^2_2 \left( \sigma_{\min}({Q^n_{SS}}^*) - \Delta^{1/2} \|u\|_2 \, \sigma_{\max} \left(\frac{1}{n} \sum^n_{i = 1} x_S^{(i)} {x_S^{(i)}}^T \right) \right) \nonumber,\\ 
&\geq& \|u\|^2_2 \left( C_{\rm min}/2 - \Delta^{3/2} \|u\|_2 \right). 
\end{eqnarray}
Thus can write,
\begin{equation}
G(u) \geq \|u\|_2 \, \Delta^{1/2} \left( -\Delta \|u\|^2_2  + \Delta^{-1/2} \frac{C_{\rm min}}{2} \|u\|_2 - \lambda - \|W^n_S\|_{\infty} \right),
\end{equation}
from which we derive expression \eqref{opt_dist_bound}.

If $\cE_i$ holds we can assume without loss of generality $\| \frac{W^n_S}{\lambda} \|_\infty < \epsilon$.
Now notice that $1 - \sqrt{1 - x} \leq x, x \in [0,1]$ and thus we can write,
\begin{equation}
|R^n_j| \leq \Delta \left( \frac{C_{\min}}{4 \Delta^{3/2}} \frac{16(1 + \epsilon) \Delta^2 \lambda}{C^2_{\min}} \right)^2 \leq \frac{16 \Delta^2 \lambda^2 (1+ \epsilon)^2 }{C^2_{\min}}.
\end{equation}
If we now want that
\begin{equation}
\frac{\Delta|R^n_j|}{\lambda C_{\min}} \leq \frac{\epsilon}{8},
\end{equation}
then we can simply impose that $\lambda < C^3_{\min} \epsilon / (2^7 (1+\epsilon^2) \Delta^3)$, which finishes the proof.

\end{proof}

\subsection{$n \lambda^2$ must be unbounded with $p$}

\begin{proof}(Lemma \ref{th:mart22})
In this proof $S = \{i\}$ and $S^C = \partial r \backslash \{i \}$.

We prove the lemma by computing a lower bound on the probability that $\|\nabla_ {S^C} L (\hutheta;\{\ux^{(\ell)} \}) \|_\infty > \lambda$ under the assumption that $n \lambda^2 \leq K$ and $\hutheta_{S^C} = 0$ and $\hutheta_{S} > 0$ \footnote{The requirement $\htheta_{ri} > 0$, necessary for correct reconstruction, allows us to ignore the $\|.\|_\infty$ and $\|.\|_1$ in what follows.}. This will prove the corresponding upper bound on the probability of success of $\rlr(\lambda)$. 


First we show that there exists an $C(\theta)$ such that if  $\|\hutheta_S\|_\infty > C$ then with high probability $\rlr(\lambda)$ fails. 

Begin by noticing that $\E_{G,\theta} (L(\hutheta)) \geq \htheta_{ir} (1 - \tanh \theta)$ and that $|L(\hutheta) - \E_{G,\theta} (L(\hutheta))| \leq 2 \log 2 + 2 \|\hutheta\|_1$. Then use Azuma's inequality to get the following bound,
\begin{align}
&\prob_{n,G,\theta} ( L(\hutheta) + \lambda \| \hutheta \|_1 > L(\underline{0}) ) \\
&=\prob_{n,G,\theta} ( L(\hutheta) - \E_{G,\theta} (L(\hutheta)) > \log 2 - \lambda \| \hutheta \|_1  - \E_{G,\theta} (L(\hutheta)))\\
& \geq 1 - e^{\frac{-2n (\log 2 - \lambda \htheta_{ir} - \E_{G,\theta} (L (\hutheta) ) )^2 }{(2 \log 2 + 2 \htheta_{ir} )^2}}. 
\end{align}
If $\|\hutheta_S\|_\infty > C(\theta)$ for large enough $C(\theta)$ then we can lower bound the previous expression by
\begin{align}
1 - e^{\frac{-2n (\frac{\htheta_{ir}}{2})^2 (1 - \tanh \theta )^2 }{(2 \log 2 + 2 \htheta_{ir} )^2}} \geq 1 -  e ^ {\frac{-n C^2 (1 - \tanh \theta)^2 }{8 (\log 2 + C)^2}} \geq 1 - e ^ {\frac{-n  (1 - \tanh \theta)^2}{32}} \label{prob_htheta_is_big}.
\end{align}
Since  $L(\hutheta) + \lambda \| \hutheta \|_1 > L(\underline{0})$ contradicts the fact that $\hutheta$ is the optimal solution found by \rlr \, this shows that with high probability $\htheta_{ir}$ must be smaller than $C(\theta)$.

Under the assumption that $\htheta_{ri} \leq C$ and $n \lambda^2 \leq K$ we will now compute a lower bound for the event $\|\nabla_ {S^C} L (\hutheta)\|_\infty > \lambda$.

\begin{align}
&\prob_{n,G,\theta} \left( \| \frac{1}{\lambda} \nabla_ {S^C} L (\hutheta) \|_\infty > 1 \bigg| \htheta_{ri} \leq C \right)\\
&= \prob_{n,G,\theta} \left( \| \frac{1}{\lambda} \frac{1}{n} \sum^n_{\ell=1} X^{(\ell)}_r X^{(\ell)}_{S^C} (1 - \tanh (X^{(\ell)}_r X^{(\ell)}_i \htheta_{ri}) ) \|_\infty > 1\right) \\
& \geq 1 - \prob_{n,G,\theta} \left(  \frac{1}{\sqrt{n}} \sum^n_{\ell=1} X^{(\ell)}_r X^{(\ell)}_{r'} (1 - \tanh (X^{(\ell)}_r X^{(\ell)}_i \htheta_{ri}) )  \leq \sqrt{K} \; , \forall r' \in S^C \right) \\
& \geq 1 - \E_{G,\theta} \left( \prob_{n,G,\theta} \left(  \frac{1}{\sqrt{n}} \sum^n_{\ell=1} X^{(\ell)}_{r'} C_{\ell}  \leq \sqrt{K} \; , \forall r' \in S^C \bigg| \{C_\ell\}^n_{\ell=1} \right) \right),
\end{align}
where $C_{\ell} = X^{(\ell)}_r (1 - \tanh(X^{(\ell)}_r X^{(\ell)}_i \htheta_{ri}))$.

Conditioned on $\{C_{\ell}\}^n_{l=1}$ all the $\sum_{\ell} X^{(\ell)}_{r'} C_{\ell}$ are independent and identically distributed. Hence, choosing one particular $r'_0 \in S_C$, and defining $V_{\ell} = X^{\ell}_{r'_0}$ we can rewrite the previous expression as,
\begin{eqnarray}
1 - \E_{G,\theta} \left( \prob_{n,G,\theta} \left(  \frac{1}{\sqrt{n}} \sum^n_{\ell=1} V_{\ell} C_{\ell} \leq \sqrt{K} \bigg| \{C_{\ell}\}^n_{_{\ell}=1} \right)^{p-1} \right).
\end{eqnarray}
We now use the central limit theorem for independent nonidentical random variables to upper bound the conditional probability inside the expectation. It is easy to see that Lyapunov conditions hold. In fact, let $s^2_n = \sum^n_{\ell = 1} \text{Var}(V_{\ell} C_{\ell} | \{C_{\ell}\}^n_{_{\ell}=1}) = \sum^n_{_{\ell} = 1}  C^2_{\ell}$  then for some $\delta > 0$,
\begin{eqnarray}
&&\E_{G,\theta}(|V_{\ell} C_{\ell}|^{2+\delta} | \{C_{\ell}\}^n_{\ell=1}) = |C_{\ell}|^{2+\delta} < \infty \; , \forall \ell \\
\end{eqnarray}
and
\begin{eqnarray}
&&\lim_{n \to \infty} \frac{1}{s^{2+\delta}_n} \sum^n_{\ell=1} \E_{G,\theta} (|V_{\ell} C_{\ell} - \E_{G,\theta} (V_{\ell} C_{\ell}|\{C_{\ell}\}^n_{\ell=1}) |^{2+\delta} | \{C_{\ell}\}^n_{\ell=1})\\
&&=\lim_{n \to \infty} \frac{1}{(\sum^n_{\ell=1} C^2_{\ell})^{1+\delta/2}} \sum^n_{\ell=1} |C_{\ell}|^{2+\delta} \leq \lim_{n \to \infty} n^{-\delta/2} \left(\frac{1+\tanh C(\theta) }{1-\tanh C(\theta)}\right)^{2+\delta} = 0.
\end{eqnarray}
Thus we can write,
\begin{eqnarray}
&\prob_{n,G,\theta} \left(  \frac{1}{\sqrt{n}} \sum^n_{\ell=1} V_{\ell} C_{\ell} \leq \sqrt{K} \bigg| \{C_{\ell}\}^n_{\ell=1} \right) = 
\prob_{n,G,\theta} \left(  \frac{\sum^n_{\ell=1} V_{\ell} C_{\ell}}{\sqrt{\sum^n_{\ell=1} C^2_{\ell}}} \leq \sqrt{K} \frac{\sqrt{n}}{\sqrt{\sum^n_{\ell=1} C^2_{\ell}}} \bigg| \{C_{\ell}\}^n_{\ell=1} \right) \\
& \leq \prob_{n,G,\theta} \left(  \frac{\sum^n_{\ell=1} V_{\ell} C_{\ell}}{\sqrt{\sum^n_{\ell=1} C^2_{\ell}}} \leq \frac{\sqrt{K}}{1 - \tanh C(\theta)} \bigg| \{C_{\ell}\}^n_{\ell=1} \right) = \Phi \left( \frac{\sqrt{K}}{1 - \tanh C(\theta)}  \right) + \epsilon_n
\end{eqnarray}
where $\Phi$ is the cumulative distribution of the normal(0,1) distribution and $\epsilon_n \to 0$ with $n$.
We can finally write,
\begin{equation}
\prob_{n,G,\theta} \left( \| \frac{1}{\lambda} \nabla_ {S^C} L (\hutheta) \|_\infty > 1 \bigg| \htheta_{ri} \leq C \right)  \geq  1 - e^{(p-1) (\log (\Phi \left( \frac{\sqrt{K}}{1 - \tanh C(\theta)}  \right) + \epsilon_n )) } \geq  1 - e^{-p M(K, \theta)}
\end{equation}
for $n$ big enough. In the above expression $M(K,\theta) \rightarrow 0$ as $K \rightarrow \infty$.
From this bound and \eqref{prob_htheta_is_big} we get the desired upper bound on the probability of success of \rlr.
\end{proof}

\subsection{Random regular graphs and the violation of the incoherence condition} \label{sec:rand_reg_incoh_cond}

\begin{proof}(Lemma \ref{th:key_result_reg_graph})
We explain here the calculations with respect to
the tree model \eqref{IsingTree}. Throughout all calculations
we assume that $0 < \theta < \infty$.
An important property that follows from the fixed point equation \eqref{eq:ising_field_fix_point} 
is that, if $g(\ux_{\Tree(t)})$ is a function of the variables in 
$\Tree(t)$ then
\begin{eqnarray}
\E_{\Tree(t),\theta,+}\{g(\uX_{\Tree(t)})\} = 
\E_{\Tree(t+1),\theta,+}\{g(\uX_{\Tree(t)})\}\ ,
\end{eqnarray}
with the obvious identification of $\Tree(t)$ as a subtree of $\Tree(t+1)$.

Let $r$ be a uniformly random vertex in $G$ and $i\neq j$ two neighbors of $r$.
Using the local weak convergence property \eqref{eq:Weak} with $t=1$
we get
\begin{eqnarray}
\lim_{p\to\infty}(Q^*_{SS})_{ii}& \equiv &a = \E_{\Tree(1),\theta,+} 
\Big(\frac{1}{\cosh^2 \theta M}\Big)\, ,\\
\lim_{p\to\infty}(Q^*_{SS})_{ij}& \equiv &b = \mathbb{E}_{\Tree(1),\theta,+} 
\Big(\frac{X_i X_j}{\cosh^2 \theta M}\Big),
\end{eqnarray}
where $M\equiv\sum_{i\in\dtree(1)}X_i$ is the sum of the variables on the 
leaves of a depth $1$ tree, and $i,j\in\dtree(1)$.
For $r'$ at distance $t > 1$ from $r$, consider the 
$\Delta$-dimensional vector in 
\begin{eqnarray}
\lim_{p\to\infty}(Q^*_{S^c S})_{r'}= F_S(t)\, .
\end{eqnarray}
Elements of $F_S(t)$ are of the form $\mathbb{E}_{\Tree(t),\theta,+} 
\Big(\frac{X_{r'} X_i}{\cosh^2 \theta M}\Big)$ where $i\in\dtree(1)$. These elements can take only two different values: one if $r'$ is a child of $j$ and other if not. We denote the first value by $\rm F_d(t)$ and the second by $\rm F_i(t)$. 
Since $\hat{z}_S^* = \mathds{1}$ is an eigenvector of $Q^*_{SS}$ with eigenvalue $a + (\Delta - 1) b$
we can write,
\begin{equation}
\lim_{p \to \infty} \|Q^*_{S^C S} {Q^*}^{-1}_{SS} \hat{z}_S^*\|_{\infty} = \sup_{t \geq 1} |A(t)| \label{alf_inf}
\end{equation}
where
\begin{eqnarray}
A(t) &=& \frac{F_d(t) + (\Delta -1) F_i(t)}{a + (\Delta - 1) b} = \frac{\E_+(X_{r'} M/ \cosh^2(\theta M))}{\E_+(X_i M/ \cosh^2(\theta M))}.
\end{eqnarray}
In this expression, and through the rest of the proof, $\E_+$ will denote $\E_{\Tree(t'),\theta,+}$ where
$t'$ is the smallest value such that all the variables inside the expectation are in $\Tree(t')$.
Now, conditioning on the value of $X_i$ ($i \in \dtree(1)$) we can write,
\begin{eqnarray}
\E_+(X_{r'} M/ \cosh^2(\theta M)) &=& c_1 \< t \>_+  +  c_2 \< t \>_-,\\
\E_+(X_i M/ \cosh^2(\theta M)) &=& c_1 - c_2.
\end{eqnarray}
where 
\begin{eqnarray}
c_1 &=& \E_+(\ident_{X_i = 1} M / \cosh^2(\theta M)),\\
c_2 &=& \E_+(\ident_{X_i = -1} M / \cosh^2(\theta M)),\\
\< t \>_+ &=& \E_+(X_{r'}|X_i = 1),\\
\< t \>_- &=& \E_+(X_{r'}|X_i = -1).\\
\end{eqnarray}
In the expression above the binomial coefficients are to be assume zero whenever its parameters are
not integer values.
In order to prove that the incoherence condition is violated we will now show that $B \equiv \lim_{t \rightarrow \infty}A(t) > 1$ if $\theta$ is large enough.
Writing a first order recurrence relation for $\< t \>_+$ and $\< t \>_-$ it is not hard to see that,
\begin{eqnarray}
\< t \>_+ &=& \frac{\beta - \alpha}{\alpha + \beta -2} + \frac{2(\alpha-1)}{\alpha + \beta - 2} (\alpha + \beta - 1)^t,\\
\< t \>_- &=& \frac{\beta - \alpha}{\alpha + \beta -2} + \frac{2(1 - \beta)}{\alpha + \beta - 2} (\alpha + \beta - 1)^t,
\end{eqnarray}
where 
\begin{eqnarray}
\alpha &=& \prob_+ (X_{r''} = 1 | X_{r'} = 1) = e^{h^* + \theta}/(e^{h^* + \theta} + e^{-h^* - \theta}),\\
\beta &=& \prob_+ (X_{r''} = -1 | X_{r'} = -1) = e^{-h^* + \theta}/(e^{h^* - \theta} + e^{-h^* + \theta}),
\end{eqnarray}
and $r''$ denotes a child of $r$, i.e., a node at distance $t+1$ from $r$. Recall that $h^*$ is the unique positive solution of \eqref{eq:ising_field_fix_point}. In the above expression $\prob_+$ denotes the probability associated with the measure \eqref{IsingTree} where again we can restrict $\Tree$ to the smallest
tree containing all the variables that compose the event whose probability we are trying to compute.
Since $0<\alpha + \beta - 1 <1$
we have that
\begin{equation}
B = \frac{\beta - \alpha}{\alpha + \beta -2} \frac{c_1 + c_2}{c_1 - c_2}.
\end{equation}
A little bit of algebra allows us to write,
\begin{eqnarray}
\frac{\beta - \alpha}{\alpha + \beta -2} &=& \frac{(1-\beta) - (1 - \alpha)}{(1 - \alpha) + (1 - \beta)}\\
&=&\frac{\prob_+(X_{r''} = 1 | X_{r'} = -1) - \prob_+(X_{r''} = -1 | X_{r'} = 1) }{\prob_+(X_{r''} = 1 | X_{r'} = -1) + \prob_+(X_{r''} = -1 | X_{r'} = 1)}\\
&=& \frac{\frac{\prob_+(X_{r''} = 1 , X_{r'} = -1)}{ \prob_+(X_{r'} = - 1)}- \frac{\prob_+(X_{r''} = -1 , X_{r'} = 1)}{\prob_+(X_{r'} = 1)} }{ \frac{\prob_+(X_{r''} = 1 ,X_{r'} = -1)}{ \prob_+(X_{r'} = - 1)} + \frac{\prob_+(X_{r''} = -1 , X_{r'} = 1)}{\prob_+(X_{r'} = 1)}}\\
&=& \frac{1 / \prob_+(X_{r'} = - 1)- 1/ \prob_+(X_{r'} = 1) }{1 / \prob_+(X_{r'} = - 1)+ 1 / \prob_+(X_{r'} = 1)}\\
&=&\frac{ \prob_+(X_{r'} = 1) - \prob_+(X_{r'} = - 1) }{ \prob_+(X_{r'} = 1) +  \prob_+(X_{r'} = -1)} = \E_{\Tree(1),\theta,+} (X_{r'}) = \tanh(\Delta h^* / (\Delta - 1)).
\end{eqnarray}
In addition, taking into account that $c_1$ and $c_2$ can be expressed as,
\begin{eqnarray}
c_1 = \frac{2}{Z} \sum^\Delta_{m = -\Delta} {\Delta-1 \choose \frac{\Delta+m-2}{2}}  \frac{m e^{h^* m} }{\cosh \theta m},\\
c_2 = \frac{2}{Z} \sum^\Delta_{m = -\Delta} {\Delta-1 \choose \frac{\Delta+m}{2}}  \frac{m e^{h^* m} }{\cosh \theta m},\\
\end{eqnarray}
we have,
\begin{equation}
\frac{c_1 + c_2}{c_1 - c_2} = \frac{\sum^\Delta_{m = 1} {\Delta \choose \frac{\Delta+m}{2}}  m \frac{\sinh m h^*}{\cosh \theta m}} {\sum^\Delta_{m = 1} {\Delta \choose \frac{\Delta+m}{2}} \frac{m^2}{\Delta} \frac{\cosh h^* m}{\cosh \theta m}}.
\end{equation}
Expanding everything in powers of $e^{-h^*}$ we get,
\begin{align}
\lim_{p \rightarrow \infty} \|Q^*_{S^c S} {Q^*_{SS}}^{-1} \| \geq B = \left( 1 - 2 e^{-2 h^* \Delta / (\Delta - 1)} + ...\right) \left( 1 + (\Delta -2) e^{-2 h^* + 2} + ...\right)\\
\left( 1 - \frac{(\Delta - 2)^2}{\Delta} e^{-2 h^* + 2} + ...\right) = 1 + 2\frac{\Delta-2}{\Delta} e^{-2 h^* + 2} + ....
\end{align}
Since $h^*$ grows with $\theta$ \footnote{$h^* = (\Delta - 1 + o(1)) \theta$} this expansion proves the first part of Lemma \ref{th:mart23}. In fact, this expression shows that for large $\theta$, as $\theta$ increases, $B$ decays to 1 from above. Hence, there exists a $\theta_{\rm thr} (\Delta)$ such that for all $\theta > \theta_{\rm thr} (\Delta)$ we will have $\lim_{p \rightarrow \infty}\|Q^*_{S^C S} Q^*_{SS}\|_{\infty} > 0$.

\begin{remark} \label{sec:remark_on_incoh_region}
It is interesting to see that the condition for $B \geq 1$ is equivalent to $c_1 (\alpha -1) + c_2 (1 - \beta) \geq 0$. This implies that if $B \geq 1$ then $A(t) \geq B$ and if $B \leq 1$ then $A(t) \leq B$. Hence,
when $B > 1$ we have $A(t) > 1 \; \forall t$ and when $B < 1$ we have $A(t) < 1 \; \forall t$. Consequently, $\{\theta: A(t) > 1\}= \{\theta: B >1\}$ which does not depend on $t$. It is not hard to prove
that $A(t) > 0 \; \forall t,\theta$ and thus, 
$\{ \theta : \lim_{p \rightarrow \infty} \|Q^*_{S^C S} {Q^*_{SS}}^{-1} \|_{\infty} > 1  \} = \{\theta: B >1\}$.
\end{remark}

We now study how $\theta_{\rm thr}(\Delta)$ scales with $\Delta$ for large $\Delta$.
Notice that $B = 1$ is equivalent to $S(\theta) \equiv c_1 (\alpha - 1) + c_2 (1 - \beta) = 0$.
It is not hard to see that this equation has a single solution \footnote{By Remark \ref{sec:remark_on_incoh_region}, this tells us that there is a single point where $\lim_{p \rightarrow \infty} \|Q^*_{S^C S} {Q^*_{SS}}^{-1} \|_{\infty}$ crosses 1.}.
We show that if we search for solutions, $\theta$, that scale like $\Delta^{-1}$ then in the limit when $\Delta \to \infty$ we get an expression that exhibits a single nontrivial zero. This means that for large $\Delta$ the solution of $S(\theta) = 0$ must be of the form $\tilde{\theta} \Delta^{-1}(1 + o(1))$, where $\ttheta$ is the solution of the scaled equation.  

First notice that when $\Delta \to \infty$ and $\theta = \tilde{\theta} / (\Delta -1)$ then $h^*$ converges to the solution of $h^* = \tilde{\theta} \tanh h^*$. We denote this solution by $h^*_{\infty}$. Hence, for large finite $\Delta$ we can say that $h^* = h^*_{\infty} + O(\Delta^{-1})$.

We now write new expressions for $c_1,c_2,\alpha$ and $\beta$ namely,
\begin{eqnarray}
c_1 &=& \frac{1}{2} \E_+(M/\cosh^2(\theta M)) + \frac{1}{2 \Delta} \E_+(M^2 / \cosh^2(\theta M)),\\
c_2 &=& \frac{1}{2} \E_+(M/\cosh^2(\theta M)) - \frac{1}{2 \Delta} \E_+(M^2 / \cosh^2(\theta M)),\\
1-\alpha &=& \frac{1}{2} (1 - \tanh (h^* + \ttheta / (\Delta -1))),\\
1-\beta &=& \frac{1}{2} (1 + \tanh (h^* - \ttheta / (\Delta -1))).\\
\end{eqnarray}
Expanding the function $\tanh(.)$ in $\alpha$ and $\beta$ in powers of $\Delta^{-1}$ we can write
\begin{eqnarray}
S(\theta) &=& \frac{1}{2} \tanh h^* \; \E_+ (M / \cosh^2 (\theta M )) - \frac{1}{2 \Delta} \E_+(M^2 / \cosh^2 (\theta M)) \\
&+&\frac{\ttheta}{2 \Delta (\Delta - 1)} \text{sech}^2 h^* \; \E_+(M^2/\cosh^2(\theta M)) + O(\Delta^{-2}).
\end{eqnarray}
Note that we have not expanded $h^*$ in powers of $\Delta^{-1}$.
Defining $\E_+^0(.)$ to be the expectation with respect to the tree model \eqref{IsingTree} where all connections to node $r$ have been removed (the field on each node is still $h^*$) we can write,
\begin{eqnarray}
\E_+(M / \cosh^2(\theta M)) &=& \frac{\E_+^0 (M / \cosh(\ttheta M / (\Delta - 1)) )}{\E_+^0 (\cosh (\ttheta M / (\Delta -1)) )},\\
\E_+(M^2 / \cosh^2(\theta M)) &=& \frac{\E_+^0 (M^2 / \cosh(\ttheta M / (\Delta - 1)) )}{\E_+^0 (\cosh (\ttheta M / (\Delta -1)) )}.
\end{eqnarray}
In addition, making use of the symmetry of the regular tree and expanding $\cosh(\ttheta M / (\Delta -1))$ around $\ttheta M' / (\Delta -1)$ and $\ttheta M'' / (\Delta -1)$ ($M'$ and $M''$ to be defined later) we can write
\begin{eqnarray}
\E_+^0 \left(\frac{M} {\cosh(\ttheta M / (\Delta - 1) )}\right) &=& \Delta \; \E_+^0 \left(\frac{X_i}{ \cosh(\ttheta M / (\Delta - 1))} \right),\\
\E_+^0 \left(\frac{X_i}{ \cosh(\ttheta M / (\Delta - 1))} \right)&=& \tanh h^* \; \E_+^0 \left(\frac{1}{\cosh(\ttheta M' / (\Delta - 1))} \right)\\
&-& \frac{\ttheta}{\Delta - 1} \E_+^0 \left( \frac{\tanh(\ttheta M' / (\Delta - 1))}{\cosh(\ttheta M' / (\Delta -1))} \right) + O(\Delta^{-2}),\\
\E_+^0 \left(\frac{M^2} {\cosh(\ttheta M / (\Delta - 1) )}\right) &=& \Delta \E_+^0 \left(\frac{1}{\cosh(\ttheta M / (\Delta - 1))} \right)\\
&+&\Delta (\Delta - 1) \; \E_+^0 \left(\frac{X_i X_j}{ \cosh(\ttheta M / (\Delta - 1))} \right),\\
\E_+^0 \left(\frac{X_i X_j}{ \cosh(\ttheta M / (\Delta - 1))} \right) &=& \tanh^2 h^* \; \E_+^0 \left(\frac{1}{\cosh(\ttheta M'' / (\Delta - 1))} \right)\\
&-& 2 \frac{\ttheta}{\Delta -1} \tanh h^* \; \E_+^0 \left(\frac{ \tanh(\ttheta M'' / (\Delta - 1))}{ \cosh(\ttheta M'' / (\Delta - 1))} \right) + O(\Delta^{-2}),
\end{eqnarray}
where $M' = M - X_i$ and $M'' = M - X_i - X_j$. Using these relations, the law of large numbers and the relation $h^* = h^*_{\infty} + O(\Delta^{-1})$ where $h^*_{\infty} = \ttheta \tanh h^*_{\infty}$ it is now possible to calculate the limit
\begin{eqnarray}
\lim_{\Delta \to \infty} S(\ttheta/(\Delta-1)) =\frac{-1 + h^*_{\infty} \tanh h^*_{\infty}}{2 \cosh^4 h^*_{\infty}}.
\end{eqnarray}
This finishes the proof of the second part of the lemma since $h_{\infty}$ can now be determined by $h_{\infty} \tanh h_{\infty} = 1$ and $\ttheta = h_{\infty}^2$.

We now show how to deduce the above expression. Let us introduce
the following notation,
\begin{align}
E_0 &=  \E_+^0 \left(\frac{1}{\cosh(\ttheta M / (\Delta - 1))} \right),
E_1 =  \E_+^0 \left(\frac{1}{\cosh(\ttheta M' / (\Delta - 1))} \right),\\
E_2 &=  \E_+^0 \left(\frac{1}{\cosh(\ttheta M'' / (\Delta - 1))} \right),
F_1 =  \E_+^0 \left(\frac{\tanh(\ttheta M' / (\Delta - 1))}{\cosh(\ttheta M' / (\Delta - 1))} \right),\\
F_2 &=  \E_+^0 \left(\frac{\tanh(\ttheta M'' / (\Delta - 1))}{\cosh(\ttheta M'' / (\Delta - 1))} \right),
D = \E_+^0 \left(\cosh(\ttheta M / (\Delta - 1)) \right).
\end{align}
With this in mind and recalling that $\theta = \ttheta / (\Delta - 1)$ we can write,
\begin{align}
S(\theta) &=\\
& \frac{\tanh h^*}{2D} \left (\Delta E_1 \tanh h^* - \frac{ \ttheta\Delta}{\Delta -1} F_1 \right)\\
&-\frac{1}{2 \Delta D} \left(1 - \frac{\ttheta}{\Delta -1} \text{sech}^2 h^* \right) \left( \Delta E_0 + \Delta (\Delta -1) \left( E_2 \tanh^2 h^* - \frac{2 \ttheta}{\Delta - 1} F_2 \tanh h^* \right) \right)\\
&+O(\Delta^{-1})\\
&= \frac{\tanh^2 h^*}{2 D} (\Delta E_1 - (\Delta - 1)E_2 ) - \frac{E_0}{2D} - \frac{\Delta}{\Delta - 1}\frac{F_1 }{2D} \ttheta \tanh h^* + \frac{F_2}{D} \ttheta  \tanh h^*\\
&+ \frac{E_2}{2D} \ttheta \text{sech}^2 h^* \tanh^2 h^* +  \frac{1}{\Delta -1} \frac{E_0}{2D} \ttheta \text{sech}^2 h^* - \frac{1}{\Delta - 1} \frac{F_2}{2D} \ttheta^2 \tanh h^* \text{sech}^2 h^*  + O(\Delta^{-1}).
\end{align}
Now notice that expanding the $\cosh(.)$ inside $E_1$ in expression $\Delta E_1 - (\Delta - 1)E_2$ around $M'' \ttheta / (\Delta -1)$ we can rewrite the same expression as,
\begin{align}
E_2 - \frac{\Delta}{\Delta - 1} \ttheta \tanh h^* F_2 + O(\Delta^{-1}).
\end{align}
Inserting this in the above expression finally gives us,
\begin{align}
S(\theta) &= \frac{E_2}{2 D} \tanh^2 h^* -\frac{\Delta}{\Delta -1} \frac{F_2}{2D}\ttheta \tanh^3 h^* - \frac{E_0}{2D} - \frac{\Delta}{\Delta - 1}\frac{F_1 }{2D} \ttheta \tanh h^* + \frac{F_2}{D} \ttheta  \tanh h^*\\
&+ \frac{E_2}{2D} \ttheta \text{sech}^2 h^* \tanh^2 h^* + O(\Delta^{-1}).
\end{align}

By the law of large numbers we have,
\begin{align}
\lim_{\Delta \to \infty} M/ (\Delta -1) &= \lim_{\Delta \to \infty} M'/ (\Delta -1) = \lim_{\Delta \to \infty} M''/ (\Delta -1)\\
&= \lim_{\Delta \to \infty} \E_+ (X_i) \bigg|_{\theta = \ttheta/(\Delta -1)}= \tanh h^*_{\infty},
\end{align}
and since all the variables inside the expectations are uniformly bounded, we can take the limit inside all the expectations of our expression for $S(\theta)$. Doing so we get,
\begin{align}
\lim_{\theta \to \infty} S(\ttheta / (\Delta - 1)) &= \frac{\tanh^2 h^*_{\infty}}{2\cosh^2 h^*_{\infty}} - \frac{\ttheta \tanh^4 h^*_{\infty}}{2\cosh^2 h^*_{\infty}} - \frac{1}{2 \cosh^2 h^*_{\infty}} - \frac{\ttheta \tanh^2 h^*_{\infty}}{2 \cosh^2 h^*_{\infty}}\\
& + \frac{\ttheta \tanh^2 h^*_{\infty}}{\cosh^2 h^*_{\infty}} + \frac{\ttheta \tanh^2 h^*_{\infty}}{2 \cosh^4 h^*_{\infty}}.
\end{align}

If we now use the relation $h_{\infty} = \ttheta \tanh h_{\infty}$ this expression can be simplified to,
\begin{align}
\frac{1}{2 \cosh^4 h^*_{\infty}} (-1 + h^*_{\infty} \tanh h^*_{\infty}).
\end{align}

Finally, we show that there exists a constant $C_{\min}$ such that 
\begin{equation}
\lim_{p \to \infty} \sigma_{\min}(Q^*_{SS}) = \sigma_{\min} \left( \lim_{p \to \infty}  Q^*_{SS} \right) > C_{\min}. \footnote{The equality is guaranteed since the sequence of matrices $\{Q^*_{SS}\}^\infty_{p = 1}$ have fix finite dimensions.} 
\end{equation}
First notice that the eigenvalues of $\lim_{p \to \infty}  Q^*_{SS}$ are $a-b$ and $a + (\Delta -1) b = c_1 - c_2$. It is immediate to see that $a-b > 0$. In addition, since $\Delta (c_1 - c_2) = \E_+ (M^2 / \cosh^2 (\theta M)) > 0$ it follows that $c_1 - c_2 >0$. Hence we can choose $C_{\min} = \min\{a-b,c_1-c_2\} >0$.

\end{proof}

\section{Analysis of $\rlr(\lambda)$  for other families of graphs} \label{sec:inco_other_graphs}

\begin{figure}
\begin{center}
\includegraphics[width=0.2\linewidth]{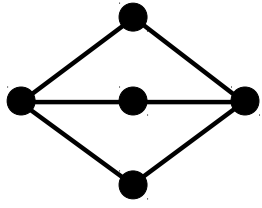}
\end{center}
\caption{Example of small graph for which the incoherence fails.} \label{fig:simple_graph_incoh_cond}
\end{figure}

As already discussed, the success of \rlr \, is closely related
to the incoherence condition.
For small graphs, brute force computations allow to explicitly evaluate
this condition. For example, 
consider the reconstruction of the neighborhood of the leftmost node in
the graph of Figure \ref{fig:simple_graph_incoh_cond}.
The corresponding incoherence parameter takes the for,
\begin{equation} \label{eq:inch_g_p_5}
\|Q^*_{S^CS} {Q^*_{SS}}^{-1} \|_{\infty}= \frac{3 x (1 + x^2)}{1 + 3 x^2},
\end{equation}
where $x = \tanh \theta$. 
For $x > x_*\equiv \frac{1}{3} \left(1-\sqrt[3]{2}+2^{2/3}\right)\approx
0.44249$ (i.e. for $\theta> \atanh(x_*) \approx 0.475327$) the right
hand side is larger than $1$, whence the 
incoherence condition is violated $\|Q^*_{S^CS} {Q^*_{SS}}^{-1} \|_{\infty}>1$.

This simple calculation strongly suggests that $\rlr(\lambda)$ fails
on the graph of Figure \ref{fig:simple_graph_incoh_cond} for $\theta>\atanh(x_*)$,
although it does not provide a complete proof of this failure.
In this appendix we study three classes of  graphs of increasing size.
We show that with high probability \rlr \, 
succeeds in reconstructing trees.
On the other hand, we show  that it fails --for $\theta$ large
enough-- at reconstructing large
two-dimensional grids,
and that in fails in reconstructing  graphs $G_p$ from the toy example
in Section \ref{sec:toy_example}.

\subsection{Trees}

\begin{lemma}
If $G$ is a tree rooted at $r$ with depth $> 1$ and node $r$ has degree $\Delta > 1$ then, for this node
\begin{equation}
\|Q^*_{S^CS} {Q^*_{SS}}^{-1} \|_{\infty} = \tanh \theta < 1,
\end{equation}
$\sigma_{\min}(Q^*_{SS}) \geq (1-\tanh^2 \theta)/\cosh^2(\theta \Delta)$
and $\sigma_{\max}(\E_{G,\theta} (X^T_S X_S)) = 1 + (\Delta-1) \tanh^2 \theta$.
\label{th:alpha_for_tree}
\end{lemma}

\begin{proof}
In what follows $\E$ will denote $\E_{G,\theta}$. Consider a node $r' \in S^C$ and let $k \in S$ be the unique node in $S$ that belongs to the shortest path connecting $r'$ to $r$. Let $t$ be the distance between $r'$ and $k$. For every $i \in S$ one can write,
\begin{equation*}
Q^*_{r' i} = \E(X_{r'} X_i/\cosh^2 (\theta M)) = \E(X_{r'} X_k) \E(X_k X_i/\cosh^2 (\theta M)) = (\tanh \theta)^t \;\E(X_k X_i/\cosh^2 (\theta M)).
\end{equation*}
This equation is still valid if $k = i$. We can thus write that $Q^*_{r' S} = (\tanh \theta)^t \; Q^*_{k S}$ and hence $Q^*_{r' S} (Q^*_{SS})^{-1} = (\tanh \theta)^t e_k$ where $e_k$ is a row vector with all entries equal to zero except $k^{\text{th}}$ entry that equals 1. Therefore we can write $\| Q^*_{r' S} (Q^*_{SS})^{-1} \|_1 = (\tanh \theta)^t $. Since there must exist at least one node $r' \in S$ for which the corresponding node $k$  is at distance 1 from $S$, that is for which $t = 1$, we conclude that $\|Q^*_{S^CS} {Q^*_{SS}}^{-1} \|_{\infty} = \tanh \theta < 1$.

To prove the spectral bounds first notice that the positive-semidefinite matrix $Q^*_{SS}$ has entries $Q^*_{ij} = (a-b) \delta_{ij} + b$ where $a = \E(1/ \cosh^2 (\theta M))$ and $b = \E(X_1 X_2/ \cosh^2(\theta M))$ and where 1 and 2 are any two distinct nodes in $S$. A matrix of this form has eigenvalues $a-b$ and $a + (\Delta-1)b$. It is not hard to see that $b \geq 0$ and hence
\begin{equation}
\sigma_{\min}(Q^*_{SS}) = a - b = \E ((1-X_1 X_2) / \cosh^2(\theta M)) \geq \E (1-X_1 X_2) / \cosh^2(\theta \Delta).
\end{equation}
Since $\E (1-X_1 X_2) = 1 - \tanh^2 \theta$ the lower bound follows.

The computation of the value of the maximum eigenvalue value of $\E_{G,\theta} (X^T_S X_S)$ is trivial since this matrix is also of the form $(a-b) \delta_{ij} + b$ with $a = 1$ and $b = \tanh^2 \theta$.

\end{proof}

%
\subsection{Two-dimensional grids}

\begin{lemma}
If $G$ is a two dimensional grid with periodic boundary conditions
(each node connects to its four closest neighbors) then for $p$ large
enough $\theta > \theta_c$ we have
$\|Q^*_{S^CS} {Q^*_{SS}}^{-1} \|_{\infty} > 1 + \epsilon$ and 
$\sigma_{\min}(Q^*_{SS}) > C_{\min}$ where $\theta_c$, $\epsilon > 0$ and $C_{\min}$ are independent of $p$.
\label{th:alpha_for_grid}
\end{lemma}

\begin{proof}
We shall compute a lower bound on $\|Q^*_{S^CS} {Q^*_{SS}}^{-1}
\|_{\infty}$ by means of a low temperature expansion, i.e. a Taylor
expansion in powers of $e^{-\theta}$.
We will show that for this lower bound the lemma holds. 

Label the central node as node 0, the neighboring nodes as 1, 2, 3 and 4. Denote as node 5 be the common neighbor of node 1 and node 4. Throughout this proof we will denote $\E_{G,\theta}$ by $\E$ and $\prob_{G,\theta}$ by $\prob$.
\begin{figure}
\begin{center}
\includegraphics[width=0.25\linewidth]{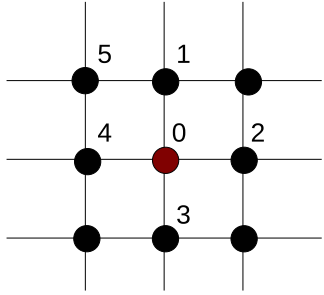}
\end{center}
\caption{Labeling of the nodes in the grid.}\label{fig:grid_label}
\end{figure}

First notice that due to the periodic boundary condition there is symmetry along the vertical and horizontal axis in the lattice. Knowing this, matrix $Q^*_{SS}$ can be written as
\begin{equation}
\left[ \begin{array}{cccc} a & b & c & b \\ b & a & b & c \\ c & b & a & b \\ b & c & b & a \end{array} \right],
\end{equation}
where $a = \E(1/\cosh^2 (\theta M))$, $b = \E(X_1 X_2/\cosh^2 (\theta M))$ and $c = \E(X_1 X_3/\cosh^2 (\theta M))$, where $M = \sum_{i \in \partial i} X_i$, that is, $M$ is the sum of the variables in the neighborhood of $i$ ($i$ not included).
Since we only want to prove a lower bound on $\|Q^*_{S^CS} {Q^*_{SS}}^{-1} \|_{\infty}$ we only consider the row of $Q^*_{S^CS}$ associated with node 5. This row has the form,
\begin{equation}
\left[ \begin{array}{cccc} d & e & e & d\end{array} \right],
\end{equation}
where $d = \E(X_1 X_5/\cosh^2 (\theta M))$ and $e = \E(X_2 X_5/\cosh^2 (\theta M))$.
To compute the low temperature expansions of each of these quantities we first write,
\begin{eqnarray}
a &=& \prob (|M| = 0) + \frac{1}{\cosh^2 2 \theta} \prob(|M| = 2) + \frac{1}{\cosh^2 4 \theta} \prob(|M| = 4),\\
\E \left( \frac{X_i X_j}{\cosh^2 \theta M}\right) &=& [\prob (|M| = 0, X_i X_j = 1) - \prob (|M| = 0, X_i X_j = -1)] \\
&+& \frac{1}{\cosh^2 2 \theta} [\prob (|M| = 2, X_i X_j = 1) - \prob (|M| = 2, X_i X_j = -1)]\\
&+& \frac{1}{\cosh^2 4 \theta} [\prob (|M| = 4, X_i X_j = 1) - \prob (|M| = 4, X_i X_j = -1)].
\end{eqnarray}
The problem thus resumes to the computation of the above probabilities. We will exemplify the calculation of the low temperature expansion of $\prob(|M| = 0)$, the calculation of the expansion for the other terms follows in a similar fashion.

Let $\mathcal{H}(\ux) = \sum_{(ij) \in E} x_i x_j$, $\mathcal{H}_{\max} = \max_{\ux} \mathcal{H}(\ux) = |E|$ and $\delta \mathcal{H}(\ux) = \mathcal{H}(\ux) - \mathcal{H}_{\max} = -2 \mathcal{P}(\ux)$ where $\mathcal{P}(\ux)$ is the length of the boundary separating positive spins from negative spins in configuration $\ux$. Then,
\begin{eqnarray}
\prob(|M| = 0) &=& \frac{2}{Z} \sum_{\{\ux: x_0 = 1, M = 0\}} e^{\theta \mathcal{H}(\ux)} \\
&=& \frac{2}{Z} e^{\theta \mathcal{H}_{\max}} \sum_{s \geq 4} \sum_{\{\ux: x_0 = 1, M = 0, \mathcal{P} = s \}} e^{-2 \theta s}.
\end{eqnarray}
The term $2 e^{\theta \mathcal{H}_{\max}}/Z$ appears in all $a,b,c,d$ and $e$ and thus is irrelevant for the computation of $[Q^*_{S^CS} {Q^*_{SS}}^{-1}]_5$. Since only configurations with zero magnetization contribute to the sum there are two basic types of configurations we need to consider, both of which must have exactly two neighbors of node 0 with negative spin. These are represented in figure \ref{fig:grid_basic_confi_type}.
\begin{figure}
\begin{center}
\includegraphics[width=0.25\linewidth]{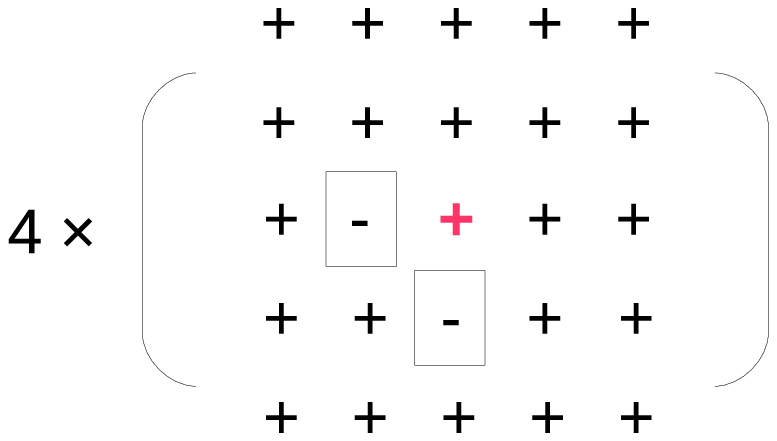} \hspace{1.2cm}
\includegraphics[width=0.25\linewidth]{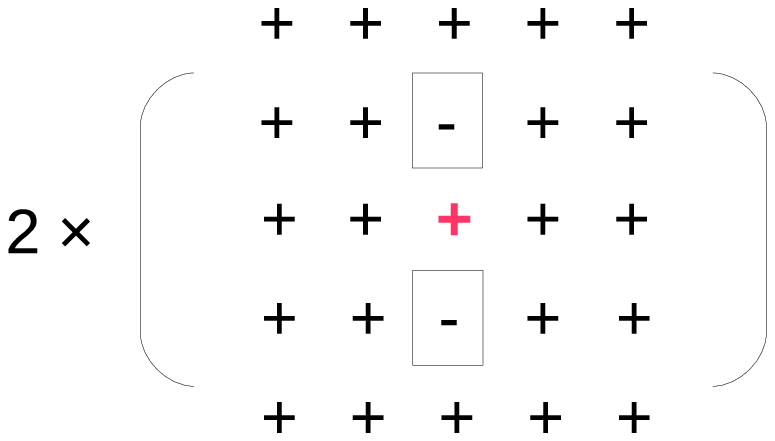}
\end{center}
\caption{Basic type of configurations for the calculation of $\prob(|M| = 0)$. The number in front of each picture represents the number of equivalent symmetric configurations that need to be taken into account.}\label{fig:grid_basic_confi_type}
\end{figure}
Starting from these two basic states we need to consider the first few lowest energy configurations. To help the counting there are two parameters that we keep track of: the number of negative spins, $t$, and the perimeter of the boundary, $s$.
The first type of state produces the counting expressed in table \ref{table:grid_mag_zero_first_conf_count}. The associated configurations are represented in figure \ref{fig:grid_basic_confi_first_type_extended}.

\begin{table}[ht] 
\caption{Low energy states from first basic configuration for low temperature expansion of $\prob(|M| = 0)$} 
\centering
\begin{tabular}{c c c}
\hline
Negative spins, t & Boundary perimeter, s & Number of states \\ [0.5ex]
\hline
2 & 8 & 4 $\times$ 1 \\
3 & 8 & 4 $\times$ 1 \\
3 & 10 & 4 $\times$ 4 \\
4 & 10 & 4 $\times$ 6 \\
5 & 10 & 4 $\times$ 2 \\ [1ex]
\hline
\end{tabular} 
\label{table:grid_mag_zero_first_conf_count}
\end{table}

\begin{figure}
\begin{center}
\includegraphics[width=0.25\linewidth]{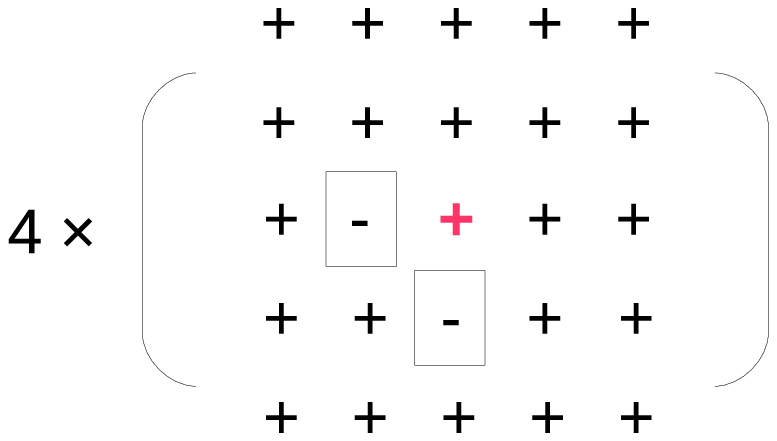}
\includegraphics[width=0.25\linewidth]{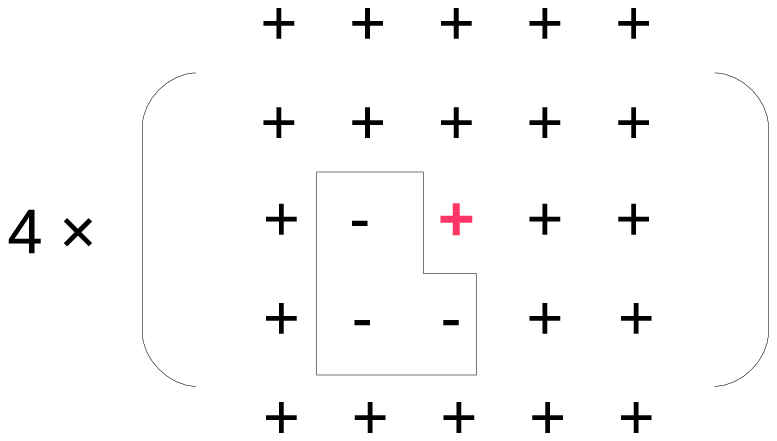}
\includegraphics[width=0.25\linewidth]{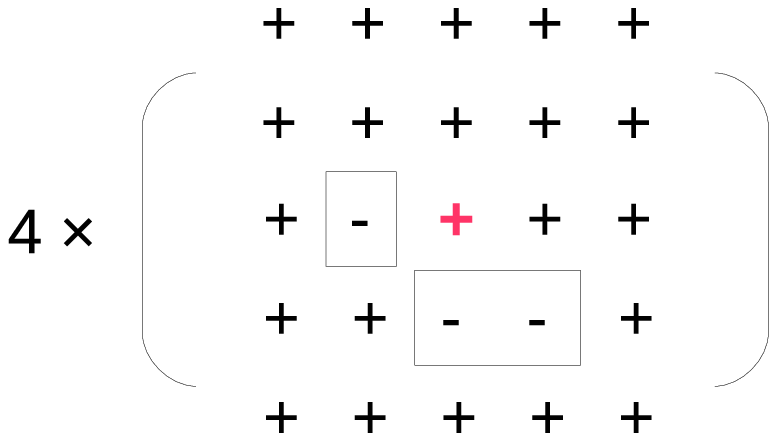}
\includegraphics[width=0.25\linewidth]{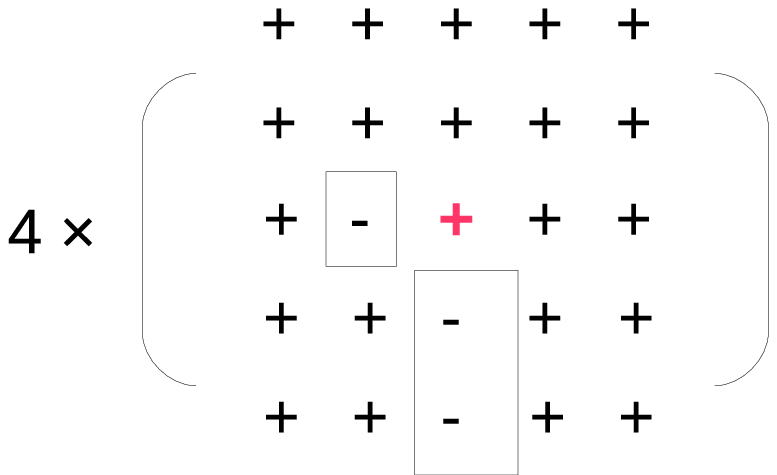}
\includegraphics[width=0.25\linewidth]{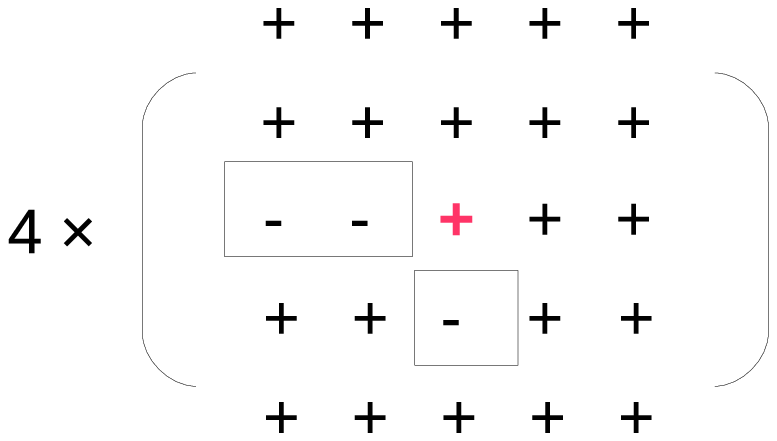}
\includegraphics[width=0.25\linewidth]{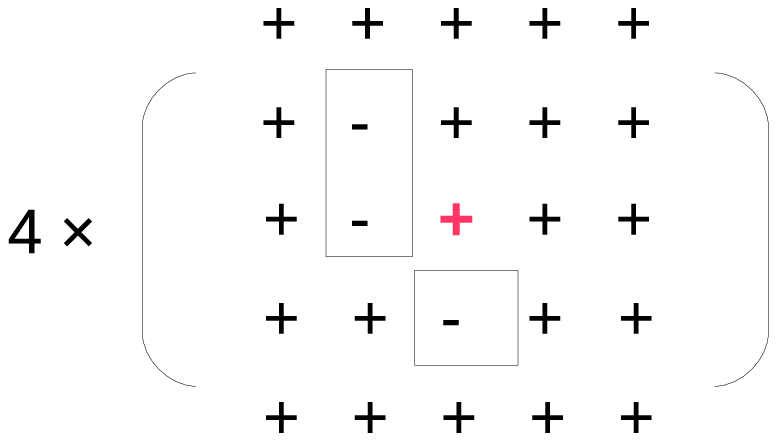}
\includegraphics[width=0.25\linewidth]{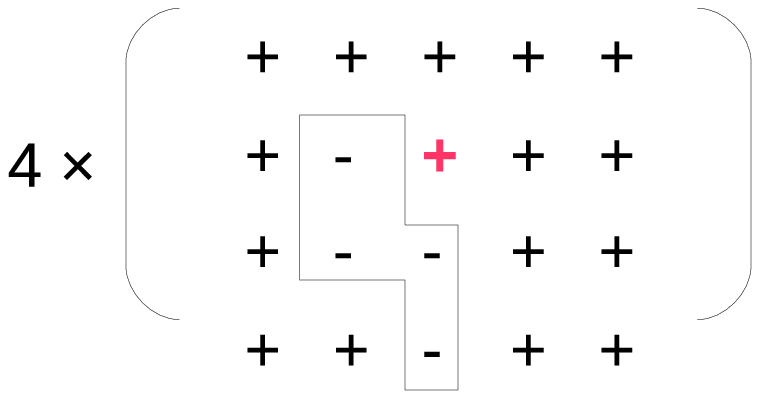}
\includegraphics[width=0.25\linewidth]{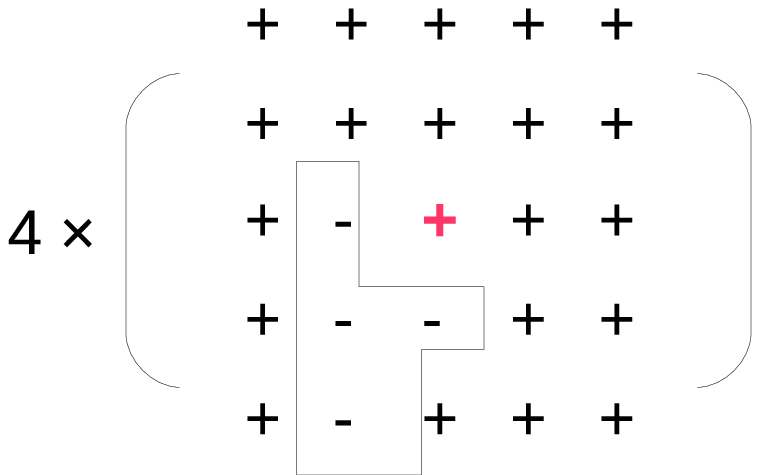}
\includegraphics[width=0.25\linewidth]{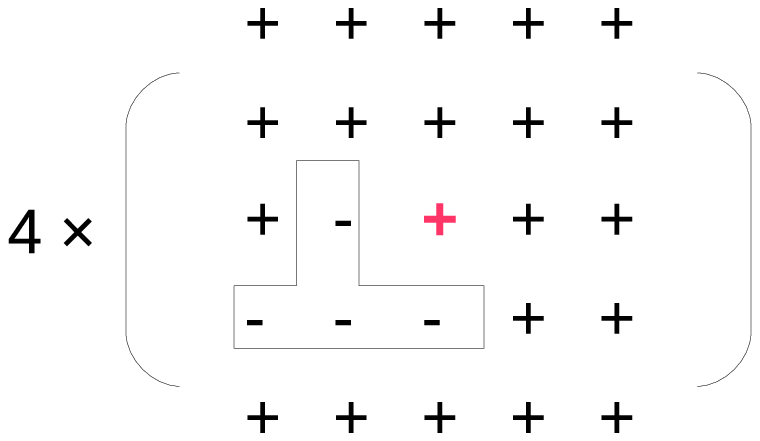}
\includegraphics[width=0.25\linewidth]{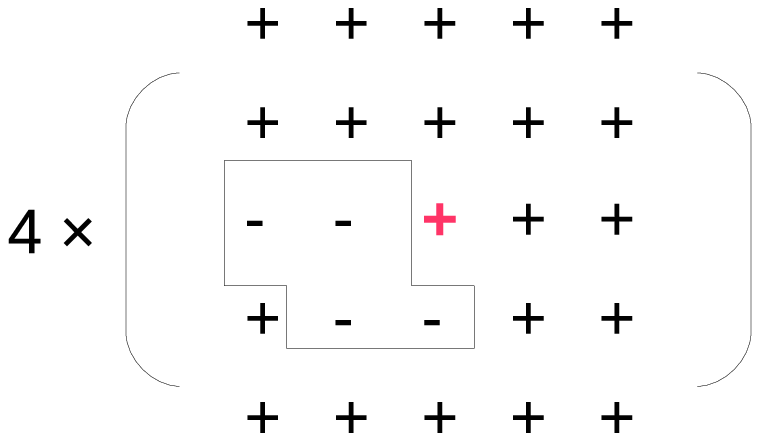}
\includegraphics[width=0.25\linewidth]{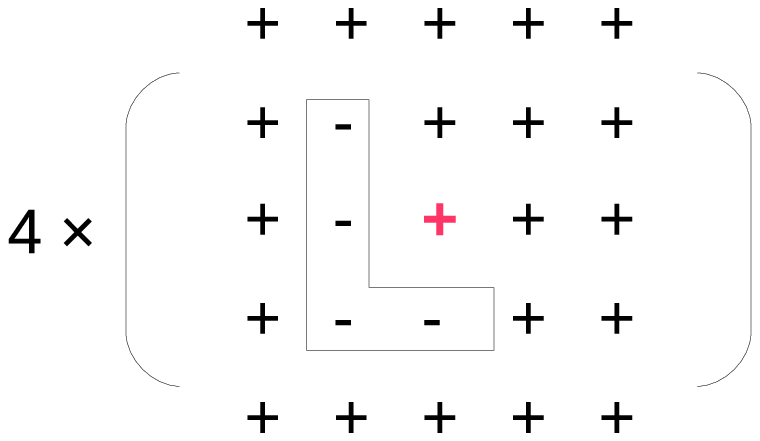}
\includegraphics[width=0.25\linewidth]{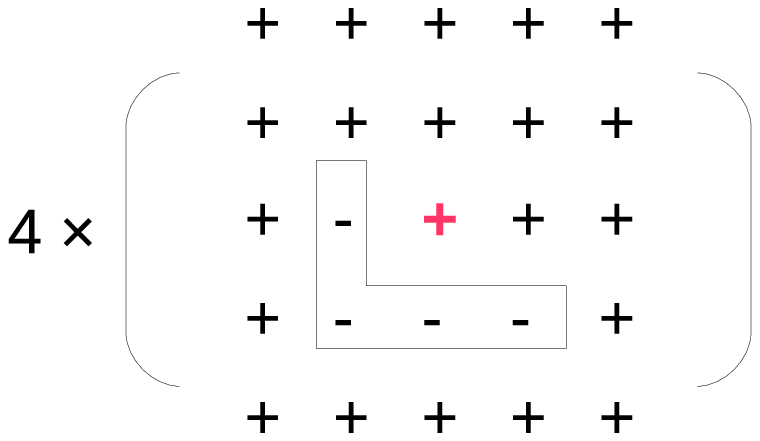}
\includegraphics[width=0.25\linewidth]{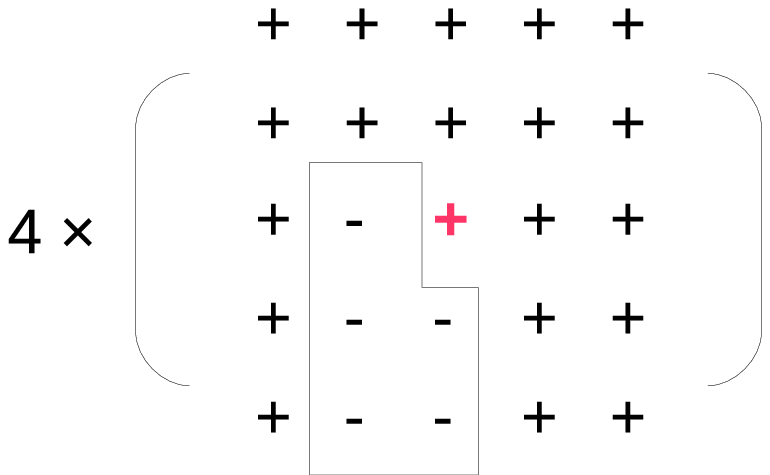}
\includegraphics[width=0.25\linewidth]{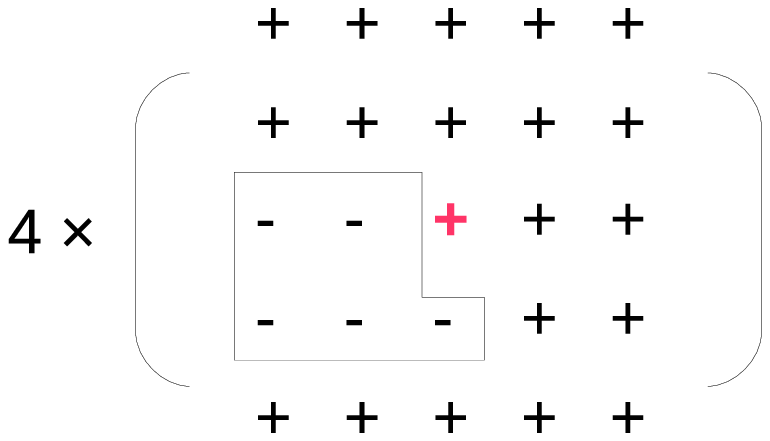}
\end{center}
\caption{Configurations derived from first basic type of configuration for the calculation of $\prob(|M| = 0)$.}\label{fig:grid_basic_confi_first_type_extended}

\end{figure}

For the second type of basic configuration the counting is in table \ref{table:grid_mag_zero_second_conf_count} and the associated configurations in figure \ref{fig:grid_basic_confi_second_type_extended}.

\begin{table}[ht] 
\caption{Low energy states from second basic configuration for low temperature expansion of $\prob(|M| = 0)$}
\centering
\begin{tabular}{c c c}
\hline
Negative spins, t & Boundary perimeter, s & Number of states \\ [0.5ex]
\hline
2 & 8 & 4 $\times$ 1 \\
3 & 10 & 4 $\times$ 6 \\[1ex]
\hline
\end{tabular} 
\label{table:grid_mag_zero_second_conf_count}
\end{table}

\begin{figure}
\begin{center}
\includegraphics[width=0.25\linewidth]{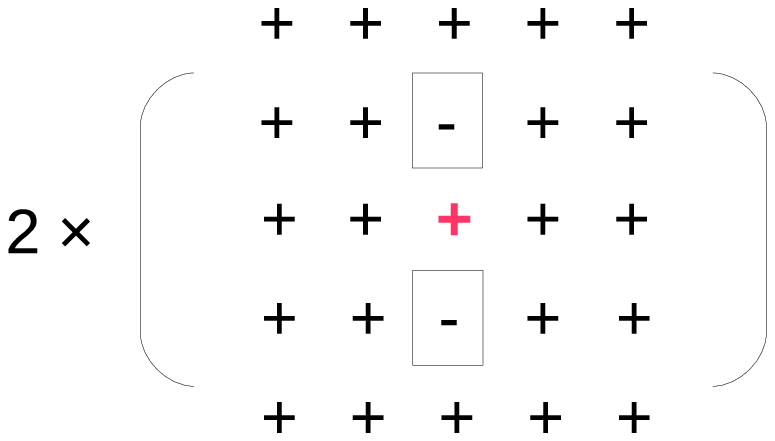}
\includegraphics[width=0.25\linewidth]{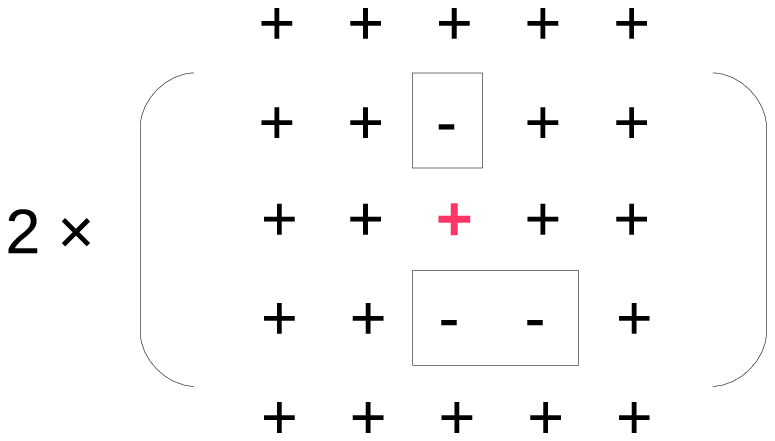}
\includegraphics[width=0.25\linewidth]{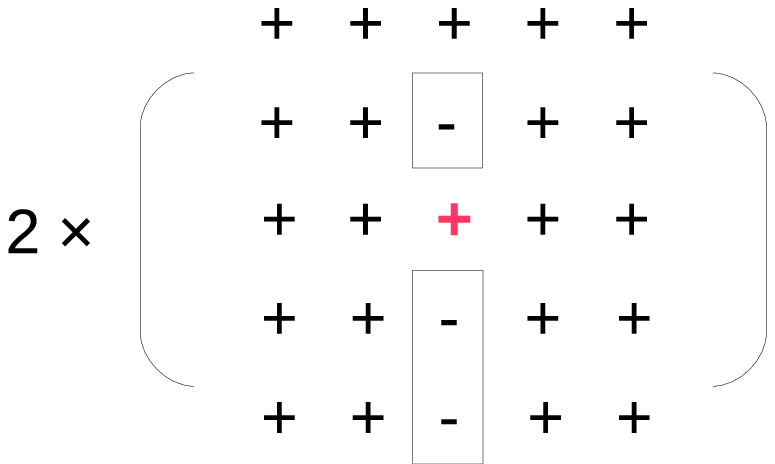}
\includegraphics[width=0.25\linewidth]{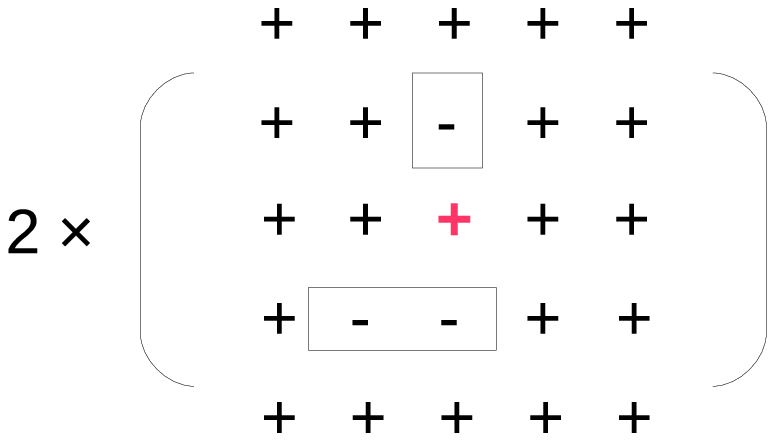}
\includegraphics[width=0.25\linewidth]{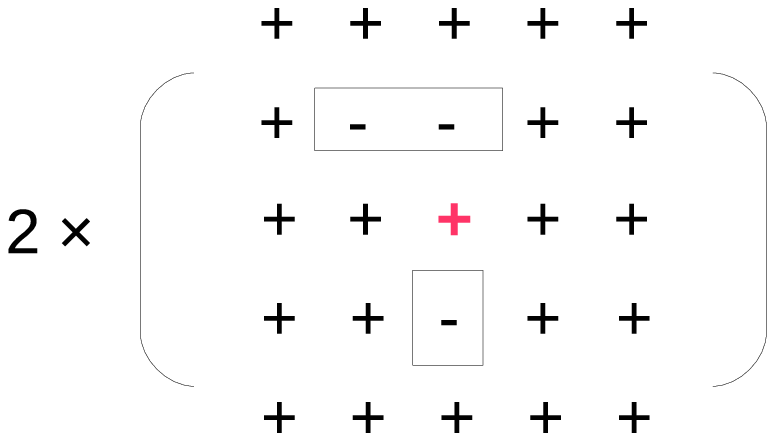}
\includegraphics[width=0.25\linewidth]{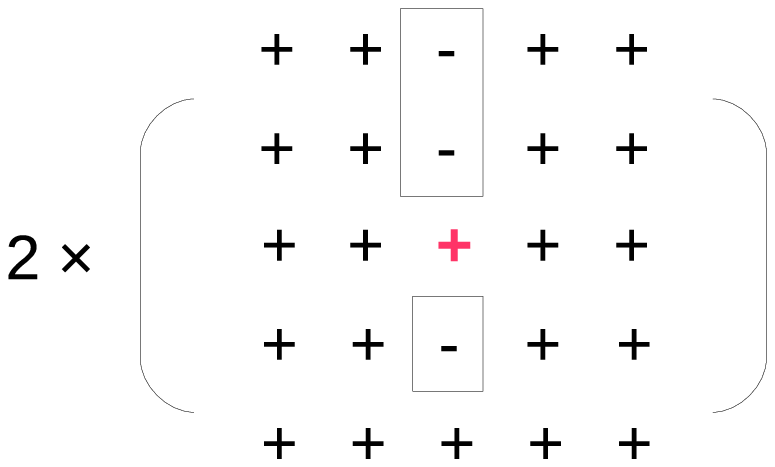}
\includegraphics[width=0.25\linewidth]{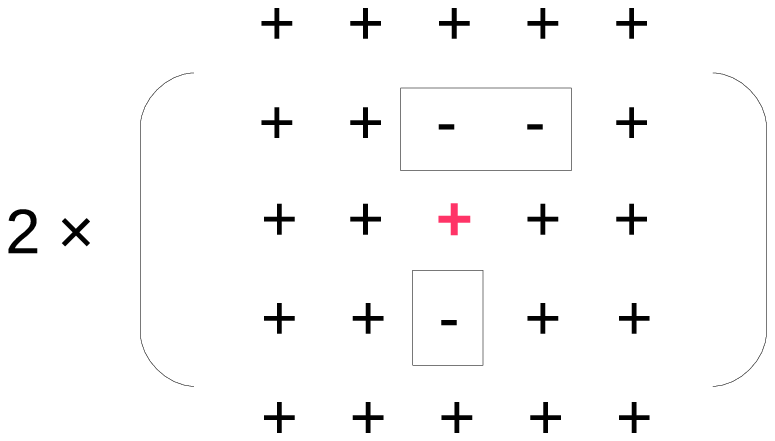}
\end{center}
\caption{Configurations derived from second basic type of configuration for the calculation of $\prob(|M| =  0)$.}\label{fig:grid_basic_confi_second_type_extended}
\end{figure}

We can thus write,
\begin{equation}
\prob(|M| = 0) = \frac{2}{Z} e^{\theta \mathcal{H}_{\max}} (10 e^{-16 \theta} + 60 e^{-20 \theta} + O(e^{-24 \theta})).
\end{equation}

For the expansion of $\prob(|M| = 2)$ we also have two basic states types from which all the other ones are built. The first type has only one negative spin in the neighborhood of node 0 and the second type has 3 negative spins in the neighborhood of node 0. See figure \ref{fig:grid_basic_confi_type_mag_2}.

\begin{figure}
\begin{center}
\includegraphics[width=0.25\linewidth]{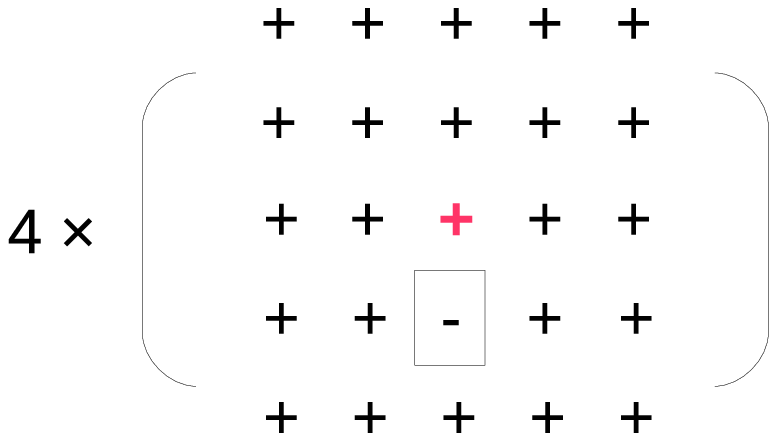} \hspace{1.2cm}
\includegraphics[width=0.25\linewidth]{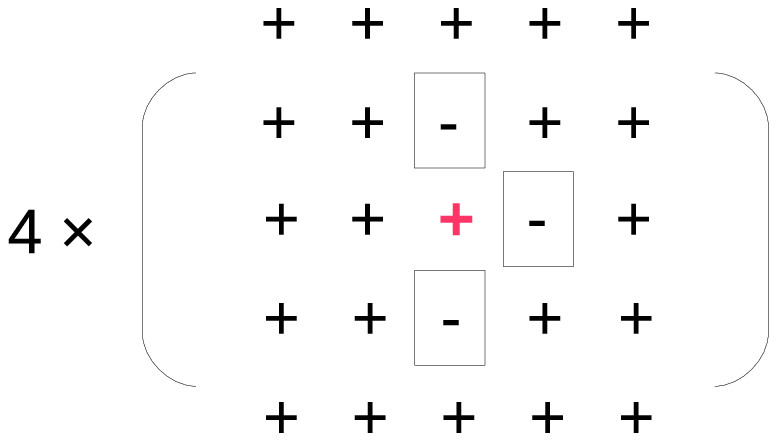}
\end{center}
\caption{Basic type of configurations for the calculation of $\prob(|M| = 2)$. The number in front of each picture represents the number of equivalent symmetric configurations that need to be taken into account.}\label{fig:grid_basic_confi_type_mag_2}
\end{figure}

The counting of states derived from the first basic state type and second basic state type are recorded in tables \ref{table:grid_mag_two_first_conf_count} and \ref{table:grid_mag_two_second_conf_count} respectively.

\begin{table}[ht] 
\caption{Low energy states from first basic configuration for calculation of $\prob(|M| = 2)$} 
\centering
\begin{tabular}{c c c}
\hline
Negative spins, t & Boundary perimeter, s & Number of states \\ [0.5ex]
\hline
1 & 4 & 4 $\times$ 1 \\
2 & 6 & 4 $\times$ 3 \\
2 & 8 & 4 $\times$ ($|E|$ - 8) \\
3 & 8 & 4 $\times$ 10 \\
4 & 8 & 4 $\times$ 2 \\ [1ex]
\hline
\end{tabular} 
\label{table:grid_mag_two_first_conf_count}
\end{table}

\begin{table}[ht] 
\caption{Low energy states from second basic configuration for calculation of $\prob(|M| = 2)$} 
\centering
\begin{tabular}{c c c}

\hline
Negative spins, t & Boundary perimeter, s & Number of states \\ [0.5ex]
\hline
3 & 12 & 4 $\times$ 1\\[1ex]
\hline
\end{tabular} 
\label{table:grid_mag_two_second_conf_count}
\end{table}

We can thus write,
\begin{equation}
\prob(|M| = 2) = \frac{2}{Z} e^{\theta \mathcal{H}_{\max}} (4 e^{-8 \theta} + 12 e^{-12 \theta} + O(e^{-16 \theta})).
\end{equation}

For the expansion of $\prob(|M| = 4)$ we again have two basic states types from which all the other ones are built. The first type has all spins positive in the neighborhood of node 0 and the second type has all spins negative in the neighborhood of node 0. The counting of states in printed in table \ref{table:grid_mag_four_first_and_second_conf_count}.

\begin{table}[ht] 
\caption{Low energy states for calculation of $\prob(|M| = 4)$} 
\centering
\begin{tabular}{c c c}
\hline
Negative spins, t & Boundary perimeter, s & Number of states \\ [0.5ex]
\hline
0 & 0 & 1 \\
1 & 4 & $|E|$ - 5\\
\hline\\\hline
4 & 16 & 1 \\[1ex]
\hline
\end{tabular} 
\label{table:grid_mag_four_first_and_second_conf_count}
\end{table}

We thus have,
\begin{equation}
\prob(|M| = 4) = \frac{2}{Z} e^{\theta \mathcal{H}_{\max}} (4 e^{-8 \theta} + 12 e^{-12 \theta} + O(e^{-16 \theta})).
\end{equation}

Using the expansion $1/\cosh^2(x) = 4 e^{-2 x} (1 - 2 e^{-2 x} + 3 e^{-4x}+ O(e^{-8x}))$ we can finally write,
\begin{equation}
a = \frac{2}{Z} e^{\theta \mathcal{H}_{\max} } (4 e^{-8 \theta} + 16 e^{-12 \theta} + O(e^{-16 \theta})).
\end{equation}

For the probabilities involved in the calculation of $b$ we get the following expansions,
\begin{eqnarray}
\prob(|M| = 0, X_1X_2 = 1) &=& \frac{2}{Z} e^{\theta \mathcal{H}_{\max}} (4 e^{-16 \theta} + 24e^{-20 \theta} + O(e^{-24 \theta})),\\
\prob(|M| = 0, X_1X_2 = -1) &=& \frac{2}{Z} e^{\theta \mathcal{H}_{\max}} (6 e^{-16 \theta} + 38 e^{-20 \theta} + O(e^{-24 \theta})),\\
\prob(|M| = 2, X_1X_2 = 1) &=& \frac{2}{Z} e^{\theta \mathcal{H}_{\max}} (2 e^{-8 \theta} + 6 e^{-12 \theta} + O(e^{-16 \theta})),\\
\prob(|M| = 2, X_1X_2 = -1) &=& \frac{2}{Z} e^{\theta \mathcal{H}_{\max}} (2 e^{-8 \theta} + 6 e^{-12 \theta} + O(e^{-16 \theta})),\\
\prob(|M| = 4, X_1X_2 = 1) &=& \frac{2}{Z} e^{\theta \mathcal{H}_{\max}} (1 + (|E| - 5) e^{-8 \theta} + O(e^{-12 \theta})),\\
\prob(|M| = 4, X_1X_2 = -1) &=& 0,
\end{eqnarray}
and putting everything together we obtain,
\begin{equation}
b = \frac{2}{Z} e^{\theta \mathcal{H}_{\max}} (4 e^{-8 \theta} + (4|E| -30) e^{-16 \theta} + O(e^{-20 \theta})).
\end{equation}

For the probabilities involved in the calculation of $c$ we get the following expansions,
\begin{eqnarray}
\prob(|M| = 0, X_1X_3 = 1) &=& \frac{2}{Z} e^{\theta \mathcal{H}_{\max}} (2 e^{-16 \theta} + 12 e^{-20 \theta} + O(e^{-24 \theta})),\\
\prob(|M| = 0, X_1X_3 = -1) &=& \frac{2}{Z} e^{\theta \mathcal{H}_{\max}} (8 e^{-16 \theta} + 48 e^{-20 \theta} + O(e^{-24 \theta})),\\
\prob(|M| = 2, X_1X_3 = 1) &=& \frac{2}{Z} e^{\theta \mathcal{H}_{\max}} (2 e^{-8 \theta} + 6 e^{-12 \theta} + O(e^{-16 \theta})),\\
\prob(|M| = 2, X_1X_3 = -1) &=& \frac{2}{Z} e^{\theta \mathcal{H}_{\max}} (2 e^{-8 \theta} + 6 e^{-12 \theta} + O(e^{-16 \theta})),\\
\prob(|M| = 4, X_1X_3 = 1) &=& \frac{2}{Z} e^{\theta \mathcal{H}_{\max}} (1 + (|E| - 5) e^{-8 \theta} + O(e^{-12 \theta})),\\
\prob(|M| = 4, X_1X_3 = -1) &=& 0,
\end{eqnarray}
and putting everything together we obtain,
\begin{equation}
c = \frac{2}{Z} e^{\theta \mathcal{H}_{\max}} (4 e^{-8 \theta} + (4|E| -34) e^{-16 \theta} + O(e^{-20 \theta})).
\end{equation}

For the probabilities involved in the calculation of $d$ we get the following expansions,
\begin{eqnarray}
\prob(|M| = 0, X_1X_5 = 1) &=& \frac{2}{Z} e^{\theta \mathcal{H}_{\max}} (6 e^{-16 \theta} + 38 e^{-20 \theta} + O(e^{-24 \theta})),\\
\prob(|M| = 0, X_1X_5 = -1) &=& \frac{2}{Z} e^{\theta \mathcal{H}_{\max}} (4 e^{-16 \theta} + 19 e^{-20 \theta} + O(e^{-24 \theta})),\\
\prob(|M| = 2, X_1X_5 = 1) &=& \frac{2}{Z} e^{\theta \mathcal{H}_{\max}} (3 e^{-8 \theta} + 9 e^{-12 \theta} + O(e^{-16 \theta})),\\
\prob(|M| = 2, X_1X_5 = -1) &=& \frac{2}{Z} e^{\theta \mathcal{H}_{\max}} (e^{-8 \theta} + 3 e^{-12 \theta} + O(e^{-16 \theta})),\\
\prob(|M| = 4, X_1X_5 = 1) &=& \frac{2}{Z} e^{\theta \mathcal{H}_{\max}} (1 + (|E| - 6) e^{-8 \theta} + O(e^{-12 \theta})),\\
\prob(|M| = 4, X_1X_5 = -1) &=& \frac{2}{Z} e^{\theta \mathcal{H}_{\max}} (e^{-8 \theta} + O(e^{-12 \theta})),
\end{eqnarray}
and putting everything together we obtain,
\begin{equation}
d = \frac{2}{Z} e^{\theta \mathcal{H}_{\max}} (4 e^{-8 \theta} + 8 e^{-12 \theta} + (4 |E| - 30) e^{-16 \theta} + O(e^{-20 \theta})).
\end{equation}

For the probabilities involved in the calculation of $e$ we get the following expansions,
\begin{eqnarray}
\prob(|M| = 0, X_2X_5 = 1) &=& \frac{2}{Z} e^{\theta \mathcal{H}_{\max}} (4 e^{-16 \theta} + 22 e^{-20 \theta} + O(e^{-24 \theta})),\\
\prob(|M| = 0, X_2X_5 = -1) &=& \frac{2}{Z} e^{\theta \mathcal{H}_{\max}} (6 e^{-16 \theta} + 38 e^{-20 \theta} + O(e^{-24 \theta})),\\
\prob(|M| = 2, X_2X_5 = 1) &=& \frac{2}{Z} e^{\theta \mathcal{H}_{\max}} (3 e^{-8 \theta} + 7 e^{-12 \theta} + O(e^{-16 \theta})),\\
\prob(|M| = 2, X_2X_5 = -1) &=& \frac{2}{Z} e^{\theta \mathcal{H}_{\max}} (e^{-8 \theta} + 5 e^{-12 \theta} + O(e^{-16 \theta})),\\
\prob(|M| = 4, X_2X_5 = 1) &=& \frac{2}{Z} e^{\theta \mathcal{H}_{\max}} (1 + (|E| - 6) e^{-8 \theta} + O(e^{-12 \theta})),\\
\prob(|M| = 4, X_2X_5 = -1) &=& \frac{2}{Z} e^{\theta \mathcal{H}_{\max}} (e^{-8 \theta} + O(e^{-12 \theta})),
\end{eqnarray}
and putting everything together we obtain,
\begin{equation}
e = \frac{2}{Z} e^{\theta \mathcal{H}_{\max}} (4 e^{-8 \theta} + 8 e^{-12 \theta} +(4 |E| - 46) e^{-16 \theta}  + O(e^{-20 \theta})).
\end{equation}

Using the expansions for $a,b,c,d$ and $e$ and computing the series expansion of $[Q^*_{S^CS} {Q^*_{SS}}^{-1}]_5$ in powers of $e^{-\theta}$ we finally obtain,
\begin{equation}
\|[Q^*_{S^CS} {Q^*_{SS}}^{-1}]\|_{\infty} \geq \|[Q^*_{S^CS} {Q^*_{SS}}^{-1}]_5\|_1 = 1 + e^{-4 \theta} + O(e^{-8 \theta}).
\end{equation}
Following the ideas of \cite{Lebowitz} one can then show that the above formal expansion converges (a priori it could be case that one of the higher order terms would depend on $|E|$). This finishes the first part of the proof.

We now prove that there exists $C_{\min} >0$ such that $\lim_{p \to \infty} \sigma(Q^*_{SS}) > C_{\min}$. This will prove the second part of the theorem. First notice that the eigenvalues of $Q^*_{SS}$ are $\{a-c,a+2b+c,a-2b + c\}$. Now notice that,
\begin{align}
a-c &= \E \left(\frac{1-X_1X_2}{\cosh^2 (\theta M)}\right),\\
a + 2b + c &= \frac{1}{4} \E \left( \frac{M^2}{\cosh^2 (\theta M)} \right),\\
a - 2b + c &= \frac{1}{4} \E \left( \frac{(X_1 + X_3 - X_2 - X_4)^2}{\cosh^2 (\theta M)}  \right).
\end{align}
where for $a+2b+c$ and $a-2b+c$ we made use of the symmetry of the lattice. Since $1-X_1 X_2$, $M$ and $X_1+X_3-X_2-X_4$ only depend on a fixed finite number of spins, and since $\theta < \infty$, there is a positive probability, independent of $p$, of their being non-zero. Hence, all eigenvalues of $Q^*_{SS}$ are strictly positive even as $p \to \infty$.
\end{proof}

\subsection{Graphs $G_p$ from the toy example} \label{sec:rlf_for_gp}

In this section we show that $\rlr(\lambda)$ fails to reconstruct the 
graphs $G_p$ defined in Section \ref{sec:toy_example} (see  Figure 
\ref{fig:SimpleGraph}) for all $\lambda$ when $\theta$ is large
enough. Note that this differs from previous analysis in the sense that we do not require that $\lambda \rightarrow 0$. We also show that this 
`critical' $\theta$ behaves like $\Delta^{-1}$ for large $\Delta$. Our
analysis is based on numerical evaluation  of functions for which
explicit analytic expressions can be given along the lines of 
Section \ref{sec:toy_examp_comp}. Hence, our argument
should be understood as a sketch of a proof.

The success of $\rlr(\lambda)$ is dictated by the behavior of
$L(\utheta_{r,.};\{x^{(\ell)}\}^n_{\ell = 1})$ when $n$ is large. In fact, it is easy to use concentration inequalities to show that the solution of $\rlr$ for finite $n$ converges with high probability to the minima of $L_{\infty}(\utheta) +  \lambda \|\utheta\|_1$ where $L_{\infty}(\utheta) \equiv \lim_{n \to \infty}
L(\utheta_{r,.};\{x^{(\ell)}\}^n_{\ell = 1})$.

If $\lambda \to 0$
as $n \to \infty$, we have seen that the success of $\rlr$ is dictated
by the incoherence condition, which in turn is determined by the
Hessian of $L_{\infty}(\utheta)$. It is not hard to see that for this family of graphs, $\|Q^*_{S^C
  S} {Q^*_{SS}}^{-1}\|_{\infty}$ is increasing with $p$. For $p = 5$,
Eq.~\eqref{eq:inch_g_p_5} tells us that the incoherence condition will
be violated for $\theta$ high enough. Hence, by Lemma \ref{th:mart21}, $\rlr$ will fail for all
$G_p$ ($p \geq 5$) when $\lambda \to 0$ as $n \to \infty$. The question now is: how does $\rlr(\lambda)$ behave if $\lambda \rightarrow 0$ does not hold?

If $\lambda
> \text{constant} > 0$, the success of $\rlr$ is dictated by the
minima of $L_{\infty}(\utheta) +  \lambda \|\utheta\|_1$.  For this specific family of graphs, it is also not hard to see that for $0 < \theta < \infty$, $L_{\infty}$ is strictly convex and that due to symmetry the unique minimum of $L_{\infty}(\utheta) +  \lambda \|\utheta\|_1$ must satisfy $\htheta_{13} = \htheta_{14} = \dots = \htheta_{1p}$ for any $\lambda$. This allows us to consider $L_{\infty}(\utheta)$ as a function of only two parameters. We call it $L'(\theta_{13},\theta_{12}) \equiv L(\theta_{12},\theta_{13},\theta_{13},...,\theta_{13})$. Now, the problem of understanding $\rlr$ for $\lambda >0$, large $n$ and any $p$ becomes tractable and associated to understanding the following problem,
\begin{equation}
\min_{\theta_{13},\theta_{12}} L'(\theta_{13},\theta_{12}) + \lambda (p-2)|\theta_{13}| + \lambda |\theta_{12}|.
\end{equation}
We can analyze this optimization problem by 
solving it numerically. Figure \ref{fig:patho_graphs_sol_plot} shows the solution path of this problem as a function of $\lambda$ for $p = 5$ and for different values of $\theta$.

\begin{figure} 
\begin{center}
\includegraphics[width=0.5\linewidth]{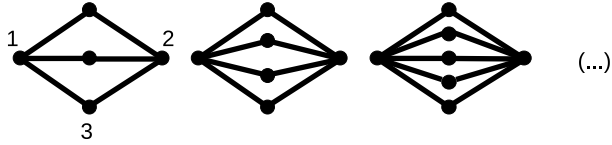}
\end{center}
\caption{For this family of graphs of increasing maximum degree $\Delta$ $\rlr(\lambda)$ will fail for any $\lambda > 0$ if $\theta > K/\Delta$, where $K$ is a large enough constant.} \label{fig:patho_graphs}
\end{figure}

From the plots we see that for high values of $\theta$, $\rlr$ will never yield a correct reconstruction (unless we assume $\lambda = 0$) since for these $\theta$s all curves are strictly above the horizontal axis, that is, $\hat{\theta}_{12} > 0$. However, if $\theta$ is bellow a certain value, call it $\theta_T$ ($\theta_T \approx 0.61$ for graph $G_5$), then there are solution that yield a correct reconstruction if we choose values of $\lambda > 0$. In fact, for $\theta < \theta_T$ all curves exhibit a portion (above a certain $\lambda$) that have $\htheta_{12} = 0$ and $\htheta_{13} > 0$. That is, for $\theta < \theta_T$, $\rlr$ makes a correct structural reconstruction. If we make $\theta$ even smaller then the curves identify themselves with the horizontal axis. We call by $\theta_L$ the value of $\theta$ below which this occurs.

\begin{figure}
\begin{center}
\includegraphics[width=0.7\linewidth]
{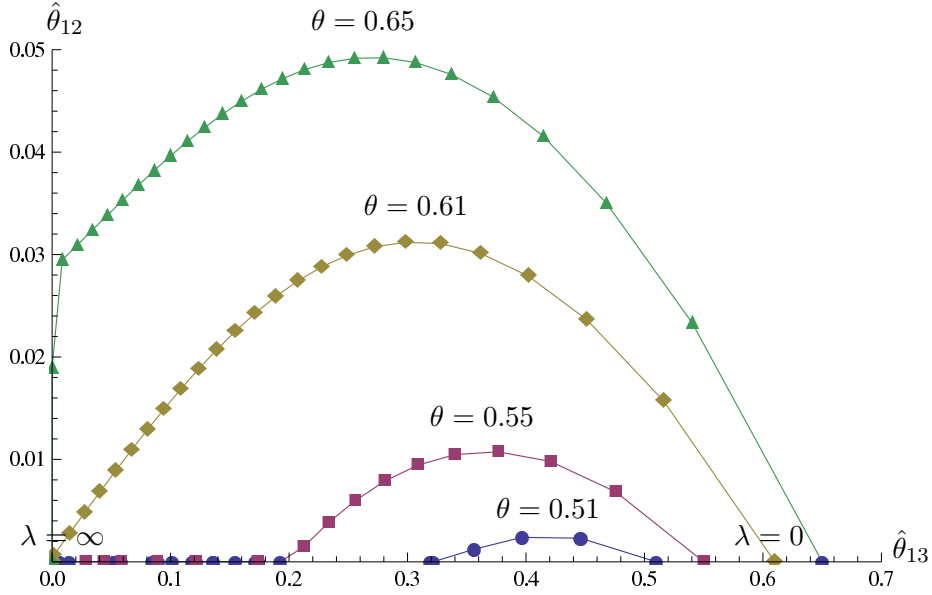}
\put(-150,25){$\theta = 0.51$}
\put(-175,60){$\theta = 0.55$}
\put(-220,210){$\theta = 0.65$}
\put(-200,140){$\theta = 0.61$}
\put(-60,15){$\lambda = 0$}
\put(-330,15){$\lambda = \infty$}
\put(0,10){$\hat{\theta}_{13}$}
\put(-320,210){$\hat{\theta}_{12}$}
\end{center}
\caption{Solution curves of $\rlr(\lambda)$ as a function of $\lambda$
  for different values of $\theta$ and $p=5$. Along each curve,
  $\lambda$ increases from right to left. Plot points separated by $\delta \lambda = 0.05$ are included to show the speed of the parameterization with $\lambda$. For $\lambda \rightarrow \infty$ all curves tend to the point $(0,0)$. Remark: Curves like the one for $\theta = 0.55$ are identically zero above a certain value of $\lambda$.}\label{fig:patho_graphs_sol_plot}
\end{figure}

We again note that all previous considerations were made in the limit when $n \rightarrow \infty$. For high finite $n$, with high probability the solution curves will not be the ones plotted but rather be random  fluctuations around these. For $\lambda = 0$, finite $n$ and $\theta > \theta_L$, the solution curves will no longer start from $\hutheta = \utheta^* = (\theta,0)$ but will have a positive non vanishing probability of having $\htheta_{12} > 0$. This reflects the fact that for finite $n$ the success of $\rlr(\lambda)$ requires $\lambda$ to be positive. However, for $\theta < \theta_L$ and $\lambda > 0$ such that we are in the region where the curves for $n = \infty$ are identically zero, the curves for finite $n$ will have an increasing probability of being identically zero too. Thus, for these values of $\lambda$ and $\theta$, the probability of successful reconstruction will tend to 1 as $n \rightarrow \infty$. From the plots we also conclude that, unless the whole curve (for $n = \infty$) is identified with zero, $\rlr(\lambda)$ restricted to the assumption $\lambda \rightarrow 0$ will fail with positive non vanishing probability for finite $n$. For $\theta < \theta_L$, when the curves (for $n = \infty$) become identically zero, there will be a scaling of $\lambda$ with $n$ to zero that will allow for a probability of success converging to 1 as $n \rightarrow \infty$.

Requiring $\lambda \rightarrow 0$ makes $\theta_L$ be the critical value above which reconstruction with $\rlr$ fails. This is the scenario in which we studied $\rlr$ in section \ref{sec:Pseudo}. In fact, $\theta_L$ coincides with the value above which $\|Q^*_{S^CS} {Q^*_{SS}}^{-1} \|_{\infty} > 1$. For this family of graphs we thus conclude that the true condition required for successful reconstruction is not $\|Q^*_{S^CS} {Q^*_{SS}}^{-1} \|_{\infty} < 1$ but rather that $\theta < \theta_T$. Surprisingly, for graphs in $G_p$ this condition coincides with $\E_{G,\theta}(X_1 X_3) > \E_{G,\theta}(X_1 X_2)$, i.e. the correlation between neighboring nodes must be bigger than that between non-neighboring nodes. Notice that this condition is in fact the condition required for $\thres$ to work. Consequently, for this family of graphs, the thresholding algorithm will always have a working range in terms of $\theta$ larger than that of $\rlr$, when restricted to $\lambda \rightarrow \infty$. In fact, a simple calculation using the local weak convergence used in proving Lemma \ref{th:mart23} shows that with high probability, for large random regular graphs, the correlation between neighboring nodes is always strictly greater than between non-neighboring nodes. This shows that the thresholding algorithm has as operation range $\theta \in (0,\infty)$ for random regular graphs, compared to $\theta \in (0, \theta_L)$ for $\rlr$.

We will now prove that for large enough $\Delta = p-2$ there is a unique $\theta_T(\Delta)$ (solution of $\E_{G,\theta,\Delta}(X_1 X_3) = \E_{G,\theta,\Delta}(X_1 X_2)$) that scales like $1/\Delta$ and above which $\E_{G,\theta,\Delta}(X_1 X_3) < \E_{G,\theta,\Delta}(X_1 X_2)$. Let 1 and 2 be the two nodes with degree greater than 2 and let 3 be any other node (of degree 2), see Figure \ref{fig:patho_graphs}. Define $x_\Delta = \E_{G,\theta,\Delta}(X_1X_2)$ and $y_\Delta = \E_{G,\theta,\Delta}(X_1X_3)$. It is not hard to see that, 
\begin{equation}
x_{\Delta+1} = \frac{x_\Delta + \tanh^2\theta }{1+\tanh^2\theta \; x_\Delta} \phantom{aaaa} y_{\Delta+1}= \frac{\tanh \theta \; x_\Delta + \tanh\theta }{1+\tanh^2\theta \; x_\Delta}.
\end{equation}
From these expression we see that the condition $x_\Delta (\theta) > y_\Delta (\theta) $ is equivalent to $x_{\Delta-1} (\theta) > \tanh \theta$. Remembering that expectations on the Ising model \eqref{eq:IsingModel} can be computed from subgraphs of $G$, \cite{fisher}, an easy calculation shows that, 
\begin{equation}
x_\Delta (\theta) = \frac{(1+z(\theta))^\Delta-(1-z(\theta))^\Delta}{(1+z(\theta))^\Delta+(1-z(\theta))^\Delta},
\end{equation}
where $z(\theta) = \tanh^2(\theta)$. Since $x_\Delta \rightarrow 1$ with $\Delta$ then any $\theta_T$ also goes to $0$ with $\Delta$ and attending to the slope and concavity of $x_\Delta(\theta)$ and $\tanh(\theta)$ for small $\theta$ it is easy to see that for large $\Delta$ there will exist a unique solution $\theta_T(\Delta)$. Furthermore, the condition $x_{\Delta+1} (\theta) = y_{\Delta+1} (\theta) $ can now be written like,
\begin{equation}
\sqrt{z(\theta)} = \frac{(1+z(\theta))^\Delta-(1-z(\theta))^\Delta}{(1+z(\theta))^\Delta+(1-z(\theta))^\Delta}.
\label{eq:path_graph_corr_switch}
\end{equation}
Assuming $z = K \Delta^{-\gamma}$, multiplying both sides of the previous equation by $\Delta^{\gamma/2}$ and taking the limit when $\Delta \rightarrow \infty$ we obtain,
\begin{equation}
\sqrt{K} = \lim_{\Delta \rightarrow \infty} \Delta^{\gamma/2} \tanh(K \Delta^{1-\gamma}),
\end{equation}
which will result in a non trivial relation for $K$ only if $\gamma = 2$. In this case we get $K^{1/2} = K$ and thus for any $\epsilon > 0$, if $\Delta$ is sufficiently high, there will be a (unique) solution of \eqref{eq:path_graph_corr_switch} inside the interval $[(1-\epsilon)/\Delta^2, (1+\epsilon)/\Delta^2]$. Since $z(\theta) = \tanh^2(\theta)$ then $\theta_T(\Delta)$ scales likes $1/\Delta$ as we wanted to prove.

\bibliographystyle{amsalpha}

\end{document}